%% file: main.tex
\input{tex/preamble}

\input{tex/symbols}

\begin{document}

\twocolumn[
\icmltitle{Theoretical Guarantees for One-Shot Magnitude Pruning and Compute-Adaptive Early Exit}

\icmlsetsymbol{equal}{*}

\begin{icmlauthorlist}
	\icmlauthor{Erdem Koyuncu}{uic}
\end{icmlauthorlist}

\icmlaffiliation{uic}{Department of Electrical and Computer Engineering, University of Illinois Chicago, IL, USA}
\icmlcorrespondingauthor{Erdem Koyuncu}{ekoyuncu@uic.edu}

\icmlkeywords{Machine Learning, ICML}

\vskip 0.3in
]
\printAffiliationsAndNotice{}  

\input{tex/abstract}

\input{tex/introduction}

\input{tex/static-neurons}

\input{tex/adaptive-neurons}
\input{tex/static-networks}

\input{tex/adaptive-networks}

\input{tex/numerical}

\input{tex/conclusions}

\section*{Acknowledgements}
This work was supported in part by Army Research Lab under Grant W911NF-24-2-0172, and by National Science Foundation under Grant 2531376.


\section*{Impact Statement}
This paper presents work whose goal is to advance the field of machine learning. There are many potential societal consequences of our work, none of which we feel must be specifically highlighted here.

\bibliography{bib/sample,bib/misc-refs}
\bibliographystyle{icml2026}

\appendix
\onecolumn
\input{tex/app/notation}
\input{tex/app/gammatheorem}

\input{tex/app/ee-neuron-app}
\input{tex/app/prune-network-app}
\input{tex/app/thres-calib-app}

\input{tex/app/adaptive-net-app}

\input{tex/app/simul-app}

\input{tex/app/related-work}

\end{document}

%% file: tex/preamble.tex
\documentclass{article}

 \usepackage{cuted, caption, subcaption}

\usepackage{enumitem}
\usepackage{ amsmath, amssymb, graphicx, url,amsthm,subcaption,xcolor,amsthm,proof}
\usepackage{stackengine}
\usepackage{pgfplots,wrapfig}
\usepgfplotslibrary{groupplots,statistics}
\pgfplotsset{compat=1.17}

\definecolor{darkgreen}{RGB}{0,100,0}

\usepackage{mathtools, multirow}
\usepackage{bm}
\usepackage{microtype, array}

\allowdisplaybreaks

\AtBeginDocument{
	\setlength{\abovedisplayskip}{6pt}
	\setlength{\belowdisplayskip}{6pt}
	\setlength{\abovedisplayshortskip}{3pt}
	\setlength{\belowdisplayshortskip}{3pt}
	\setlength{\jot}{2pt}
}

\AtBeginDocument{
	\setlength{\textfloatsep}{8pt}
	\setlength{\floatsep}{6pt}
	\setlength{\intextsep}{6pt}
	\setlength{\abovecaptionskip}{3pt}
	\setlength{\belowcaptionskip}{0pt}
}

\usepackage[utf8]{inputenc} 
\usepackage[T1]{fontenc}    
\usepackage{hyperref}       
\usepackage{url}            
\usepackage{booktabs}       
\usepackage{amsfonts}       
\usepackage{nicefrac}       
\usepackage{microtype}      
\usepackage{xcolor}         

\title{Understanding Pruning and Early Exit in Neural Networks Using Conditional Neurons}

\author{Erdem Koyuncu}

\newcounter{cc}
\usepackage{microtype}
\usepackage{graphicx}
\usepackage{subcaption}
\usepackage{booktabs} 

\usepackage{hyperref}

\usepackage[accepted]{icml2026}

\usepackage{amsmath}
\usepackage{amssymb}
\usepackage{mathtools}
\usepackage{amsthm}

\usepackage[capitalize,noabbrev]{cleveref}

\theoremstyle{plain}

\newtheorem{example}{Example}[section]
\newtheorem{theorem}{Theorem}[section]
\newtheorem{proposition}[theorem]{Proposition}
\newtheorem{lemma}[theorem]{Lemma}
\newtheorem{corollary}[theorem]{Corollary}
\theoremstyle{definition}

\newtheorem{assumption}[theorem]{Assumption}
\theoremstyle{remark}

\usepackage[textsize=tiny]{todonotes}

\icmltitlerunning{Theoretical Guarantees for One-Shot Magnitude Pruning and Compute-Adaptive Early Exit}

\DeclareMathOperator{\sgn}{sign}

%% file: tex/symbols.tex
\newcommand{\inputvec}{\mathbf{x}}
\newcommand{\weightvec}{\mathbf{w}}
\newcommand{\prunedvec}{\mathbf{w}_{\mathrm{p}}}
\newcommand{\renormprunedvec}{\widetilde{\mathbf{w}}_{\mathrm{p}}}
\newcommand{\eevec}{\renormprunedvec}
\newcommand{\teachvec}{\mathbf{t}}
\newcommand{\trainvec}{\overline{\mathbf{x}}}
\newcommand{\maxmarginvec}{\overline{\mathbf{w}}}
\newcommand{\prob}{\mathbb{P}}
\renewcommand{\P}{\prob}
\newcommand{\E}{\mathbb{E}}

\newcommand{\avec}{\mathbf{a}}
\newcommand{\zvec}{\mathbf{z}}
\newcommand{\matW}{\mathbf{W}}
\newcommand{\matM}{\mathbf{M}}
\newcommand{\readoutvec}{\mathbf{y}}
\newcommand{\netout}{f}

\newcommand{\renormnetout}{\widehat{f}}
\newcommand{\prunedW}{\widetilde{\matW}}
\newcommand{\renormW}{\widehat{\matW}}

\newcommand{\netexitthres}{\lambda}

%% file: tex/abstract.tex

\begin{abstract}
We study compute reduction in neural networks through a unified partial versus full computation view, captured by one-shot magnitude pruning in the static regime and early exit in the adaptive regime. In an asymptotic single-neuron model, we prove a concentration theorem for one-shot magnitude pruning with explicit rates. We also introduce the conditional perceptron for early exit and show that its excess generalization error decays as a power of the compute gap, with an exponent that grows to infinity as the alignment between partial and full computations tends to one. We then extend the analysis to deep networks, characterizing how pruning-induced distortions accumulate with depth and deriving a corresponding compute–accuracy tradeoff for frozen-backbone early exit under a neural network Gaussian process model. Numerical simulations corroborate the predicted scaling laws.
\end{abstract}

%% file: tex/introduction.tex
\section{Introduction}

The field of artificial intelligence (AI) has seen dramatic progress, driven largely by advances in the design and training of large-scale neural networks. A key observation is that the performance of neural networks scales predictably with model size, dataset size, and compute budget. This empirical observation is often described by scaling laws~\cite{kaplan2020scaling}, which suggest that performance improves smoothly as compute is increased. However, the theoretical understanding of these trends remains elusive.

One promising strategy for achieving scalable performance at reduced compute cost is to activate only a subset of a model’s parameters depending on the input, a paradigm known as conditional computation~\cite{bengio2013estimating, davis2013lowrank, eigen2013learning}. This includes a variety of mechanisms, such as early-exit networks~\cite{han2021dynamic,li2024seenn,teerapittayanon2016branchynet, gormez2022early}, where confidence scores from intermediate layers can trigger early outputs to save computation, and sparsely-gated mixture-of-experts (MoE) models~\cite{fedus2021switch,shazeer2017outrageously,jiang2024mixtral}, where input-dependent gating activates only a subset of expert sub-networks for each input. These architectures can achieve substantial compute savings while maintaining accuracy, but a principled explanation of when and why this happens remains limited.

This work  develops a tractable theoretical characterization of compute--accuracy tradeoffs in conditional computation. Our approach begins at the level of a \emph{single neuron}, where sharp results are possible, and then transitions to \emph{deep networks} by showing how the same mechanisms propagate through depth. We associate conditional computation with the coexistence of a \emph{full} computation and a \emph{partial} computation. In the static setting, the partial computation is used for every input, which corresponds to magnitude pruning. In the adaptive setting, the partial computation is used to form a confidence score and the full computation is invoked only when needed, which corresponds to early exit. We first analyze these two regimes for a single neuron, and then lift them to multi-layer networks. In this context, the main contributions of this work are summarized as follows:

\begin{itemize}[leftmargin=*, itemsep=0em, parsep=0em, topsep=0em, partopsep=0em]
	\item \textbf{Static neurons.}
	We give a closed-form characterization of \emph{one-shot} magnitude-pruning distortion in a single neuron by proving that, in high dimensions, the cosine similarity between the full and (renormalized) pruned weight vectors concentrates to a pruning-rate dependent constant, with explicit convergence rates.
	This complements existing pruning theory that more often emphasizes iterative prune--train pipelines or general compression and generalization bounds~\cite{dery2024everybody,sun2023simple}.
	
	\item \textbf{Adaptive neurons.}
	We introduce the \emph{conditional perceptron} as an analytically tractable model of early exit, where a partial computation forms a confidence score and the full computation is invoked only when needed.
	In a student--teacher setting~\cite{opper1995statistical}, we prove that the excess generalization error decays as a power of the compute gap, with an exponent controlled by the alignment between partial and full computations, and diverging as this alignment tends to one.
	This provides an explicit compute--accuracy curve beyond purely policy-driven empirical tradeoffs.
	
	\item \textbf{Static networks.}
	We lift the one-shot pruning distortion to depth by coupling unpruned and renormalized-pruned computations in an infinite-width Gaussian limit, yielding a recursion that quantifies how misalignment accumulates across layers and how stability depends on the activation.
	This addresses the limited theoretical treatment of depth-accumulated distortion under one-shot pruning.
	
	\item \textbf{Adaptive networks.}
	For frozen backbones with distilled intermediate exits, we analyze early exit in an infinite-width Gaussian coupling model, characterize the joint law of intermediate and final scores, and derive an alignment-controlled compute--accuracy tradeoff as a function of exit depth through an effective correlation parameter.
\end{itemize}

Collectively, we provide a theoretical study of compute reduction in neural networks, from single neurons to deep architectures, and from static pruning to adaptive early exit. Our goal is not to provide a quantitatively exact description of every trained architecture, but to isolate analytically tractable mechanisms governing alignment loss, depth accumulation, and activation dependence. The experiments then test whether these mechanisms persist in practical pretrained models, supporting robustness of the governing trends beyond the idealized settings.

%

The rest of this paper is organized as follows: Sections~\ref{sec:static-neurons} and~\ref{sec:adaptive-neurons} analyze static pruning and adaptive early exit for a single neuron, respectively. Sections~\ref{sec:static-networks} and~\ref{sec:adaptive-networks} then lift these mechanisms to deep networks. Some numerical experiments are provided in Section~\ref{sec:numerical}. We draw our main conclusions in Section~\ref{sec:conclusions}.  The appendices provide the technical proofs, additional numerical results, and a summary of related work.


%% file: tex/static-neurons.tex
\section{Static Neurons}
\label{sec:static-neurons}

We begin with the static compute-reduction regime at the level of a single neuron, where one-shot magnitude pruning is applied once and then kept fixed. Let $\inputvec\in\mathbb{R}^n$ be an input and let $\weightvec\in\mathbb{R}^n$ be the corresponding feature-extracting weight vector. The \emph{full} computation forms the local field $\weightvec^{\top}\inputvec$. Given a pruning rate $q\in(0,1)$, one-shot magnitude pruning removes the $\lceil qn\rceil$ smallest-magnitude coordinates of $\weightvec$, producing a pruned vector $\prunedvec$.  A renormalized pruned vector $\renormprunedvec \triangleq \|\weightvec\| \prunedvec / \|\prunedvec\|$, can be obtained by rescaling $\prunedvec$ so that its norm matches that of $\weightvec$. The \emph{partial} computation can form either of the local fields $\prunedvec^{\top}\mathbf{x}$  or $\renormprunedvec^{\top}\mathbf{x}$. 

Our goal is to characterize how close the full and partial computations are in high dimensions. We show that, under standard isotropic weight models, the cosine similarity between $\mathbf{w}$ and its pruned counterparts concentrate around a deterministic limit that depends only on the pruning rate. We further provide an effective finite-$n$ deviation bound, yielding explicit asymptotic convergence rates. This concentration helps explain why one-shot pruning can preserve decision behavior, and it will serve as a building block for the adaptive early-exit analysis in Section~\ref{sec:adaptive-neurons}, and for the network-level depth-accumulation analysis in Section~\ref{sec:static-networks}. 

\newcommand{\gammashape}{s}

%
%

To state our main concentration theorem, we introduce notation related to the order statistics of Gamma random variables. Let $G_1,G_2,\ldots,G_n$ be i.i.d.\ copies of a Gamma random variable  $G \sim \mathrm{Gamma}(s,\theta)$ with shape $s$ and scale $\theta$. Denote their order statistics by $G_{(1)}\leq \cdots \leq G_{(n)}$. For a given $y\in[0,\infty)$, define the truncated Gamma random variable $G_{\leq y}$ as $G$ conditioned on $G \leq y$, where $y\in[0,\infty)$. The notation $G_{\geq y}$ is defined similarly.

\newcounter{constgammaonec}\setcounter{constgammaonec}{\value{cc}}
\newcommand{\constgammaone}{C_{\theconstgammaonec}}\addtocounter{cc}{1}

\newcounter{constgammatwoc}\setcounter{constgammatwoc}{\value{cc}}
\newcommand{\constgammatwo}{C_{\theconstgammatwoc}}\addtocounter{cc}{1}

\newcounter{constgammaepsc}\setcounter{constgammaepsc}{\value{cc}}
\newcommand{\constgammaeps}{C_{\theconstgammaepsc}}\addtocounter{cc}{1}

\newcounter{constgammanoc}\setcounter{constgammanoc}{\value{cc}}
\newcommand{\constgammano}{C_{\theconstgammanoc}}\addtocounter{cc}{1}

\newcommand{\maintheothres}{\tau}

\begin{theorem}
	\label{theorem:gamma}
	Let $G_1, G_2, \ldots, G_n$ be independent and identically distributed Gamma random variables with shape $s$ and scale $\theta$.
	Denote the corresponding order statistics by $G_{(1)} \leq \cdots \leq G_{(n)}$.
	For a given $0 < q < 1$, let
	\begin{align}
		\Xi_n \triangleq \frac{G_{(\lceil q n \rceil +1)} + \cdots + G_{(n)}}{G_1 + \cdots + G_n}.
	\end{align}
	Let $\omega_q$ be defined by $\P(G \leq \omega_q) = q$.
	Let
	\begin{align}
		\label{maintheothresdef}
		\tau_q \!\triangleq\! \biggl( 1 \!+\! \frac{q \mathbb{E}\bigl[G_{\leq \omega_q} \bigr]}{ (1-q) \mathbb{E}\bigl[G_{\geq \omega_q} \bigr]}\biggr)^{\!-1} \!\!\!\!\!=\!\! \frac{\Gamma(s\!+\!1,\frac{\omega_q}{\theta})}{\Gamma(s\!+\!1)} \!\in\! [0,1].
	\end{align}
Then, $\Xi_n \xrightarrow{p} \tau_q$ as $n\rightarrow\infty$. 	Moreover, there are constants $\constgammaone,\constgammatwo,\constgammaeps,\constgammano>0$ that depend only on $(q,s,\theta)$ such that
	for every $\epsilon\in(0,\constgammaeps]$ and every integer $n \ge \constgammano / \epsilon$, we have
	\begin{align}
		\P(\bigl|\Xi_n-\tau_q\bigr|\ge \epsilon)
		\le
		\constgammaone\,n^{-1/2}
		+
		4\exp(-\constgammatwo n\epsilon^2).
		\label{eq:gamma_clean_bound}
	\end{align}
\end{theorem}

One of the main technical tools used to establish Theorem~\ref{theorem:gamma} is a conditional independence property of order statistics \cite{ahsanullah2013conditional}: The lower block $(G_{(1)},\ldots,G_{(k-1)})$ and the upper block $(G_{(k+1)},\ldots,G_{(n)})$ are conditionally independent given $G_{(k)}$, where $k = \lceil q n \rceil$.
Since the numerator of $\Xi_n$ depends only on the upper block and the denominator splits into the sum of the upper block, the lower block, and $G_{(k)}$, this conditioning decouples the two contributions.
Conditioning on $G_{(k)}$, the problem reduces to controlling sums under the induced truncated laws, enabling a CLT approximation with Berry--Esseen error control and a separate concentration estimate for the random threshold $G_{(k)}$.
The $O(n^{-1/2})$ term in \eqref{eq:gamma_clean_bound} comes from bounding the Berry--Esseen remainder uniformly over the conditioning value $G_{(k)}$. The $\exp(-\constgammatwo n\epsilon^2)$ term has two main sources: On the typical event $|G_{(k)}-\omega_q| < \epsilon$, the conditional CLT turns the deviation event $|\Xi_n-\tau_q|>\epsilon$ into a Gaussian-tail event at scale $\sqrt{n}\epsilon$, yielding the factor $\exp(-c n\epsilon^2)$.
Separately, the probability that the random threshold falls outside this typical regime, $\P(|G_{(k)}-\omega_q|>\epsilon)$, contributes an additional $\exp(-c n\epsilon^2)$ term. For applications to neural networks, the following corollary is of particular interest.

\begin{corollary}
	\label{corol:concentration}
	Suppose $\weightvec$ is uniform on $\mathbb{S}^{n-1}$.  Let $\prunedvec$ be obtained by one-shot magnitude pruning at rate $q$, and let $\renormprunedvec$ denote the renormalized pruned vector.
	As $n\rightarrow\infty$, we have $\weightvec^{\top} \renormprunedvec \rightarrow \sqrt{\tau_q}$, and $\weightvec^{\top} \prunedvec  \rightarrow \tau_q$, where $\tau_q$ is defined in \eqref{maintheothresdef} with $G\sim\mathrm{Gamma}(\tfrac{1}{2},2)$.
\end{corollary}

\begin{figure}
	\centerline{\scalebox{0.8}{
			\includegraphics{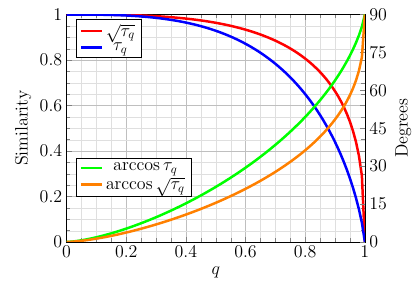}}}\vspace{-5pt}
	\caption{Asymptotic similarity between $\mathbf{w}$ and its magnitude-pruned version as a function of the pruning rate $q$.}\vspace{-6pt}
	\label{fig:concentration}
\end{figure}

Thus, one-shot magnitude pruning preserves a deterministic amount of similarity with the original neuron in high dimension: both $\prunedvec$ and its renormalized version $\renormprunedvec$ have non-vanishing inner-product similarity with $\weightvec$ as $n$ grows, and their inner products converge to $\tau_q$ and $\sqrt{\tau_q}$, respectively. Moreover, this limiting similarity is a non-linear function of the pruning rate $q$ and remains close to one even under aggressive pruning. For instance, at $q=0.4$ (keeping $60\%$ of the coordinates), we obtain $\weightvec^{\top}\renormprunedvec\approx 0.982$, corresponding to an angle of about $10.8^{\circ}$; see also Fig. \ref{fig:concentration}. In particular, one-shot pruning induces only a small directional distortion at the single-neuron level in the high-dimensional regime, offering a first-principles explanation for its strong empirical performance reported in prior work. 

The concentration framework extends to blockwise semi-structured $N\!:\!M$ pruning \cite{frantar2023sparsegpt, liu2026armor}. We derive the limiting similarities in Appendix \ref{sec:semistructured_nm} in closed form. We now turn to the adaptive regime for a single neuron.



%% file: tex/adaptive-neurons.tex
\section{Adaptive Neurons}

\label{sec:adaptive-neurons}

We now move from \emph{static} compute reduction (fixed for all inputs) to an
\emph{adaptive} regime at the single neuron level. The key idea is to first evaluate a cheap
\emph{partial} local field and use its magnitude as a confidence test. If the partial field is
confident, we exit early; otherwise, we fall back to full computation. This section formalizes
this mechanism via the conditional perceptron and studies its compute--generalization tradeoff in a
student--teacher setting.

\subsection{Conditional Perceptrons}
\label{subsec:cond-perceptrons}

\newcommand{\exitthres}{\lambda}               
\newcommand{\act}{\operatorname{sgn}}          

Consider an input $\inputvec\in\mathbb{R}^n$ and weights
$\weightvec\in\mathbb{R}^n$.
We use the same one-shot magnitude pruning and renormalization as in
Section~\ref{sec:static-neurons}. For a pruning rate $q\in(0,1)$, let $k=\lceil qn\rceil$ and let
$\mathbf{w}_{\mathrm{p}}$ be the magnitude-pruned version of $\mathbf{w}$ obtained by zeroing the
$k$ smallest-magnitude coordinates. In this section, it will be convenient to work with the renormalization $\renormprunedvec = \|\weightvec\| \prunedvec / \|\prunedvec\|$, which we call the \textit{early-exit} vector.

The full and partial local fields are given by $	v_0 \triangleq \weightvec^{\top}\mathbf{x}$ and $v_1 \triangleq \renormprunedvec^\top\mathbf{x}$, respectively. The neuron first evaluates the partial field $v_1$ and uses $|v_1|$ as a confidence signal to decide
whether to exit early. Concretely, the output is given by
\begin{align}
	\label{eq:condpercep_def_new}
	y_{\mathrm{c}}
	\triangleq
	\begin{cases}
		\act(v_1), & |v_1| \ge \exitthres,\\
		\act(v_0), & |v_1| < \exitthres,
	\end{cases}
\end{align}
where $\exitthres>0$ is the exit threshold. We refer to \eqref{eq:condpercep_def_new} as a
\emph{conditional perceptron} since the output is \emph{conditioned} on the magnitude of the partial
local field $v_1$: when $|v_1|\ge \exitthres$ we return the early-exit decision $\act(v_1)$, and
otherwise we include the remaining $k$ coordinates and fall back to the full decision $\act(v_0)$. The ordinary (unconditional) perceptron is recovered in the limit $\lambda\to\infty$.

Computing $v_1$ uses $(n-k)$ multiplications and $(n-k-1)$ additions. For $k=0$ and including one FLOP for the activation, 
an unconditional perceptron spends $\mu_{\mathrm{uc}}\triangleq 2n$ FLOPs for every input. On the other hand, with one comparison for
$|v_1|\ge \exitthres$ and one FLOP for the activation, the early-exit branch costs $2(n-k)+1$ FLOPs.
If we do not exit, we must additionally incorporate the pruned coordinates to obtain $v_0$ from $v_1$
via $v_0
	=( \|\prunedvec\|/\|\weightvec\|) v_1
	+ \sum_{j=1}^{k} w_{i_j}x_{i_j}$, 
which costs $1+2k$ more FLOPs, plus one FLOP for the final activation. Thus, the conditional perceptron spends
\begin{align}
	\label{eq:condpercep_flops_new}
	\mu_{\mathrm{c}}
	\triangleq
	\begin{cases}
		2(n-k)+1, & |v_1| \ge \exitthres,\\
		2n+2, & |v_1| < \exitthres.
	\end{cases}
\end{align}
On an ensemble of inputs for which the early-exit condition triggers frequently, the average compute is strictly below the unconditional baseline. At the same time, one expects the early-exit and full decisions to agree for most inputs, since the omitted coordinates have small magnitude.

\begin{example}
	Let $n=3$ and $\mathbf{w}=[\sqrt{21}\ \ -6\ \ 8]^T$ with $\|\mathbf{w}\|=11$. For $q=1/3$ we have
	$k=\lceil qn\rceil=1$, so  pruning yields $\prunedvec=[0\ \ -6\ \ 8]^T$ and
	$\eevec=(11/10)\mathbf{w}_{\mathrm{p}}=[0\ \ -6.6\ \ 8.8]^T$. The partial field
	is $v_1=-6.6x_2+8.8x_3$, which costs $3$ FLOPs. We check $|v_1|\ge \exitthres$ (1 FLOP) and
	apply $\act(\cdot)$ (1 FLOP), totaling $5$ FLOPs if we exit. Otherwise, we compute
	$v_0=(10/11)v_1+\sqrt{21}\,x_1$ (3 more FLOPs) and apply $\act(\cdot)$ (1 FLOP), totaling $9$ FLOPs.
\end{example}

\newcounter{constgammauncondperceptc}\setcounter{constgammauncondperceptc}{\value{cc}}
\newcommand{\constgammauncondpercept}{C_{\theconstgammauncondperceptc}}\addtocounter{cc}{1}


\subsection{Learning on the Unconditional Perceptron}
\label{sec:learningontheordinaryperceptron}
Consider a dataset of $N_t$ training vectors $\{\trainvec_1,\ldots,\trainvec_{N_t}\} \subset \mathbb{R}^n$ and a teacher $\teachvec \in \mathbb{R}^n$. The teacher assigns labels $z_i \in \{-1,+1\}$ via
$z_i \triangleq \act(\teachvec^{\top} \trainvec_i)$ for $i\in\{1,\ldots,N_t\}$.

We consider now a student $\weightvec\in\mathbb{R}^{n}$ acquired through some learning algorithm, as a function of only the input-output pairs $(\trainvec_1,z_1),\ldots,(\trainvec_{N_t},z_{N_t})$. The specific learning algorithm to obtain $\weightvec$ out of the training vectors is not vital for our purposes. For example, the student can be chosen to be the vector that  classifies the training data with the maximal margin, i.e. $\maxmarginvec \triangleq \arg\max_{\weightvec_0\in\mathbb{S}^{n-1}}\min_{i\in\{1,\ldots,{N_t}\}} z_i\weightvec_0^{\top}\trainvec_i$. The generalization error provided by any student $\weightvec$ is the probability $\epsilon_{\mathrm{uc}}(\weightvec,\teachvec) \triangleq \prob_{\inputvec}(\act(\inputvec^{\top}\weightvec) \neq \act(\inputvec^{\top}\teachvec))$ of mismatch between the student and teacher decisions when the input $\inputvec$ is assumed uniform on $\mathbb{S}^{n-1}$. By rotational
invariance, this depends only on the angle between $\weightvec$ and $\teachvec$. Defining the cosine
similarity $\mathrm{sim}(\weightvec,\teachvec)\triangleq \weightvec^{\top}\teachvec / (\|\weightvec\|\,\|\teachvec\|)$, 
we have $\epsilon_{\mathrm{uc}}(\weightvec,\teachvec) = \frac{1}{\pi}\arccos \mathrm{sim}(\weightvec,\teachvec)$. 

We are often interested in the error averaged over random datasets and teachers. For instance, if
$\trainvec_1,\ldots,\trainvec_{N_t},\teachvec\sim\mathcal{N}(\mathbf{0}_n,\mathbf{I}_n)$ are
independent, then the max-margin classifier $\maxmarginvec$ yields the average error
$\mathbb{E}[\epsilon_{\mathrm{uc}}(\maxmarginvec,\teachvec)]$. Exact evaluation is generally hard, but
in the limit $n,N_t\to\infty$ with $\alpha\triangleq N_t/n$ fixed, the asymptotic error
$\epsilon_{\mathrm{uc}}$ is well-studied. In particular, as $\alpha\to\infty$ it
obeys $\epsilon_{\mathrm{uc}} \sim \constgammauncondpercept/\alpha$ for a constant $\constgammauncondpercept>0$
\cite{sorscher2022beyond,opper1995statistical}.

\subsection{Analyzing the Conditional Perceptron}
\label{sec:analcp}
Consider the student max-margin classifier $\maxmarginvec$ induced by a dataset
$\{\trainvec_1,\ldots,\trainvec_{N_t}\}$ and a teacher $\teachvec$, as in
Section~\ref{sec:learningontheordinaryperceptron}. We use $\weightvec=\maxmarginvec$ as the main
weight vector of the conditional perceptron in Section~\ref{subsec:cond-perceptrons} with sparsity
$q$. A generalization error occurs if the
conditional output \eqref{eq:condpercep_def_new} disagrees with the teacher label
$\act(\teachvec^{\top}\inputvec)$ on an input $\inputvec$.

Averaging first over the test input and then over the randomness of the dataset and teacher, we
obtain the (average) conditional generalization error. Let $\epsilon_{\mathrm{c}}$ denote its
$n,N_t\to\infty$ limit with $\alpha= N_t/n$ fixed. Likewise, averaging 
\eqref{eq:condpercep_flops_new} (with $\weightvec=\maxmarginvec$) over the same randomness, let
$\widehat{\mu}_{\mathrm{c}}$ denote the average FLOPs of the conditional perceptron. It is also convenient to normalize by the FLOP cost of an unconditional perceptron, which was calculated as $\mu_{\mathrm{uc}}= 2n$ in Section \ref{subsec:cond-perceptrons}. This yields 
$\overline{\mu}_{\mathrm{c}} \triangleq \lim_{n\to\infty}\widehat{\mu}_{\mathrm{c}}/(2n)$.
Then, $\overline{\mu}_{\mathrm{c}}=1$ corresponds to the cost of the unconditional perceptron. The following theorem analyzes the compute-error tradeoff of this model.

\begin{theorem}
	\label{theorem:tradeoff}
	For a given computation constraint $\overline{\mu}_{\mathrm{c}} \le 1-\epsilon$, where $\epsilon>0$,
	a generalization error of $\epsilon_{\mathrm{c}} \le
\epsilon_{\mathrm{uc}}+(2\epsilon/q)^{A}$ is achievable, where
	$A = 0.5/(1-\rho^2)$ and $\rho \triangleq
	\cos(\arccos\sqrt{\tau_q}+\pi\epsilon_{\mathrm{uc}})$.
\end{theorem}

Now, fix some ``reasonable'' sparsity rate $q$ so that $\rho$ is not too far from zero, and imagine $\epsilon$ as the only variable. Theorem \ref{theorem:tradeoff} shows that as the normalized FLOPs approach $1$, the system's performance approaches that of the unconditional perceptron at a polynomial rate, specifically $O(\epsilon^A)$ for some  $A > 1$, where $\epsilon$ denotes the compute gap. Hence, a nontrivial reduction in average compute can incur only a modest increase in error. Moreover, $A \rightarrow \infty$ as $\rho\to
1$, so the approach to full-compute performance becomes increasingly sharp in regimes where pruning
preserves strong alignment (large $\tau_q$) and the unconditional student already generalizes well (small $\epsilon_{\mathrm{uc}}$).

 Together, these bounds provide a theoretical mechanism behind a commonly observed
behavior in early-exit systems: accuracy can remain close to the full-compute baseline over a wide
range of compute budgets, and degrades noticeably only once the budget becomes too tight.
This qualitative trend is consistent with our pruning experiments in Fig.~1, where similarity remains
high across moderate sparsity values and only drops under aggressive sparsification.

Theorem~\ref{theorem:tradeoff} should be interpreted as an explicit
achievability bound and scaling law, not as a tight characterization of the
conditional perceptron. A sharper result would likely require directly
characterizing the overlap
$\langle \widetilde{\mathbf{w}}_p,\mathbf{t}\rangle$ between the
normalized pruned student and the teacher. This is difficult because
$\widetilde{\mathbf{w}}_p$ is obtained by pruning the learned student weights
and then renormalizing, so its support depends on student coordinates that are
already correlated with the teacher through the training data. Thus, one would
need to control the joint law of the student coordinates, teacher coordinates,
and the order-statistic support selected by magnitude pruning. Our proof
avoids this joint analysis by separating the teacher--student overlap from the
student--pruned-student overlap through an angular triangle inequality.


%% file: tex/static-networks.tex

\section{Static Networks}
\label{sec:static-networks}

In Sections~\ref{sec:static-neurons} and \ref{sec:adaptive-neurons}, we focused on a single neuron in
order to obtain an asymptotic characterization of neuronwise magnitude pruning and
early exit. We now turn to fully connected deep networks. The key new issue is depth: Even when
pruning induces only a mild distortion at a single layer, this distortion can compound across layers
and lead to a significant mismatch between the pruned and unpruned computations. Our goal is to
quantify this depth accumulation under the same practical pruning rule, namely neuronwise magnitude
pruning of weights without retraining.

\subsection{Pruning with Renormalization in Networks}
\label{sec:prunerenormmodel}

We consider an $L$-layer fully connected network with widths
$m_0,m_1,\ldots,m_L$, where $m_0$ is the input dimension.  For each $\ell$, the entries of the Layer $\ell$ weight matrix $\matW^{(\ell)}\in\mathbb{R}^{m_\ell\times m_{\ell-1}}$ are IID, have zero-mean, unit variance with finite fourth moments. We also consider readout weights $\mathbf{y}\in\mathbb{R}^{m_L}$ that are zero-mean, IID, and independent of the entries of $\matW^{(\ell)}$ for every $\ell$. 

For an input
$\inputvec\in\mathbb{R}^{m_0}$, the forward pass is defined by $\avec^{(0)}(\inputvec)\triangleq
\inputvec$, and $\zvec^{(\ell)}(\inputvec)\triangleq
m_{\ell-1}^{-1/2}\matW^{(\ell)}\avec^{(\ell-1)}(\inputvec),\,\ell\in\{1,\ldots,L\}$. Here, 
$\avec^{(\ell)}(\inputvec)\triangleq \phi(\zvec^{(\ell)}(\inputvec))$, and $\phi$ is some activation function. We consider a scalar output with linear readout
$\netout(\inputvec)\triangleq m_L^{-1/2}\langle \readoutvec,\avec^{(L)}(\inputvec)\rangle$.

Fix a pruning ratio $q\in(0,1)$ and let $k_\ell\triangleq \lceil q\,m_{\ell-1}\rceil$. For each layer
$\ell$ and neuron $i\in\{1,\ldots,m_\ell\}$, we prune the  weight vector
$\smash{\matW^{(\ell)}_{i,:}\in\mathbb{R}^{m_{\ell-1}}}$ by retaining only the $(1-q)$ fraction of indices
with largest magnitudes. Let $\smash{\matM^{(\ell)}_{i,:}\in\{0,1\}^{m_{\ell-1}}}$ be the resulting mask and
define the pruned weights by $\prunedW^{(\ell)}_{ij}\triangleq \matW^{(\ell)}_{ij}\matM^{(\ell)}_{ij}$. We consider a renormalized version that removes the purely multiplicative energy loss induced by
pruning. Let $\sqrt{\tau_q}\in(0,1)$ be the asymptotic cosine similarity between an isotropic Gaussian vector and its pruned counterpart as evaluated in  Corollary~\ref{corol:concentration}. We define $\renormW^{(\ell)}\triangleq \prunedW^{(\ell)}/\sqrt{\smash[b]{\tau_q}}$, and let
$\widehat{\avec}^{(\ell)}(\inputvec)$ and $\renormnetout(\inputvec)$ denote the layer-$\ell$
intermediate feature vector and the network output obtained by replacing $\matW^{(\ell)}$ with
$\renormW^{(\ell)}$ in the forward pass, respectively.

Our motivation for renormalization is to separate scale loss from directional
mismatch. Pruning changes both the norm and direction of each weight vector.
Without correction, activations are attenuated across layers, so degradation may
simply reflect shrinking signal energy. Renormalization is not needed for the
wide-limit analysis itself; it restores the original second moment and isolates
the directional mismatch caused by pruning. Under neuronwise magnitude pruning,
this is also the natural way to preserve signal scale and alignment across
layers. Appendix~\ref{app:nonrenorm_static_networks} gives the corresponding
non-renormalized recursion and comparison.

\subsection{Main Result with Applications}

\newcommand{\kernm}{\mathbf{K}}              
\newcommand{\rpkernm}{\widehat{\mathbf{K}}} 
\newcommand{\ckernm}{\mathbf{C}}            

Consider arbitrary inputs $\inputvec_1,\ldots,\inputvec_r$ to the neural network. 
For each layer $\ell\in\{0,\ldots,L\}$ and $i,j\in\{1,\ldots,r\}$, define the empirical activation kernels
$\kernm^{(\ell)}_{ij}\triangleq m_\ell^{-1}\langle \avec^{(\ell)}(\inputvec_i),\avec^{(\ell)}(\inputvec_j)\rangle$ and
$\rpkernm^{(\ell)}_{ij}\triangleq m_\ell^{-1}\langle \widehat{\avec}^{(\ell)}(\inputvec_i),\widehat{\avec}^{(\ell)}(\inputvec_j)\rangle$,
as well as the cross-kernel
$\ckernm^{(\ell)}_{ij}\triangleq m_\ell^{-1}\langle \avec^{(\ell)}(\inputvec_i),\widehat{\avec}^{(\ell)}(\inputvec_j)\rangle$.
Let $\kernm^{(\ell)},\rpkernm^{(\ell)},\ckernm^{(\ell)}\in\mathbb{R}^{r\times r}$ denote the corresponding Gram matrices.

We will use the standard Gaussian kernel operator associated with $\phi$. For any positive
semidefinite matrix $\Sigma=\begin{psmallmatrix}u & v\\ v & u'\end{psmallmatrix}$, define
$\mathcal{T}_{\phi}(\Sigma)\triangleq \E[\phi(U)\phi(V)]$ where $(U,V)\sim\mathcal{N}(0,\Sigma)$. Assume an activation $\phi$ normalized so that for $Z\sim\mathcal{N}(0,1)$,
\begin{align}
	\E[\phi(Z)]\!=\!0,
	\E[\phi(Z)^2]\!=\!1,\,
	\chi_{\phi}\triangleq \E[Z\phi(Z)]\!\in\![0,1].
	\label{eq:chi_def}
\end{align}

\begin{theorem}
	\label{theorem:deep_general_phi}
	Fix some $q\in(0,1)$ and consider the inputs and random network as described in Section~\ref{sec:prunerenormmodel}. 
	Consider the joint limit $m_1,\ldots,m_L\to\infty$ with $L$ fixed. With probability tending to one,
	the empirical kernels and cross-kernels converge pointwise to deterministic limits
	$\{K_\infty^{(\ell)}\}_{\ell=0}^L$, $\{\widehat{K}_\infty^{(\ell)}\}_{\ell=0}^L$, and $\{C_\infty^{(\ell)}\}_{\ell=0}^L$, whose entries are defined below.
	Specifically, for $i,j\in\{1,\ldots,r\}$, we have $
	K_{\infty,ij}^{(0)} = \widehat{K}_{\infty,ij}^{(0)} = C_{\infty,ij}^{(0)} = \frac{1}{m_0}\langle \inputvec_i,\inputvec_j\rangle$, 
	and for $\ell\in\{1,\ldots,L\}$, we have the recursions
	\begin{align}
		\widehat{K}_{\infty,ij}^{(\ell)} & = K_{\infty,ij}^{(\ell)}  =
		\mathcal{T}_\phi\!\Biggl(
		\begin{bmatrix}
			K_{\infty,ii}^{(\ell-1)} & K_{\infty,ij}^{(\ell-1)}\\
			K_{\infty,ji}^{(\ell-1)} & K_{\infty,jj}^{(\ell-1)}
		\end{bmatrix}
		\Biggr), \label{eq:Krec_pointwise} \\
		C_{\infty,ij}^{(\ell)} & =
		\mathcal{T}_\phi\!\Biggl(
		\begin{bmatrix}
			K_{\infty,ii}^{(\ell-1)} & \sqrt{\tau_q}\,C_{\infty,ij}^{(\ell-1)}\\
			\sqrt{\tau_q}\,C_{\infty,ji}^{(\ell-1)} & \widehat{K}_{\infty,jj}^{(\ell-1)}
		\end{bmatrix}
		\Biggr). \label{eq:Crec_pointwise}
	\end{align}

	For the outputs, we have $\mathrm{Cov}\bigl(f(\inputvec_i),f(\inputvec_j)\bigr)  \to K_{\infty,ij}^{(L)}$ and  $\mathrm{Cov}\bigl(\widehat{f}(\inputvec_i),\widehat{f}(\inputvec_j)\bigr) \to \widehat{K}_{\infty,ij}^{(L)}$. The  cross-covariance satisfies $\mathrm{Cov}\bigl(f(\inputvec_i),\widehat{f}(\inputvec_j)\bigr)\to C_{\infty,ij}^{(L)}$. Consequently, for each fixed input index $i$, the asymptotic MSE satisfies
	\begin{align}
		\mathrm{E}\bigl[(f(\inputvec_i)-\widehat{f}(\inputvec_i))^2\bigr]
		\to
		2(K_{\infty,ii}^{(L)}-C_{\infty,ii}^{(L)}).
		\label{eq:mse_general_deep}
	\end{align}

\end{theorem}
Here, we provide an example application to two widely-used activation functions. Assume the inputs satisfy $\frac{1}{m_0}\|\inputvec_i\|^2=1,\, \forall i\in\{1,\ldots,r\}$. Suppose the activation is also normalized such that \eqref{eq:chi_def} holds. Then, by definition, $K_{\infty,ii}^{(\ell)}=1$ for all $\ell=0,1,\ldots,L$, and by \eqref{eq:mse_general_deep}, the asymptotic MSE is $2(1-\alpha_{L})$, where $\alpha_\ell \triangleq C_{\infty,ii}^{(\ell)}$ with $\alpha_0\triangleq 1$. In the next two examples, we adapt classical Gaussian kernel identities for $\sgn$ and (normalized) ReLU to our setting \eqref{eq:Crec_pointwise}. We provide the
derivations Appendices \ref{app:proof_deep_sign_mse} and \ref{app:proof_deep_nrelu_mse} for completeness; see also, e.g.,
\cite{williams1997computing} for the arcsine law and \cite{cho2009kernel} for the ReLU
kernel.

\begin{example}
	\label{cor:deep_sign_mse}
	Let $\phi(t)=\act(t)$. 	Then, \eqref{eq:chi_def} is satisfied and  $\alpha_\ell = \frac{2}{\pi}\arcsin(\sqrt{\tau_q}\,\alpha_{\ell-1}),\,\ell \geq 1$. 
\end{example}

\begin{example}
	\label{cor:deep_nrelu_mse}
	Let $\phi(t)=\sqrt{2}\max\{t,0\}$ be the normalized ReLU. The factor $\sqrt{2}$ keeps the self-kernel stable, so that
	$K_{\infty,ii}^{(\ell)}=1$ for all $\ell$ in the unpruned recursion \eqref{eq:Krec_pointwise}.
	Then, \eqref{eq:chi_def} holds, and for $\ell \geq 1$, we have $\alpha_\ell = \kappa_{\rm R}(\sqrt{\tau_q}\,\alpha_{\ell-1})$, where 
	$\kappa_{\rm R}(\rho) \triangleq \tfrac{1}{\pi}(\sqrt{1-\rho^2}+\rho(\pi-\arccos\rho)),
	\, \rho\in[-1,1]$.
\end{example}



Theorem~\ref{theorem:deep_general_phi} provides a Gaussian-limit description of neuronwise magnitude pruning in fully connected deep networks, without retraining. Relative to the single-neuron analysis in Sections~\ref{sec:static-neurons} and \ref{sec:adaptive-neurons}, the new phenomenon is depth: even when pruning causes only a mild mismatch at one layer, the mismatch can propagate through the nonlinearity and compound across layers.

A first consequence of the theorem is that renormalization cleanly removes the trivial scale attenuation caused by pruning, as intended. In the wide limit, the renormalized-pruned and unpruned networks share the same limiting self-kernels at every layer, since $\widehat{K}_{\infty,ij}^{(\ell)}=K_{\infty,ij}^{(\ell)}$ in \eqref{eq:Krec_pointwise}. This means that any degradation in the renormalized pruned computation is not due to a shrinking signal energy, but due to a loss of alignment between the two forward passes.\footnote{The same Gaussian-limit argument also applies without renormalization; in that case the self-kernel and cross-kernel recursions acquire extra factors of $\tau_q$, reflecting the fact that each pruned row has asymptotic squared norm and unnormalized overlap equal to $\tau_q$. We state this variant explicitly in Appendix~\ref{app:nonrenorm_static_networks}.} This loss is captured entirely by the cross-kernel recursion \eqref{eq:Crec_pointwise}. The only difference relative to the unpruned recursion is the factor $\sqrt{\tau_q}$ multiplying the cross terms. This is the same overlap parameter that appeared in the neuron-level results, where pruning reduces the cosine similarity between a weight vector and its pruned version. In deep networks, the overlap loss is injected at each layer and then transformed by the activation through $\mathcal{T}_\phi$. As a result, the output distortion is governed by the gap between the self-kernel and the cross-kernel, and the limiting MSE takes the simple form \eqref{eq:mse_general_deep}.

Under the common unit self-kernel normalization (for instance, when $\frac{1}{m_0}\|\inputvec_i\|^2=1$ and $\phi$ is variance-preserving with $\sigma_w=1$), we have $K_{\infty,ii}^{(\ell)}=1$ for all $\ell$, and \eqref{eq:mse_general_deep} reduces to an alignment deficit of the form $2(1-\alpha_L)$, where $\alpha_\ell \triangleq C_{\infty,ii}^{(\ell)}$. The two corollaries make this reduction explicit for standard activations. For $\phi(t)=\act(t)$, Corollary~\ref{cor:deep_sign_mse} yields the one-dimensional recursion
$\alpha_\ell=\frac{2}{\pi}\arcsin(\sqrt{\tau_q}\,\alpha_{\ell-1})$, which is highly sensitive near perfect alignment. For normalized ReLU, Corollary~\ref{cor:deep_nrelu_mse} yields
$\alpha_\ell=\kappa_{\rm R}(\sqrt{\tau_q}\,\alpha_{\ell-1})$ with a smoother map, implying a milder per-layer loss of alignment. In this sense, the analysis explains how the same pruning ratio can lead to qualitatively different depth accumulation depending on the nonlinearity, even after removing scale effects via renormalization. We validate these predictions numerically in Section~\ref{sec:numerical}.

%% file: tex/adaptive-networks.tex
\section{Adaptive Networks}

\label{sec:adaptive-networks}

\newcommand{\Var}{\operatorname{Var}}
\newcommand{\Cov}{\operatorname{Cov}}
\newcommand{\Corr}{\operatorname{Corr}}

Section~\ref{sec:adaptive-neurons} provides an \emph{exact} characterization of early exit at the single-neuron level in a standard high-dimensional teacher--student setting. Extending such explicit guarantees to deep networks is substantially more challenging, since both the exit and final scores are induced by high-dimensional intermediate representations that are strongly coupled across layers and shaped by training. We therefore proceed in two steps. First, we derive a deep compute--accuracy tradeoff under the assumption that the class-standardized exit and final scores are \emph{dominated} in their tails by a bivariate Gaussian, so that the tradeoff is governed by a single within-class coupling parameter. Second, to obtain \emph{exact} and explicit expressions for this coupling in a canonical regime, we specialize to the infinite-width neural network Gaussian process (NNGP) limit, where a frozen random backbone yields deterministic correlation recursions and closed-form tradeoff curves. This also captures a practically relevant \emph{frozen-backbone} setting, in which the base network is kept fixed and only an intermediate exit head is trained, mirroring common post hoc early-exit protocols for pretrained models.


\subsection{Early-Exit Model}
\label{subsec:adaptive_networks_model}
For large pretrained transformers, the most practical early-exit deployments typically avoid end-to-end retraining of
the full backbone. Updating all layers can be prohibitively expensive in compute and engineering, and may be infeasible
when the backbone is provided as a fixed, shared model. A common alternative is to keep the backbone frozen and train
only lightweight exit heads on intermediate representations, calibrating their thresholds to meet a target average
compute budget. Our analysis is tailored to this regime. 

The backbone is frozen. We attach an exit head at an
intermediate layer $\ell\!\in\!\{1,\ldots,L\!-\!1\}$, and train only this exit head.  We use the same layer
notation as in Section~\ref{sec:static-networks}. In particular, $\avec^{(\ell)}(\inputvec)$ denotes the
layer-$\ell$ feature vector produced by the frozen backbone on input $\inputvec$. Recall that the full network score is the linear readout
$f_L(\inputvec)\triangleq f(\inputvec)
	=
	m_L^{-1/2}\big\langle \readoutvec,\avec^{(L)}(\inputvec)\big\rangle$. We take the exit head at layer $\ell$ to be linear and focus on distillation to the final score. Define the distilled exit head as the best affine predictor of the final score from the layer-$\ell$ representation:
	\begin{align}
		(\boldsymbol{w}_{\ell}^{\star}, b_{\ell}^{\star})
	\!	\in\!
		\arg\min_{\boldsymbol{w},\,b}\!\ 
		\E_{\inputvec}\Bigl[\bigl(\boldsymbol{w}^{\top}\!\avec^{(\ell)}(\inputvec)\!+\!b\!-\!f_L(\inputvec)\bigr)^2\Bigr],
		\label{eq:distilldef}
	\end{align}
	and set $
	f_{\ell}(\inputvec)\triangleq (\boldsymbol{w}_{\ell}^{\star})^{\top}\avec^{(\ell)}(\inputvec)+b_{\ell}^{\star}$. 
	
	Let $Y\in\{-1,+1\}$ denote the
ground-truth label for $\inputvec$. Given a threshold $\netexitthres\ge 0$, the early-exit predictor is
\begin{align}
	\widehat{Y}_{\mathrm{c}}(\inputvec)
	\triangleq
	\begin{cases}
		\act(f_{\ell}(\inputvec)), & |f_{\ell}(\inputvec)|\ge \netexitthres,\\
		\act(f_L(\inputvec)), & |f_{\ell}(\inputvec)|< \netexitthres,
	\end{cases}
	\label{eq:deep_conditional_rule_generalization}
\end{align}
with $\epsilon_{\mathrm{c}} \triangleq \P_{\inputvec}\big(\widehat{Y}_{\mathrm{c}}(\inputvec) \neq Y\big)$. The full-compute baseline is $\widehat{Y}_{\mathrm{uc}}(\inputvec)\triangleq \act(f_L(\inputvec))$ with
$\epsilon_{\mathrm{uc}}\triangleq \P_{\inputvec}\big(\widehat{Y}_{\mathrm{uc}}(\inputvec)\neq Y\big)$. 

Although the exit head in (\ref{eq:distilldef}) is trained by least-squares
distillation, the downstream task considered here is still binary
classification: the final predictor is $\operatorname{sgn}(f_L(x))$, and
the exit predictor is $\operatorname{sgn}(f_\ell(x))$. Thus, the regression
view is used only to fit an intermediate score that approximates the final
score. The magnitude $|f_\ell(x)|$ is then a margin-like confidence statistic
for the induced binary classification decision, rather than a confidence
measure for regression accuracy itself. In practical early-exit systems, the
confidence statistic can be replaced by task-specific alternatives, such as
maximum softmax probability or top-two logit margin. The choice $|f_\ell(x)|$ is the binary-score analogue of these
criteria and keeps the analysis one-dimensional.

Let $\kappa_{\ell}\in(0,1)$ denote the normalized cost of evaluating up to layer $\ell$
(including the exit head), relative to the full-network cost $1$. The expected normalized compute is
\begin{align}
		\label{eq:deep_compute_mu}\mu_{\mathrm{c}}'(\netexitthres)
	&=
	\kappa_{\ell}\,\P_{\inputvec}\!\big(|f_{\ell}(\inputvec)|\ge\netexitthres\big)
	+
	\P_{\inputvec}\!\big(|f_{\ell}(\inputvec)|<\netexitthres\big).
\end{align}

In practice, given $\ell$ and target compute budget, we set $\lambda$ by calibration on a held-out validation set so that
the empirical exit rate matches the desired budget. This reduces tuning to a low-dimensional search over
candidate exits $\ell$, with $\lambda$ determined by quantile calibration for each candidate. In the
conditional perceptron model, this optimization admits a closed-form solution, which we report in the appendix.

\subsection{Deep compute--accuracy tradeoffs}
\label{subsec:deep_tradeoff}

Let
$Y\in\{-1,+1\}$ denote the ground-truth label. For each layer index $\ell\in\{1,\ldots,L\}$ and class
$y\in\{-1,+1\}$, define the class-conditional mean and variance $\mu_{\ell,y}\triangleq \E_{\inputvec}[f_\ell(\inputvec)\mid Y=y]$, 
	$\sigma_{\ell,y}^2\triangleq \Var_{\inputvec}(f_\ell(\inputvec)\mid Y=y)$,
and define the class-standardized score by
$\tilde f_{\ell,y}(\inputvec)
	\triangleq (f_\ell(\inputvec)-\mu_{\ell,y})/\sigma_{\ell,y}$ for $ \ell=1,\ldots,L$ and  $y\in\{-1,+1\}$.
By construction, $\E[\tilde f_{\ell,y}\mid Y=y]=0$ and $\Var(\tilde f_{\ell,y}\mid Y=y)=1$.

Our analysis uses within-class Gaussian-type \emph{upper bounds} on the joint tails of the standardized exit and final
scores. After centering and scaling within each class, we assume the exit statistic has sub-Gaussian tails, and that
the standardized final score has a linear conditional mean given the exit score, with a sub-Gaussian residual. The
scalar parameter $\rho_\ell$ is the corresponding within-class regression coefficient, equivalently the within-class
correlation between the standardized exit and final scores, and it is the sole alignment parameter that governs the
compute--accuracy exponent.

\begin{assumption}
	\label{assump:subg_coupling}
	Fix $\ell\in\{1,\ldots,L-1\}$ and a class $y\in\{-1,+1\}$. Let
	$U\triangleq \tilde f_{\ell,y}(X)$ and $V\triangleq \tilde f_{L,y}(X)$ denote the class-standardized exit and final scores,
	so that $\E[U\mid Y=y]=\E[V\mid Y=y]=0$ and $\Var(U\mid Y=y)=\Var(V\mid Y=y)=1$. Assume the following holds for both classes $y\in\{-1,+1\}$ with a common $\rho_\ell\in[-1,1]$:
	
{\bf 1.}  Sub-Gaussian exit statistic. For all $t\ge 0$, we have
		\begin{align}
			\Pr\!\big(|U|\ge t \mid Y=y\big)\le 2e^{-t^{2}/2}.
			\label{eq:subg_tail_U}
			\end{align}
{\bf 2.} Linear prediction with sub-Gaussian residual. The conditional mean is linear,
		$\E[V\mid U,Y=y]=\rho_\ell\,U$, and the residual $R\triangleq V-\rho_\ell U$ is conditionally sub-Gaussian with variance
		$1-\rho_\ell^{2}$, namely for all $\lambda\in\mathbb{R}$,
		\begin{align}
			\E\!\left[\exp\!\left(\lambda R\right)\mid U,Y\!=\! y\right]\le \exp(\tfrac{\lambda^{2}(1-\rho_\ell^{2})}{2}).
			\label{eq:subg_mgf_R}
		\end{align}
\end{assumption}
Such sub-Gaussian tail and exponential moment bounds are standard abstractions in high-dimensional probability
and statistics \cite{vershynin2018high,wainwright2019hds}. We also later provide empirical diagnostics, checking approximate linearity of $\E[V\mid U]$ after
within-class standardization and light-tailed behavior of the regression residual. 

The specific numerical constants that appear in \eqref{eq:subg_tail_U} and \eqref{eq:subg_mgf_R} are inessential for our purposes: In fact, if these inequalities are weakened by constant factors, only the prefactor in the final tradeoff changes, while the exponent remains governed by $1-\rho_\ell^2$. Informally, within each class the tails of
$(U,V)$ are no heavier than those of a centered bivariate Gaussian pair with correlation $\rho_\ell$.

Assumption~\ref{assump:subg_coupling} should therefore be read as a tail-control surrogate rather than an exact Gaussian model for every trained network. The proof of Theorem~\ref{thm:deep_tradeoff_explicit} uses it only to upper bound the exit tail and the probability that the exit and final decisions disagree among high-confidence exits. If the conditional mean is only approximately linear, or if the residual tail bound holds with different constants over the relevant range, the constants in the final bound change, but the qualitative prediction is unchanged: the excess error decays as a power of the compute gap, and the decay becomes sharper as the within-class exit-to-final alignment $\rho_\ell$ increases.


\newcommand{\deepnetexpo}{A'}

\begin{theorem}
		\label{thm:deep_tradeoff_explicit}
	For a given computation constraint $\mu_{\mathrm{c}}' \le 1-\varepsilon$, where $\varepsilon\in(0,1-\kappa_\ell)$,
	a generalization error of $\epsilon_{\mathrm{c}} \le
	\epsilon_{\mathrm{uc}}+(\frac{8\pi\varepsilon}{1-\kappa_{\ell}})^{\deepnetexpo}$ is achievable, where
	$\deepnetexpo = 0.5/(1-\rho_{\ell}^2)$.
\end{theorem}


Theorem~\ref{thm:deep_tradeoff_explicit}  is a deep analogue of Theorem~\ref{theorem:tradeoff}, which treated a single neuron.
A compute gap $\varepsilon$ forces an early exit on at least an $\varepsilon/(1-\kappa_\ell)$ fraction of test points in order to satisfy the compute constraint, and the resulting excess error decays as a power of this required exit fraction.
The exponent is governed by a single coupling parameter $\rho_\ell$, the within-class exit-to-final correlation, which we estimate in experiments by computing $\mathrm{Corr}(f_\ell,f_L)$ between the exit and final \emph{true-class} logits on the evaluation set.
In the neuron-level Theorem \ref{theorem:tradeoff} the exponent equals $A=0.5/(1-\rho^2)$, and in the deep-network theorem it becomes $\deepnetexpo=0.5/(1-\rho_\ell^2)$, increasing sharply as $\rho_\ell\to 1$.
The key difference is that the neuron-level $\rho$ is a deterministic function of the model parameters, whereas the deep-network $\rho_\ell$ depends on the frozen backbone and on the depth gap.

The parameter $\rho_\ell$ enters as a coupling coefficient in a Gaussian-type
upper bound on tail probabilities. Separately from this
surrogate parameter, one can define the \emph{true} within-class correlation between the distilled exit score and the
final score induced by the backbone and the distillation construction. For this purpose, for each width $m$, draw an i.i.d.\ Gaussian backbone $W^{(m)}$ with activation $\phi$ and scaling $1/\sqrt{m}$ as in Section~\ref{sec:static-networks}.
Let $f_{\ell,y}^{(m)}$ denote the (within-class) distillation exit regressor at layer $\ell$ defined analogously to \eqref{eq:distilldef}, and let $f_L^{(m)}$ denote the final score. Let  $\tilde{f}_{\ell,y}^{(m)}$ be the class-standardized version of $f_{\ell,y}^{(m)}$. 
Assume the activation $\phi$ is normalized in the sense of \eqref{eq:chi_def}.

\begin{proposition}
	\label{prop:rho_gap_law_quenched}
	 For each $y\in\{-1,+1\}$, define the within-class
	correlation between the standardized scores
	\begin{align}
		\tilde{\rho}_{\ell,y}^{(m)}
		\triangleq
		\E_{\inputvec}\!\big[\tilde{f}_{\ell,y}^{(m)}\tilde{f}_{L,y}^{(m)}\ \big|\ Y=y\big].
		\label{eq:rho_def_local_prop}
	\end{align}
	Then, for every $y$, as $m\to\infty$, we have
	$\tilde{\rho}_{\ell,y}^{(m)} \xrightarrow[]{\mathrm{in\ prob.}}\chi_\phi^{\,L-\ell}$, 
	where the convergence  is with respect to $W^{(m)}$.
\end{proposition}


This provides a closed-form characterization of the \emph{true} within-class alignment
between the distilled exit score and the final score in the wide random-backbone regime. In particular, the limit does not
depend on the class label $y$, and the distilled correlation concentrates to
$\widetilde{\rho}_{\ell,\infty}\triangleq \chi_\phi^{\,L-\ell}$. Plugging $\widetilde{\rho}_{\ell,\infty}$ into the exponent
in Theorem~\ref{thm:deep_tradeoff_explicit} yields the explicit depth-gap dependent value
$\deepnetexpo_{\infty}\triangleq 0.5/(1-\widetilde{\rho}_{\ell,\infty}^2)
=0.5/(1-\chi_\phi^{2(L-\ell)})$.
Thus, in the random-backbone limit the compute--accuracy exponent can be expressed explicitly as a function of the depth gap
$L-\ell$.

\begin{corollary}
	\label{cor:deep_tradeoff_gap}
	Assume the setting of Proposition~\ref{prop:rho_gap_law_quenched} and fix a compute gap
	$\varepsilon\in(0,1-\kappa_\ell)$.
	Then, as $m\to\infty$, with probability approaching $1$ over the draw of $W^{(m)}$,
	\begin{align}
		\epsilon_{\mathrm{c}}
		\le
		\epsilon_{\mathrm{uc}}
		+ [(8\pi\,\varepsilon)/(1-\kappa_\ell)]^{\deepnetexpo_{\infty}}.
		\label{eq:deep_tradeoff_gap}
	\end{align}
\end{corollary}


When $\chi_\phi\in(0,1)$, the distilled correlation decays exponentially in the depth gap,
$\widetilde{\rho}_{\ell,\infty}=\chi_\phi^{\,L-\ell}$. As the exit is moved earlier (larger $L-\ell$),
$\widetilde{\rho}_{\ell,\infty}$ decreases and the exponent
$\deepnetexpo_{\infty}=1/(2(1-\widetilde{\rho}_{\ell,\infty}^2))$ approaches $1/2$, so the compute--accuracy curve becomes less steep.
Conversely, when $\chi_\phi$ is closer to $1$ or the exit is placed near the final layers (small $L-\ell$),
$\widetilde{\rho}_{\ell,\infty}$ remains close to $1$ and the exponent becomes large. This is analogous to the low-sparsity regime
in Theorem~\ref{theorem:tradeoff}, where the coupling is close to one and the tradeoff exponent becomes large.

%

%% file: tex/numerical.tex
\section{Numerical Results}

\label{sec:numerical}

We design experiments to directly test the theory’s qualitative predictions: distortion curves under pruning and power law scaling under compute constrained early exit. Additional validations can be found in Appendix~\ref{app:more_experiments}.

	\begin{figure}[H]
	\centering
	\scalebox{0.8}{\includegraphics[width=\linewidth]{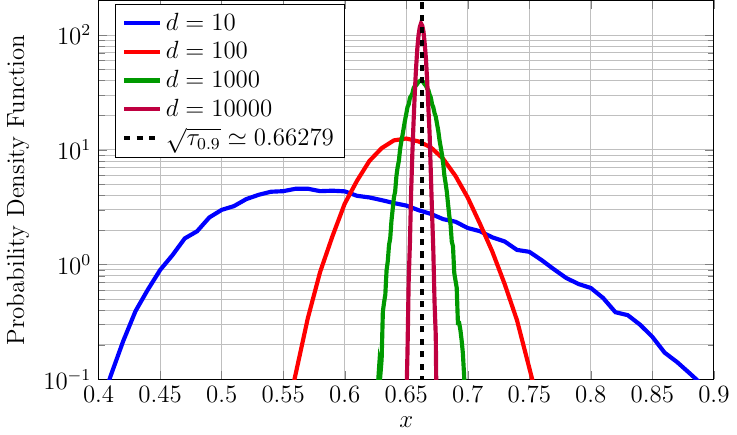}}
	\caption{Simulated PDFs for $q=0.9$.}
	\label{fig:concent_q09v1}\vspace{-10pt}
\end{figure}

{\bf Measure concentration.} Fig.~\ref{fig:concent_q09v1} verifies the concentration predicted by
Theorem~\ref{theorem:gamma} and Corollary~\ref{corol:concentration}. 
It shows simulated PDFs of $\mathbf W^\top \mathbf W_{\mathrm{ee}}$ from
$100{,}000$ trials at pruning rate $q=0.9$ and dimensions
$d\in\{10,100,1000,10000\}$. As $d$ increases, the distribution
concentrates around the theoretical limit
$\sqrt{\tau_{0.9}}\simeq 0.66279$, consistent with
Corollary~\ref{corol:concentration}.

\begin{figure*}
	\vspace{-4pt}
	\centering
	\begin{subfigure}[t]{0.32\textwidth}
		\centering
		\includegraphics[width=\linewidth]{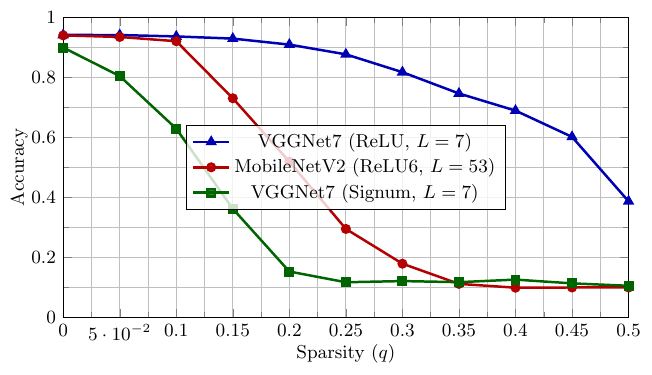}\vspace{-5pt}
		\caption{Accuracy vs sparsity $q$.}
		\label{fig:main_cifar_acc}
	\end{subfigure}
	\hfill
	\begin{subfigure}[t]{0.32\textwidth}
		\centering
		\includegraphics[width=\linewidth]{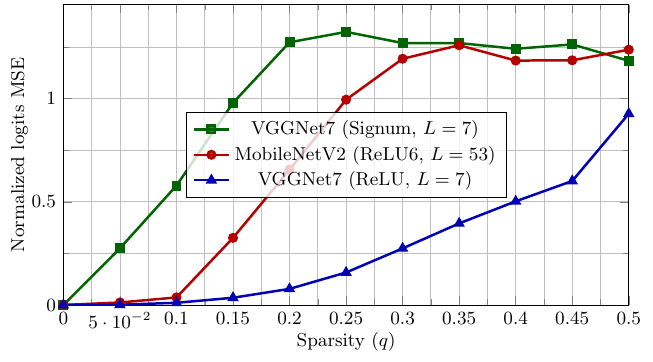}\vspace{-5pt}
		\caption{Logit NMSE vs $q$.}
		\label{fig:main_cifar_nmse}
	\end{subfigure}
	\hfill
	\begin{subfigure}[t]{0.32\textwidth}
		\centering
		\includegraphics[width=\linewidth]{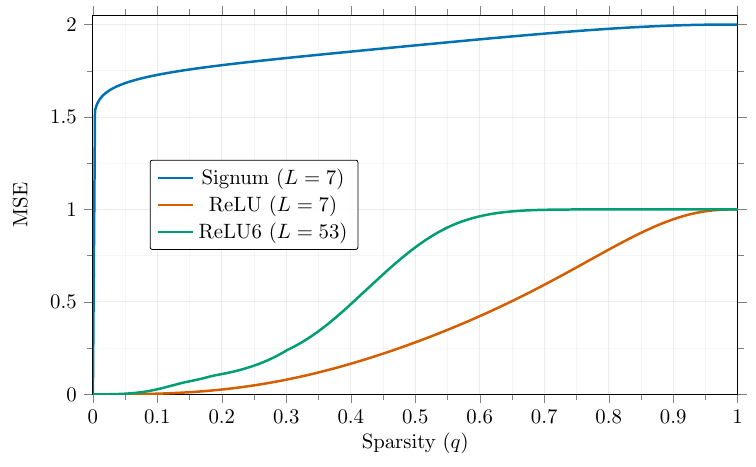}
		\vspace{-16pt}
		\caption{
			Theoretical curves from Appendix~\ref{app:cifar10_theory}.}
				\label{fig:main_cifar_anal}
	\end{subfigure}
	\vspace{-3pt}
	\caption{Numerical experiments and analytical estimates for pruning CNNs on CIFAR-10.}
	\label{fig:main_numerical}
	\vspace{-12pt}
\end{figure*}

{\bf Multi-layer pruning on CNNs.}
Figs.~\ref{fig:main_cifar_acc} and \ref{fig:main_cifar_nmse} show how unstructured magnitude pruning impacts test accuracy and a scale-free logit distortion metric. For each weight tensor, we perform magnitude pruning with sparsity $q$. We do not perform renormalization so that we are compatible with the pretrained weights. The binary signum network is highly fragile, showing rapid distortion growth and a sharp accuracy collapse even at mild sparsities.
The ReLU \texttt{VGGNet7} exhibits a broad low-distortion regime and maintains accuracy through moderate sparsity.
\texttt{MobileNetV2} (ReLU6) degrades earlier than the 7-layer ReLU model, consistent with stronger depth accumulation and saturation effects in ReLU6.
Overall, the ordering across activations matches the analytical estimate provided in Fig. \ref{fig:main_cifar_anal}. The datasets in all curves are CIFAR-10.

\begin{figure}[H]
			\centering
		\includegraphics[width=0.85\linewidth]{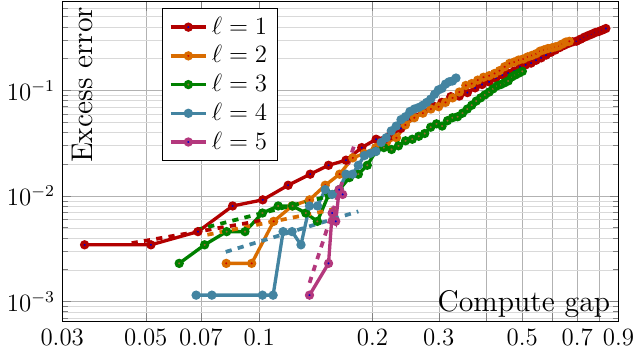}
		\caption{ Early exit on DistilBERT (SST-2).}
			\label{fig:main_xfmr_loglog}\vspace{-10pt}
	\end{figure}

{\bf Early exit on a frozen DistilBERT backbone.}
Figure~\ref{fig:main_xfmr_loglog} considers the frozen-backbone early-exit setting from Section~\ref{subsec:adaptive_networks_model}. 
We freeze a 6-block DistilBERT transformer fine-tuned for SST-2, train linear exit heads at layers $\ell \leq 5$ by distillation to the final score, and sweep the exit threshold.
On log-log axes, the excess error versus compute gap is approximately linear over a visible range, consistent with the power-law behavior in Theorem~\ref{thm:deep_tradeoff_explicit} for small compute gaps.
We also measured exit-to-final correlation $\rho_\ell=\mathrm{Corr}(f_\ell,f_L)$ as $(\rho_1,\ldots,\rho_5)=(0.446,\,0.553,\,0.669,\,0.736,\,0.977)$, 
which increases with $\ell$ as expected.

To make the comparison with Theorem~\ref{thm:deep_tradeoff_explicit} more explicit, Figure~\ref{fig:main_xfmr_loglog} now also overlays dashed line segments, in the same color as each exit curve, showing the analytical small-compute-gap log-log fits.
Specifically, for the NNGP early-exit theory, the predicted exponent at exit $\ell$ is $A'_\ell=0.5/(1-\rho_\ell^2)$. 
Using the measured correlations above gives the slopes for the dashed lines $(A'_1,\ldots,A'_5)\approx(0.62,\,0.72,\,0.91,\,1.09,\,11.00)$.
The corresponding empirical slopes fitted from the curves in Fig.~\ref{fig:main_xfmr_loglog} are
$(\widehat{A}_1,\ldots,\widehat{A}_5)\approx(1.69,\,2.28,\,2.02,\,3.46,\,11.75)$. Thus, theory and experiment show the same qualitative trend across exits: the slope increases with alignment, and both the analytical and empirical exponents single out the deepest exit as having a dramatically larger slope.

Quantifying this dependence remains challenging in the most aligned regime because the empirical excess error quickly reaches zero at finite test-set resolution, leaving few nonzero points at small compute gaps and making slope estimates unstable.
In particular, the estimate at $\ell=5$ is based on only $6$ usable nonzero points and is therefore noisier than the others.
Still, excluding $\ell=5$, the rescaled quantities $\widehat{A}_\ell(1-\rho_\ell^2)$ cluster in the range $1.1$--$1.6$, providing additional support for the predicted inverse dependence on $1-\rho_\ell^2$.
Appendix~\ref{app:assump51_diagnostic} reports additional joint statistics regarding the early-exit and final scores.

{\bf Vision transformer for CIFAR-100.}
In a complementary multi-class vision setting, Figure~\ref{fig:vit_cifar100_excess_vs_gap} considers the frozen-backbone
early-exit protocol on a ViT-Base model. The backbone has $L\!=\!12$ layers and is fine-tuned on CIFAR-100. We consider intermediate exits at
$\ell\in\{3,6,9,11\}$, and for each $\ell$ we fit a linear exit head by least-squares distillation to the final
logits (a multi-class analog of \eqref{eq:distilldef}), keeping the backbone fixed.
At inference, the exit decision is based on the exit confidence
$c_\ell(x)\triangleq \max_k \mathrm{softmax}(g_\ell(x))_k$, and we sweep the confidence threshold (via quantiles of $c_\ell$)
to obtain compute--accuracy operating points. Excess error scales approximately as a power of the compute gap, in line with Theorem~\ref{thm:deep_tradeoff_explicit}, and deeper exits exhibit sharper decay. 

Using the measured correlations
$(\rho_3,\rho_6,\rho_9,\rho_{11})=(0.246,\,0.412,\,0.589,\,0.773)$
gives the analytical exponents
$(A'_3,A'_6,A'_9,A'_{11})\approx(0.53,\,0.60,\,0.77,\,1.24)$.
A local log--log fit near the smallest compute gaps gives empirical slopes $\widehat A_3\approx1.56$, $\widehat A_6\approx3.16$, $\widehat A_9\approx8.78$,
while the estimate for $\ell=11$ is much less stable because the left edge of the curve is heavily quantized by finite test-set resolution. This larger theory--experiment gap than in the binary DistilBERT case suggests that a multiclass correction may be needed, since here the exit rule depends on the maximum softmax confidence over 100 classes rather than on a single binary score.
Further diagnostics are provided in Appendix \ref{app:vit_diagnostics}.


\begin{figure}
	\centering
	\includegraphics[width=0.8\linewidth]{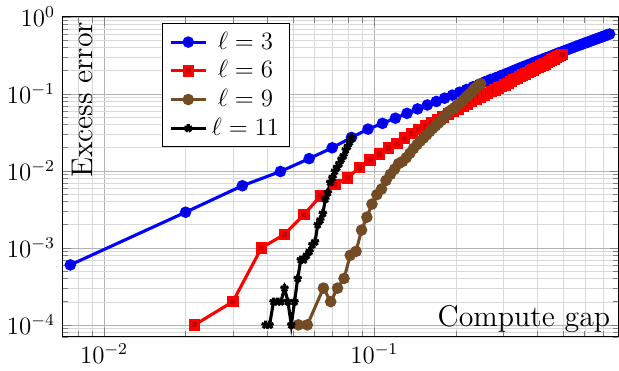}
	\caption{Early exit on a ViT-Base fine-tuned on CIFAR-100.}
	\label{fig:vit_cifar100_excess_vs_gap}
\end{figure}

%% file: tex/conclusions.tex
\section{Conclusions}
\label{sec:conclusions}

\paragraph{Scope and role of the assumptions.}
We discuss the role of different assumptions and models used throughout the paper. First,  IID weights and infinite width define a clean reference ensemble in
which pruning and early-exit quantities concentrate around deterministic
limits. These assumptions are not intended as literal models of trained
finite networks, but as idealized settings that isolate the mechanisms
studied here. Second, modeling choices such as one-shot magnitude pruning,
frozen backbones, and distilled exit heads specify concrete regimes of
practical interest. In particular, one-shot or training-free pruning is
relevant when retraining is expensive or unavailable, and frozen-backbone
early exit reflects common post hoc deployment protocols for pretrained
models. Third, the Gaussian-type score coupling in
Assumption~5.1 is a surrogate tail model rather than an exact distributional
claim. We assess this coupling through within-class diagnostics
Appendices~\ref{app:assump51_diagnostic} and \ref{app:vit_diagnostics}.

\paragraph{Summary and outlook.}We developed a tractable theory of compute reduction via one-shot magnitude pruning (static) and early exit (adaptive),
covering both single neurons and deep networks. For pruning, we proved concentration of the cosine similarity between
a weight vector and its pruned (renormalized) version around a deterministic pruning-rate limit, and showed how the
resulting misalignment propagates with depth through a correlation recursion. For early exit, we analyzed conditional
perceptrons and deep exits with frozen backbones and distilled heads, and derived compute--accuracy tradeoffs in which
the excess error decays as a power of the compute gap, with an exponent controlled by exit-to-final alignment. Our
experiments support these qualitative predictions. Future work includes extending the analysis to trained backbones and multiple exits, and comparing alternative one-shot
pruning rules under the same concentration-based framework.

%% file: tex/app/notation.tex
%
%

%% file: tex/app/gammatheorem.tex
\newcommand{\convindist}{\mathrel{\stackon[1.5pt]{$\rightarrow$}{$\scriptstyle \!d$}}}
\newcommand{\eqindist}{\mathrel{\stackon[1.5pt]{$=$}{$\scriptstyle d$}}}
\newcommand{\ratiothreshold}{y}
\newcommand{\normalizedratiothreshold}{y'}

\section{Proofs of  Results on Static Neurons}
\label{proof:theorem:gamma}
\subsection{Some auxiliary results on Gamma random variables}
\label{sec:gammaintro}
We begin by presenting some auxiliary results concerning Gamma random variables that we will need to prove the theorem. A general Gamma random variable $G \sim \mathrm{Gamma}(s,\theta)$ has PDF 
\begin{align}
	\label{gammadensity}
	f_G(x) = \tfrac{1}{\Gamma(s)\theta^{s}}x^{s-1}e^{-x/\theta},\,x \geq 0,
\end{align}
where $s$ and $\theta$ are known as the shape and the scale parameters, respectively, and  $\Gamma(\cdot)$ is the Gamma function. The lower and upper incomplete gamma functions are defined as $\gamma(s,x) \triangleq \int_0^x t^{s-1}e^{-t}\mathrm{d}t$ and $\Gamma(s,x) \triangleq \int_x^{\infty} t^{s-1}e^{-t}\mathrm{d}t$, respectively.

\newcommand{\ratiofunclower}{\underline{\mu}}
\newcommand{\ratiofuncupper}{\overline{\mu}}
\newcommand{\ratiofunc}{\mu}

%

We note the asymptotic expressions \citet[Eqs. 6.5.29 and 6.5.32]{abramowitz1948handbook}
\begin{align}
	\label{asexp1} \gamma(s,x) & \sim x^s/s,\,	x\rightarrow 0, \\
	\label{asexp2} \gamma(s,x) & \rightarrow \Gamma(s),\,	x\rightarrow \infty, \\
	\label{asexp3}  \Gamma(s,x) & \rightarrow \Gamma(s),\,x\rightarrow 0, \\
	\label{asexp4}  \Gamma(s,x) & \sim x^{s-1}e^{-x},\, x\rightarrow\infty.
\end{align}
According to (\ref{asexp1}), we have 
\begin{align}
	\label{gammaleqy}
	\mathrm{P}(G \leq y) = \frac{1}{\Gamma(\gammashape)}\gamma(\gammashape,y/\theta) \sim \frac{(y/\theta)^\gammashape}{\Gamma(1+\gammashape)} ,\,y\rightarrow 0,
\end{align}
and by (\ref{asexp4}), we obtain,
\begin{align}
	\label{gammageqy}
	\mathrm{P}(G \geq y) = \frac{1}{\Gamma(\gammashape)}\Gamma(\gammashape,y/\theta) \sim \frac{(y/\theta)^{\gammashape-1}e^{-y/\theta}}{\Gamma(\gammashape)} ,\,y\rightarrow \infty,
\end{align}

Let $G_{\leq y}$ denote a truncated Gamma random variable, obtained by conditioning $G$ on the event $G \leq y$, where $y\in[0,\infty)$. The notation $G_{\geq y}$ is defined similarly. A straightforward calculation reveals that for $a \geq 0$, we have
\begin{align}
	\mathrm{E}G_{\leq y}^a & = \frac{\int_{0}^y x^a f_G(x)\mathrm{d}x} {\int_{0}^y f_G(x)\mathrm{d}x} \\ & =  \frac{\int_{0}^y x^{a+\gammashape-1}e^{-x/\theta} \mathrm{d}x} {\int_{0}^y x^{\gammashape-1}e^{-x/\theta} \mathrm{d}x} \\ & = \frac{\int_{0}^{y/\theta} u^{a+\gammashape-1} \theta^{a+\gammashape} e^{-u} \mathrm{d}u} {\int_{0}^{y/\theta} u^{\gammashape-1}\theta^{\gammashape}e^{-u} \mathrm{d}u} \\	 \label{generallowermomentcalculation} & =  \frac{\theta^a \gamma(a+\gammashape, y/\theta)}{\gamma(\gammashape,y/\theta)},
\end{align}
and similarly,
\begin{align}
	\label{generaluppermomentcalculation}
	\mathrm{E} G_{\geq y}^a =  \frac{\theta^a \Gamma(a+\gammashape, y/\theta)}{\Gamma(\gammashape,y/\theta)}.
\end{align}
In (\ref{generallowermomentcalculation}), the first equality is by definition. To obtain the second equality, we substituted the PDF of $G$. The third equality is by a change of variables $u = x/\theta$. The final equality is by the definition of the lower incomplete Gamma function. Using (\ref{asexp1}), we can then obtain
\begin{align}
	\label{underlinemuzero}
	\mathrm{E}G_{\leq y}^a \sim \frac{\theta^a (y/\theta)^{a+\gammashape}/(a+\gammashape)}{ (y/\theta)^{\gammashape}/\gammashape} = \frac{\gammashape y^a}{a+\gammashape},\,y\rightarrow  0,
\end{align}
and using (\ref{asexp2}) yields
\begin{align}
	\label{underlinemuinf}
	\mathrm{E}G_{\leq y}^a \rightarrow \frac{\theta^a \Gamma(a+\gammashape)}{ \Gamma(\gammashape)} ,\,y\rightarrow \infty.
\end{align}

In a similar vein, using (\ref{asexp3}), we have 
\begin{align}
	\label{overlinemuzero}
	\mathrm{E}G_{\geq y}^a \rightarrow \frac{\theta^a \Gamma(a+\gammashape)}{ \Gamma(\gammashape)} ,\,y\rightarrow 0.
\end{align}
and applying (\ref{asexp4}), we obtain
\begin{align}
	\label{overlinemuinf}
	\mathrm{E}G_{\geq y}^a \sim  \frac{\theta^a  (y/\theta)^{a+\gammashape-1} e^{-y/\theta}  }{ (y/\theta)^{\gammashape-1} e^{-y/\theta} } = y^a,\,y\rightarrow\infty.
\end{align}

We begin with a useful lemma on the expected values of truncated Gamma random variables. 
\begin{lemma}
	\label{lemma:boundednessnotwo}
	The derivative of the function 
	\begin{align}
		\label{nuydefinbounded}
		y \mapsto \mathrm{E}[G_{\leq y}].
	\end{align}
	is bounded. 
\end{lemma}
\begin{proof}
	We have
	\begin{align}
		\mathrm{E}[G_{\leq y}] = \frac{\int_0^y x f(x) dx}{\int_0^y f(x) dx}
	\end{align}
	Hence, by the fundamental theorem of calculus,
	\begin{align}
		\frac{\mathrm{d}\mathrm{E}[G_{\leq y}]}{\mathrm{d}y} & = \frac{ y f(y) P(G \leq y) -  \mathrm{E}[G_{\leq y}] P(G \leq y) f(y) }{P^2(G \leq y)} \\
		& = \frac{  f(y) }{P(G \leq y)} (y -  \mathrm{E}[G_{\leq y}]  )
	\end{align}
	As $y\rightarrow 0$, according to (\ref{gammadensity}), (\ref{gammaleqy}), and (\ref{underlinemuzero}), we obtain
	\begin{align}
		\lim_{y\rightarrow 0 }	\frac{\mathrm{d}\mathrm{E}[G_{\leq y}]}{\mathrm{d}y} = \frac{s}{1+s}
	\end{align}
	On the other hand, by (\ref{gammadensity}) and (\ref{underlinemuinf}), we have
	\begin{align}
		\lim_{y\rightarrow \infty }	\frac{\mathrm{d}\mathrm{E}[G_{\leq y}]}{\mathrm{d}y} = 0
	\end{align}
	The statement of the lemma then follows from the continuity of $y \mapsto \mathrm{E}[G_{\leq y}]$. 
\end{proof}

The following useful lemma is a standard bound on powers of linear functions.

\begin{lemma}[\citeauthor{koyuncu:j3}~\citeyearpar{koyuncu:j3}, Lemma~7]
	\label{lemma:holder}
	For any real numbers $x_1,\ldots,x_n \geq 0$, we have 
	\begin{align}
		\label{holder}
		\textstyle	(\sum_{i=1}^n x_i)^{\beta} \leq n^{\beta - 1}\sum_{i=1}^n x_i^{\beta}. 
	\end{align}
\end{lemma}


%

Our main technical result, Theorem \ref{theorem:gamma}, shows that a certain ratio related to order statistics of Gamma random variables converges in probability to a threshold given by (\ref{maintheothresdef}). The following lemma shows that the derivatives of a more general form of the threshold function is bounded from above. 

\newcounter{constiiicnt}\setcounter{constiiicnt}{\value{cc}}\newcommand{\constiii}{C_{\theconstiiicnt}}\addtocounter{cc}{1}
\begin{lemma}
	\label{lemma:boundedderivatives}
	Let $q\in(0,1)$, $n \geq 1$ and $k = \lceil n q \rceil$. Define the function
	\begin{align}
		\label{hxdef}
		h(x) \triangleq \frac{(n-k)\mathrm{E}[G_{\geq x}] }{ k \mathrm{E}[G_{\leq x}] + (n-k)\mathrm{E}[G_{\geq x}]  }.
	\end{align}
	There is a constant $\constiii > 0$ that is independent of $n$ such that $|\frac{\mathrm{d}h}{\mathrm{d}x}| < \constiii ,\,\forall x\in\mathbb{R}$ and for all large enough $n$. 
\end{lemma}
\begin{proof}
	Let $\nu(x) = \mathrm{E}G_{\leq x} / \mathrm{E}G_{\geq x}$. We can rewrite (\ref{hxdef}) as
	\begin{align}
		\label{aksjdhakjsdhas}
		h(x) = \Bigl(1+\frac{k}{n-k} \nu(x)  \Bigr)^{-1},
	\end{align}
	The derivative of (\ref{aksjdhakjsdhas}) is calculated to be
	\begin{align}
		\label{aksjdhakjsdhas22}
		\frac{\mathrm{d}h}{\mathrm{d}x} =  -\frac{k}{n-k} \Bigl(1+\frac{k}{n-k} \nu(x)  \Bigr)^{-2}\frac{\mathrm{d}\nu}{\mathrm{d}x}
	\end{align}
	It follows that
	\begin{align}
		\Bigl|\frac{\mathrm{d}h}{\mathrm{d}x}\Bigr| \leq  \frac{k}{n-k} \Bigl|\frac{\mathrm{d}\nu}{\mathrm{d}x}\Bigr| \leq \frac{2q}{1-q} \Bigl|\frac{\mathrm{d}\nu}{\mathrm{d}x}\Bigr|.
	\end{align}
	The inequality follows since $\lim_{n\rightarrow\infty} \frac{k}{n-k} = \frac{q}{1-q}$. We used twice the limit as an upper bound, which will be valid for every large enough $n$. 
	
	What is now left to show is that the derivative $\frac{\mathrm{d}\nu}{\mathrm{d}x}$ is bounded.  It is sufficient to prove that the limits  $\lim_{x\rightarrow 0} \mathrm{d}\nu / \mathrm{d}x$ and $\lim_{x\rightarrow \infty} \mathrm{d}\nu / \mathrm{d}x$ exist and they are finite, and that $\mathrm{d}\nu / \mathrm{d}x$ is continuous on $(0,\infty)$. First, we calculate the derivative. Let 
	\begin{align}
		N(x) &  \triangleq \int_0^x y f(y) dy\int_x^{\infty} f(y) dy \\ & = [\mathrm{E}[G_{\leq x}]P(G \leq x)]P(G \geq x), \\
		D(x) & \triangleq \int_x^{\infty} y f(y) dy\int_0^{x} f(y) dy \\ & = [\mathrm{E}[G_{\geq x}]P(G \geq x)]P(G \leq x).
	\end{align}
	We note the alternate representation $\nu(x)= 	\frac{N(x)}{D(x)}$.  By the fundamental theorem of calculus, we obtain
	\begin{align}
		\frac{\mathrm{d}N}{\mathrm{d }x} & = 	x f(x) \mathrm{P}(G \geq x) - \mathrm{E}[G_{\leq x}] P(G \leq x) f(x), \\
		\frac{\mathrm{d}D}{\mathrm{d }y} & = 	-x f(x) \mathrm{P}(G \leq x) + \mathrm{E}[G_{\geq x}] P(G \geq x) f(x).
	\end{align}
	Using the division rule for derivatives, after some cumbersome but straightforward calculations, we can obtain
	\begin{align}
		\frac{\mathrm{d}\nu}{\mathrm{d }x} 
		& = f(x) \frac{x  \mathrm{P}(G \geq x)\mathrm{E}[G_{\geq x}] +  	x  \mathrm{P}(G \leq x)\mathrm{E}[G_{\leq x}] - \mathrm{E}[G_{\geq x}]\mathrm{E}[G_{\leq x}] \bigl[\mathrm{P}(G \leq x) + \mathrm{P}(G \geq x)\bigr]  }{ \mathrm{E}^2[G_{\geq x}]  \mathrm{P}(G \leq x)\mathrm{P}(G \geq x)}\\ 
		\label{djashdjasd} & =  \frac{ f(x)\bigl[ x  E[G] - \mathrm{E}[G_{\leq x}] \mathrm{E}[G_{\geq x}]    \bigr]}{  \mathrm{E}^2[G_{\geq x}]  \mathrm{P}(G \leq x)\mathrm{P}(G \geq x) }.
	\end{align}
	Since all functions involved in (\ref{djashdjasd}) are continuous, $\mathrm{d}\nu/\mathrm{d }x$ is continuous except possibly at $0$ and $\infty$. Using (\ref{gammadensity}), (\ref{gammaleqy}), (\ref{underlinemuzero}), (\ref{overlinemuzero}), and noting that $E[G] = s\theta$, we obtain
	\begin{align}
		\label{alksdjalkdas}
		\lim_{x\rightarrow 0}  \frac{\mathrm{d}\nu} {\mathrm{d }x} =  \frac{1}{\theta(1+\gammashape)}
	\end{align}
	Also, substituting (\ref{gammadensity}), (\ref{gammageqy}), (\ref{underlinemuinf}), and (\ref{overlinemuinf}), to (\ref{djashdjasd}), we can show that the derivative at $\infty$ is zero. Together with (\ref{alksdjalkdas}), this shows that the derivative is bounded. This concludes the proof.
\end{proof}

\newcounter{constxicnt}\setcounter{constxicnt}{\value{cc}}\newcommand{\constxi}{C_{\theconstxicnt}}\addtocounter{cc}{1}
Let us now derive upper and lower bounds on the variances of truncated Gamma random variables.
\begin{lemma}
	\label{lemma:truncgammavariancebound}
	For every $x \geq 0$, the variances of truncated Gamma random variables follow the bounds 
	\begin{align}
		\max\{s,1\}\theta^2  \geq \mathrm{var}(G_{\geq x}) & \geq \min\{s,1\}\theta^2, \\
		\constxi \geq \mathrm{var}(G_{\leq x}),
	\end{align}
	where $\constxi$ is a constant that is independent of $x$.
\end{lemma}
\begin{proof}
	Assume that the shape parameter $s$ of the Gamma random variable $G$ satisfies $s\in(0,1]$. Then, $G$ is log-convex, and according to \citeauthor{heckman1990empirical}~\citeyearpar{heckman1990empirical}, Proposition~2, the function  $x \mapsto \mathrm{var}(G_{\geq x})$ is monotonically increasing. In particular, $\mathrm{var}(G_{\geq x}) \geq \mathrm{var}(G_{\geq 0}) = \mathrm{var}(G) =s\theta^2$. On the other hand, when $s\in[1,\infty)$, the density $G$ is log-concave. In this case, \citeauthor{heckman1990empirical}~\citeyearpar{heckman1990empirical}, Proposition~1
	 shows that $x \mapsto \mathrm{var}(G_{\geq x})$ is monotonically decreasing. Hence, we have $\mathrm{var}(G_{\geq x}) \geq \lim_{x\rightarrow\infty}\mathrm{var}(G_{\geq x})$. In what follows, we calculate the limit. We have
	\begin{align}
		\mathrm{var}(G_{\geq x}) & = \mathrm{E}[G_{\geq x}^2] - \mathrm{E}^2[G_{\geq x}] \\
		& = \theta^2\frac{\Gamma(2+s,\frac{x}{\theta} )\Gamma(s,\frac{x}{\theta} )-\Gamma^2(1+s,\frac{x}{\theta} ) }{\Gamma^2(s,\frac{x}{\theta} )}
	\end{align}
	We have a $\frac{0}{0}$ indeterminancy as $x\rightarrow\infty$. We can thus apply L'H\^{o}pital's rule to obtain
	\begin{align}
		\lim_{x\rightarrow\infty}	\mathrm{var}(G_{\geq x}) & =  \lim_{x\rightarrow\infty} \frac{x^2 \Gamma(s, \frac{x}{\theta}) - 2 \theta x \Gamma(1 + s, \frac{x}{\theta}) + \theta^2 \Gamma(2 + s, \frac{x}{\theta})}{2 \Gamma(s, \frac{x}{\theta})} \\
		& =	\lim_{x\rightarrow\infty} \theta^s  \frac{ \theta \Gamma(1 + s, \frac{x}{\theta})   -x \Gamma(s, \frac{x}{\theta})  }{e^{-\frac{x}{\theta}} x^{s-1}} \\
		& =	\lim_{x\rightarrow\infty} \frac{ \theta^{1+s}\Gamma(s, \frac{x}{\theta})   }{e^{-\frac{x}{\theta}} x^{-2 + s} (\theta - s \theta + x)} \\
		\label{tetakare} & = \theta^2
	\end{align}
	The second and the third equalities also follow from L'H\^{o}pital's rule. In order to obtain the final equality, we have applied (\ref{asexp4}). Note that the derivatives of the upper incomplete Gamma function can be evaluated via the formulae $\frac{\mathrm{d}\Gamma(s,x)}{\mathrm{d}x} = \frac{\mathrm{d}}{\mathrm{d}x}  \int_x^{\infty} t^{s-1}e^{-t}\mathrm{d}t = -x^{s-1}e^{-x}$, by the fundamental theorem of calculus. Therefore, for any $s$, we obtain $\mathrm{var}(G_{\geq x}) \geq  \min\{s, 1\}\theta^2$. 
	
	Since $\mathrm{var}(G_{\leq x}) =   \mathrm{E}[ G_{\leq x}^2] - [\mathrm{E} G_{\leq x}]^2 \leq \mathrm{E} [G_{\leq x}^2]$, according to (\ref{underlinemuzero}) and (\ref{underlinemuinf}), the lower conditional variance $\mathrm{var}(G_{\leq x})$ is bounded by a constant that is independent of $x$ from above. For the upper conditional variance $\mathrm{var}(G_{\geq x})$, we consider the cases of $s\in(0,1)$ and $s\in[1,\infty)$ separately. In the former scenario, the monotonically increasing nature of \( x \mapsto \mathrm{var}(G_{\geq x}) \), as established in \citeauthor{heckman1990empirical}~\citeyearpar{heckman1990empirical}, Proposition~2, in conjunction with  (\ref{tetakare}), demonstrates that \( \mathrm{var}(G_{\geq x}) \leq \theta^2 \) for every \( x \). For $s \geq 1$, we obtain \( \mathrm{var}(G_{\geq x}) \leq \mathrm{var}(G_{\geq 0}) = s\theta^2  \), according to \citeauthor{heckman1990empirical}~\citeyearpar{heckman1990empirical}, Proposition~1. Hence, for any $s$, we have $\mathrm{var}(G_{\geq x}) \leq \theta^2\max\{1,s\}$. This concludes the proof of the lemma.
\end{proof}

\newcounter{constvicnt}\setcounter{constvicnt}{\value{cc}}\newcommand{\constvi}{C_{\theconstvicnt}}\addtocounter{cc}{1}
\newcounter{constviicnt}\setcounter{constviicnt}{\value{cc}}\newcommand{\constvii}{C_{\theconstviicnt}}\addtocounter{cc}{1}
As a corollary, we obtain lower and upper bounds on a linear combination of variances truncated Gamma random variables.
\begin{corollary}
	\label{corol:variancebound}
	Let $q\in(0,1)$, $n \geq 1$ and $k = \lceil n q \rceil$. Let 
	\begin{align}
		\label{sigmasqdef}
		\sigma^2 \triangleq y^2 (k-1) \mathrm{var}(G_{\leq x}) + (1-y)^2(n-k)\mathrm{var}(G_{\geq x}).
	\end{align}
	Then, for every $\tau \in(0,1)$ and $y\in(0,1)$ with $|y - \tau| \leq \frac{1-\tau}{2}$, we have
	\begin{align}
		\constvi n \leq \sigma^2 \leq \constvii  n
	\end{align}
	for all sufficiently large $n$, where  $\constvi$ and $\constvii$ are constants.
\end{corollary}
\begin{proof}
	According to Lemma \ref{lemma:truncgammavariancebound}, for any $s$, we obtain
	\begin{align}
		\sigma^2 & \geq (1-y)^2 (n-k) \mathrm{var}(G_{\geq x}) \\
		& \geq  (1-y)^2 (n-k) \min\{s, 1\}\theta^2
	\end{align}
	The bound $|y - \tau| \leq \frac{1-\tau}{2}$ implies 
	\begin{align}
		\sigma^2 &  \geq  \biggl(\frac{1-\tau}{2} \biggr)^2  (n-k) \min\{s, 1\}\theta^2
	\end{align}
	Also, substituting $k = \lceil nq \rceil$, we obtain 
	\begin{align}
		\sigma^2 &  \geq  \biggl(\frac{1-\tau}{2} \biggr)^2  (n-\lceil nq \rceil) \min\{s, 1\}\theta^2\\
		& \geq \frac{1-q}{2}  \biggl(\frac{1-\tau}{2} \biggr)^2  \min\{s, 1\}\theta^2  n,
	\end{align}
	for sufficiently large $n$. This proves the lower estimate on the variance. 
	
	\newcounter{constxcnt}\setcounter{constxcnt}{\value{cc}}\newcommand{\constx}{C_{\theconstxcnt}}\addtocounter{cc}{1}
	For the upper estimate,  we can first obtain
	\begin{align}
		\sigma^2 & \leq \underbrace{y^2}_{\leq 1} \underbrace{(k-1)}_{\leq n} \mathrm{var}(G_{\leq x})  + \underbrace{(1-y)^2}_{\leq 1}\underbrace{(n-k)}_{\leq n} \mathrm{var}(G_{\geq x})  \\ 		\label{laksjdlkajdslkajsdad}& = n[\mathrm{var}(G_{\leq x}) + \mathrm{var}(G_{\geq x})] ,
	\end{align}
	and applying Lemma \ref{lemma:truncgammavariancebound} proves the upper estimate on $\sigma^2$. This concludes the proof of the corollary.
\end{proof}

The following lemma is utilized to bound the error terms resulting from the Berry-Esseen estimates. 
\newcounter{constvcnt}\setcounter{constvcnt}{\value{cc}}\newcommand{\constv}{C_{\theconstvcnt}}\addtocounter{cc}{1}
\begin{lemma}
	\label{lemma:normalizedkurtosisbound}
	Let $q,y\in(0,1)$, $n \geq 1$ and $k = \lceil n q \rceil$. Let $\sigma^2$ be as defined in (\ref{sigmasqdef}) of Corollary \ref{corol:variancebound}, and 
	\begin{align}
		\rho \triangleq y^3 (k-1) \mathrm{E}\bigl|G_{\leq x} - \mathrm{E}G_{\leq x}\bigr|^3  + (1-y)^3(n-k)\mathrm{E}\bigl|G_{\geq x} - \mathrm{E}G_{\geq x}\bigr|^3.
	\end{align}
	For every large enough $n$, we have
	\begin{align}
		\int_0^{\infty} \frac{\rho}{\sigma^3}f_{G_{(k)}}(x)\mathrm{d}x \leq \constv n^{-\frac{1}{2}}
	\end{align}
	for some constant $\constv > 0$  that  is independent of $y,k,n$.
\end{lemma}
\begin{proof}
	Using the inequalities $y \leq 1$ and $k \leq n$, we obtain 
	\begin{align}
		\rho \leq n\mathrm{E}|G_{\leq x} - \mathrm{E}G_{\leq x}|^3  + n\mathrm{E}|G_{\geq x} - \mathrm{E}G_{\geq x}|^3.
	\end{align}
	Applying Lemma \ref{lemma:holder} yields
	\newcounter{constxiicnt}\setcounter{constxiicnt}{\value{cc}}\newcommand{\constxii}{C_{\theconstxiicnt}}\addtocounter{cc}{1}
	\begin{align}
		\rho & \leq 4n (\mathrm{E}G_{\leq x}^3 + \mathrm{E}^3G_{\leq x}  + \mathrm{E}G_{\geq x}^3 + \mathrm{E}^3G_{\geq x}) \\
		& \leq 8n ( \mathrm{E}G_{\geq x}^3 + \mathrm{E}^3G_{\geq x}) \\
		\label{asdkjalskdjalskjd}	& \leq \constxii n (1+x^3),
	\end{align}
	where $\constxii$ is a constant that is independent of $n$ and $x$. The last inequality is a consequence of (\ref{overlinemuzero}) and (\ref{overlinemuinf}).

	It is a standard result in probability theory that the exact PDF for the $k$th order statistic $G_{(k)}$ is given by
	\begin{align}
		f_{G_{(k)}}(x) & = \frac{n!}{(k - 1)!(n - k)!} f_G(x) [F_G(x)]^{k-1} [1 - F_G(x)]^{n-k} \\
		& \leq n2^{n-1} f_G(x)  [1 - F_G(x)]^{n-k}
	\end{align}
	Since $F_G(x) \rightarrow 1$, there exists $x_0 > 0$ and $c\in(0,1)$ such that for every $x \geq x_0$, we have $n 2^{n-1} [1 - F_G(x)]^{n-\lceil n q\rceil} \leq  c^{-n}$. As a result, we obtain $ f_{G_{(k)}}(x)  \leq f_{G}(x) c^{-n},\forall x\geq x_0$. Combining this with (\ref{asdkjalskdjalskjd}) and the lower bound on $\sigma$ in Corollary \ref{corol:variancebound}, we obtain
	\begin{align}
		\int_0^{\infty} \frac{\rho}{\sigma^3}f_{G_{(k)}}(x)\mathrm{d}x & = 	\int_0^{x_0} \frac{\rho}{\sigma^3}f_{G_{(k)}}(x)\mathrm{d}x + 	\int_{x_0}^{\infty} \frac{\rho}{\sigma^3}f_{G_{(k)}}(x)\mathrm{d}x \\ 
		&  \leq  	\int_0^{x_0} \frac{\constxii n (1+x_0^3)}{(\constvi n)^{1.5}}f_{G_{(k)}}(x)\mathrm{d}x + 	\int_{x_0}^{\infty} \frac{\constxii n (1+x^3)}{(\constvi n)^{1.5}}f_{G}(x)c^{-n}\mathrm{d}x \\
		&  \leq  	 \frac{\constxii (1+x_0^3)}{\constvi^{1.5}} n^{-\frac{1}{2}}+ \frac{\constxii (1+x_0^3)}{\constvi^{1.5}} (\mathrm{E}G + \mathrm{E}G^3) o(n^{-\frac{1}{2}}).
	\end{align}
	The proof is now complete since the moments of a Gamma random variable are also finite.	\end{proof}

\subsection{Other auxiliary results}

As the first result, we recall the central limit theorem for quantiles.
\begin{proposition}[\cite{xu2011limit}]
	\label{prop:orderstatconv}
	Let $X_1, X_2, \ldots,X_n$ be IID copies of a random variable $X$. Given $p\in(0,1)$, let $\xi_p = \inf\{x: F_X(x) \geq 1-p\}$. Suppose $F_X$ has a continuous first derivative $f_X$ in the neighborhood of $\xi_p$ and $f(\xi_p) > 0$. Then, 
	\begin{align}
		\frac{\sqrt{n} f_X(\xi_p)}{\sqrt{p(1-p)}}  \Bigl(X_{(\lceil n(1-p) \rceil)} - \xi_p \Bigr) \sim N(0,1) \mbox{ as } n\rightarrow \infty. 
	\end{align}
	
\end{proposition}

We then present the general form of Berry-Esseen theorem for non-identically distributed random variables.

\begin{proposition}[Berry-Esseen Theorem \cite{esseen1942liapunoff}]
	\label{proposition:berryesseen}
	Let $X_1, X_2, \ldots,$ be independent random variables with $E[X_i] = 0$ and $\mathrm{E}[|X_i|^3] < \infty$ for every $i\in\mathbb{Z}_{>0}$. Let $\Phi$ denote the CDF of the standard normal distribution. For all $n$,  we have
	\begin{align}
		\sup_{x \in \mathbb{R}} \biggl|\mathrm{P}\biggl(\frac{X_1 + X_2 + \ldots + X_n}{\sqrt{\sigma_1^2 + \sigma_2^2 + \ldots + \sigma_n^2}} \leq x\biggr) - \Phi(x)\biggr| \leq 8 \left( \sum_{i=1}^{n} \mathrm{var}X_i \right)^{-3/2}  \sum_{i=1}^{n} \mathrm{E}|X_i|^3.
	\end{align}
\end{proposition}


\begin{lemma}[Effective concentration of $G_{(k)}$ around $\omega_q$]
	\label{lemma:orderstat_conc}
	Let $k=\lceil nq\rceil$ and $F_G(\omega_q)=q$. For any $\delta>0$,
	\[
	\mathrm{P}\bigl(|G_{(k)}-\omega_q|\ge \delta\bigr)
	\le
	2\exp\!\Bigl(-2n\,\min\{q-F_G(\omega_q-\delta),\,F_G(\omega_q+\delta)-q\}^2\Bigr).
	\]
\end{lemma}

\begin{proof}
	Let $N_x\triangleq \sum_{i=1}^n \mathbf{1}\{G_i\le x\}\sim \mathrm{Bin}(n,F_G(x))$.
	Then $\{G_{(k)}\le x\}\iff\{N_x\ge k\}$ and $\{G_{(k)}\ge x\}\iff\{N_x\le k-1\}$.
	Apply Hoeffding's inequality to $N_x/n$ at $x=\omega_q-\delta$ and $x=\omega_q+\delta$, then union bound.
\end{proof}

\subsection{Proof of the Theorem \ref{theorem:gamma}}
\label{sec:proofoftheoremgamma}

We are now ready to prove Theorem \ref{theorem:gamma}. We organize the proof in six steps. 

{\bf Step-1:} In the first step, we approximate the cumulative distribution function of $\Xi_n$ via a normal random variable using the Berry-Esseen theorem. Let $y\in[0,1]$ and $k = \lceil nq \rceil$. We have
\begin{align}
	\mathrm{P} \Bigl(\Xi_n \leq \ratiothreshold \Bigr)  & = \mathrm{P}\Bigl(G_{(k+1)}+ \cdots + G_{(n)} \leq \ratiothreshold ( G_{(1)}+ \cdots + G_{(n)})  \Bigr) \\
	\label{alsjkdalksdjasd}	& = \int_0^{\infty} \mathrm{P} \Biggl((1-\ratiothreshold)\sum_{j=k+1}^{n} G_{(j)} - \ratiothreshold\sum_{j=1}^{k-1} G_{(j)}  \leq \ratiothreshold x \, \biggl|\, G_{(k)} = x\Biggr)f_{G_{(k)}}(x)\mathrm{d}x.
\end{align}
According to \cite{ahsanullah2013conditional}, the two sets of order statistics $(G_{(k+1)},\ldots,G_{(n)})$ and  $(G_{(1)}, \ldots, G_{(k-1)})$ that appear in (\ref{alsjkdalksdjasd})  are conditionally independent given  $G_{(k)} = x$. Moreover, we have
\begin{align}
	\Bigg[ 	\sum_{j=k+1}^{n} G_{(j)} \biggl| G_{(k)} = x   \Biggr]  \sim \sum_{j=1}^{n-k}  G_{\geq x,j},
\end{align}
and
\begin{align}
	\Bigg[ \sum_{j=1}^{k-1} G_{(j)} \biggl| G_{(k)} = x \Biggr]  \sim \sum_{j=1}^{k-1}  G_{\leq x,j},
\end{align}
where $G_{\leq x,1}  ,G_{\geq x,1},G_{\leq x,2}  ,G_{\geq x,2},\ldots$ is a sequence of IID random variables with 
\begin{align}
	\label{underlinegj}
	G_{\leq x,j} \sim G_{\leq x},
\end{align}
and 
\begin{align}
	\label{overlinegj}
	G_{\geq x,j}  \sim G_{\geq x}
\end{align}
for every $j\in\mathbb{Z}_{>0}$.

Let us now normalize the mean of (\ref{underlinegj}) and (\ref{overlinegj})  as 
\begin{align}
	G_{\leq x,j}' & \triangleq G_{\leq x,j} - \mathrm{E}[G_{\leq x,j}  ] \\ & =  G_{\leq x,j} - \mathrm{E}[G_{\leq x}], \\
	G_{\geq x,j}' & \triangleq G_{\geq x,j} - \mathrm{E}[G_{\geq x,j}  ]  \\ & =  G_{\geq x,j} - \mathrm{E}[G_{\geq x}],
\end{align}
where $j\in\mathbb{Z}_{>0}$. We have
\begin{align}
	\mathrm{P} \Bigl(\Xi_n \leq \ratiothreshold \Bigr)  & = \int_0^{\infty} \mathrm{P} \Biggl((1-\ratiothreshold)\sum_{j=1}^{n-k} G_{\geq x, j}' - \ratiothreshold\sum_{j=1}^{k-1} G_{\leq x, j}'  \leq \normalizedratiothreshold \Biggr)f_{G_{(k)}}(x)\mathrm{d}x,
\end{align}
where
\begin{align}
	\label{def:meannormalizedthreshold}
	\normalizedratiothreshold \triangleq \ratiothreshold x + \ratiothreshold (k-1)\mathrm{E}[G_{\leq x}]-(1-\ratiothreshold )(n-k)\mathrm{E}[G_{\geq x}].
\end{align}
By definition. the random variables $G_{\leq x, j}'$ and $G_{\geq x, j}'$ have zero mean. According to Lemma \ref{lemma:truncgammavariancebound}, they also have finite normalized moments. Hence, Proposition \ref{proposition:berryesseen} is applicable, and we have
\newcommand{\cltstd}{\sigma}
\newcommand{\cltmean}{\mu}
\begin{align}
	\label{laksjdlkasdadasdjsda}
	\mathrm{P} \Bigl(\Xi_n \leq \ratiothreshold \Bigr) \leq \int_0^{\infty}  \Phi( \normalizedratiothreshold /\sigma )f_{G_{(k)}}(x)\mathrm{d}x + 8\int_0^{\infty} \frac{\rho}{\sigma^3}f_{G_{(k)}}(x)\mathrm{d}x,
\end{align}
where $\Phi$ is the CDF of the standard normal random variable with zero mean and unit variance, 
\begin{align}
	\sigma^2 \triangleq y^2 (k-1) \mathrm{var}(G_{\leq x}) + (1-y)^2(n-k)\mathrm{var}(G_{\geq x}),
\end{align}
and 
\begin{align}
	\rho \triangleq y^3 (k-1) \mathrm{E}\biggl[\Bigl|G_{\leq x} - \mathrm{E}[G_{\leq x}]\Bigr|^3\biggr]   + (1-y)^3(n-k)\mathrm{E}\biggl[\Bigl|G_{\geq x} - \mathrm{E}[G_{\geq x}]\Bigr|^3\biggr].
\end{align}
We also record a decomposition that will be convenient later. For any $\delta>0$, we can write
\begin{align}
	\mathrm{P}\bigl(\Xi_n \le y\bigr)
	\le \int_{|x-\omega_q|\le \delta} \Phi\!\Bigl(\normalizedratiothreshold/\sigma\Bigr) f_{G_{(k)}}(x)\mathrm{d}x
	+ \int_{|x-\omega_q|\ge \delta} f_{G_{(k)}}(x)\mathrm{d}x
	+ 8\int_{0}^{\infty}\frac{\rho}{\sigma^3} f_{G_{(k)}}(x)\mathrm{d}x.
	\label{eq:be_decomposition}
\end{align}
We also note the corresponding lower Berry--Esseen bound, which we will use later:
\begin{align}
	\mathrm{P}\bigl(\Xi_n \le y\bigr)
	\ge
	\int_{0}^{\infty} \Phi\!\Bigl(\normalizedratiothreshold/\sigma\Bigr) f_{G_{(k)}}(x)\mathrm{d}x
	- 8\int_{0}^{\infty}\frac{\rho}{\sigma^3} f_{G_{(k)}}(x)\mathrm{d}x .
	\label{eq:be_lower}
\end{align}
Let us now rewrite the mean-normalized threshold $\normalizedratiothreshold$ defined in (\ref{def:meannormalizedthreshold})  as
\begin{align}
	\label{laksjdhlaksjdasd}
	\normalizedratiothreshold & = \ratiothreshold x - \ratiothreshold \mathrm{E}[G_{\leq x}]  - \Bigl[ k \mathrm{E}[G_{\leq x}] + (n-k)\mathrm{E}[G_{\geq x}] \Bigr] \Bigl(  h(x) - y   \Bigr),
\end{align}
where
\begin{align}
	\label{hxdefff}
	h(x) \triangleq \frac{(n-k)\mathrm{E}[G_{\geq x}] }{ k \mathrm{E}[G_{\leq x}] + (n-k)\mathrm{E}[G_{\geq x}]  }.
\end{align}

	\newcounter{constgammadeltacnt}\setcounter{constgammadeltacnt}{\value{cc}}
\newcommand{\constgammadelta}{C_{\theconstgammadeltacnt}}\addtocounter{cc}{1}

{\bf Step-2:} In the second step, we introduce a convenient proxy threshold $\tau_n \triangleq h(\omega_q)$ and show that the key quantities change only a little when we stay in a small neighborhood around the typical cutoff point $\omega_q$.

We recall $k=\lceil nq\rceil$ and the function $h(x)$ from \eqref{hxdefff}. We introduce the finite-$n$ proxy threshold
\begin{align}
\tau_n \triangleq h(\omega_q)=\frac{(n-k)\mathrm{E}[G_{\ge \omega_q}]}{k\mathrm{E}[G_{\le \omega_q}]+(n-k)\mathrm{E}[G_{\ge \omega_q}]}.
\end{align}
We fix $\epsilon\in(0,1)$ and we set
\begin{align}
\delta \triangleq \delta(\epsilon)\triangleq \frac{\epsilon}{2\constgammadelta},
\qquad
\Delta(\epsilon)\triangleq \min\Bigl\{q-F_G(\omega_q-\delta),\;F_G(\omega_q+\delta)-q\Bigr\}.
\end{align}
By Lemma~\ref{lemma:boundedderivatives}, we have $\sup_x |h'(x)|\le \constiii$ for all sufficiently large $n$.
With the choice $\constgammadelta = 2\constiii$, we obtain that,
for every $|x-\omega_q|\le \delta$,
\begin{align}
|h(x)-\tau_n|=|h(x)-h(\omega_q)|\le \constiii|x-\omega_q|\le \constiii\delta = \epsilon/4.
\end{align}

	\newcounter{constgammamucnt}\setcounter{constgammamucnt}{\value{cc}}
\newcommand{\constgammamu}{C_{\theconstgammamucnt}}\addtocounter{cc}{1}

In addition, Lemma~\ref{lemma:boundednessnotwo} implies the existence of a constant $\constgammamu > 0$
such that
\begin{align}
\mathrm{E}[G_{\le \omega_q-\delta}] \ge \mathrm{E}[G_{\le \omega_q}] - \constgammamu\,\delta.
\end{align}
Since $x\mapsto \mathrm{E}[G_{\le x}]$ is nondecreasing, we conclude that for every $|x-\omega_q|\le \delta$,
\begin{align}
k\mathrm{E}[G_{\le x}] + (n-k)\mathrm{E}[G_{\ge x}] \ge n\,\mathrm{E}[G_{\le x}] \ge n\,\mathrm{E}[G_{\le \omega_q-\delta}]
\ge n\bigl(\mathrm{E}[G_{\le \omega_q}] - \constgammamu\,\delta\bigr).
\end{align}
We also introduce the shorthand
\begin{align}
a(\epsilon)\triangleq \frac{1}{2}\bigl(\mathrm{E}[G_{\le \omega_q}] - \constgammamu\,\delta\bigr)_+,
\qquad
b(\epsilon)\triangleq \omega_q+\delta.
\end{align}
With this notation, we have for all $|x-\omega_q|\le \delta$ that
\begin{align}
k\mathrm{E}[G_{\le x}] + (n-k)\mathrm{E}[G_{\ge x}] \ge 2n\,a(\epsilon),
\qquad
x\le b(\epsilon).
\end{align}

{\bf Step-3:} We show that the ratio is very unlikely to fall noticeably below the proxy threshold. We consider $y=\tau_n-\epsilon/2$. For every $|x-\omega_q|\le \delta$, the bound above yields
\begin{align}
h(x)-y \ge (\tau_n-\epsilon/4) - (\tau_n-\epsilon/2)=\epsilon/4.
\end{align}
We then use the representation \eqref{laksjdhlaksjdasd}. For $|x-\omega_q|\le \delta$, we obtain
\begin{align}
\normalizedratiothreshold
\le yx - \bigl[k\mathrm{E}[G_{\le x}] + (n-k)\mathrm{E}[G_{\ge x}]\bigr]\,(h(x)-y)
\le yx - (2n a(\epsilon))(\epsilon/4)
\le b(\epsilon) - n a(\epsilon)(\epsilon/2).
\end{align}

	\newcounter{constgammasigcnt}\setcounter{constgammasigcnt}{\value{cc}}
\newcommand{\constgammasig}{C_{\theconstgammasigcnt}}\addtocounter{cc}{1}

We also note that $\sigma^2$ admits a uniform linear-in-$n$ upper bound. Indeed, since $y\in(0,1)$ and $k\le n$, we have
\begin{align}
	\sigma^2
	&= y^2 (k-1)\,\mathrm{var}(G_{\le x})+(1-y)^2 (n-k)\,\mathrm{var}(G_{\ge x}) \\
	&\le (k-1)\,\mathrm{var}(G_{\le x})+(n-k)\,\mathrm{var}(G_{\ge x}) \\
	&\le n\bigl(\mathrm{var}(G_{\le x})+\mathrm{var}(G_{\ge x})\bigr).
	\label{eq:sig2_crude}
\end{align}
By Lemma~\ref{lemma:truncgammavariancebound}, there exists a constant $\constgammasig>0$, depending only on $(s,\theta)$, such that
$\mathrm{var}(G_{\le x})+\mathrm{var}(G_{\ge x})\le \constgammasig$ for all $x\ge 0$.
Combining with \eqref{eq:sig2_crude} yields $\sigma^2\le \constgammasig\,n$.

Using a Gaussian tail bound, we conclude that for $|x-\omega_q|\le \delta$,
\begin{align}
\Phi\!\Bigl(\frac{\normalizedratiothreshold}{\sigma}\Bigr)
\le \exp\!\Bigl(-\frac{\bigl(n a(\epsilon)(\epsilon/2)-b(\epsilon)\bigr)_+^2}{2\constgammasig\,n}\Bigr).
\end{align}
We now apply the decomposition \eqref{eq:be_decomposition} with the choice $\delta=\delta(\epsilon)$.
We bound the complement probability $\int_{|x-\omega_q|\ge \delta} f_{G_{(k)}}(x)\mathrm{d}x
=\mathrm{P}(|G_{(k)}-\omega_q|\ge \delta)$ using Lemma~\ref{lemma:orderstat_conc}, and the
Berry--Esseen remainder term $8\int_{0}^{\infty}\frac{\rho}{\sigma^3} f_{G_{(k)}}(x)\mathrm{d}x$ using
Lemma~\ref{lemma:normalizedkurtosisbound}. Combining these bounds with the estimate on the first integral yields
\begin{align}
	\mathrm{P}(\Xi_n\le \tau_n-\epsilon/2)
	\le
	\exp\!\Bigl(-\frac{\bigl(n a(\epsilon)(\epsilon/2)-b(\epsilon)\bigr)_+^2}{2\constgammasig\,n}\Bigr)
	+2\exp\!\bigl(-2n\,\Delta(\epsilon)^2\bigr)
	+8\constv\,n^{-1/2}.
	\label{eq:lower_tail_taun}
\end{align}

{\bf Step-4:} We show that the ratio is very unlikely to fall noticeably larger than the proxy threshold. 
We consider $y=\tau_n+\epsilon/2$ and bound $\mathrm{P}(\Xi_n\ge y)$.
We begin with the identity $\mathrm{P}(\Xi_n\ge y)=1-\mathrm{P}(\Xi_n\le y)$.
Applying the lower Berry--Esseen inequality \eqref{eq:be_lower} to $\mathrm{P}(\Xi_n\le y)$ yields
\begin{align}
	\mathrm{P}(\Xi_n\ge y)
	&= 1-\mathrm{P}(\Xi_n\le y) \\
	&\le 1-\int_{0}^{\infty} \Phi\!\Bigl(\normalizedratiothreshold/\sigma\Bigr) f_{G_{(k)}}(x)\mathrm{d}x
	+ 8\int_{0}^{\infty}\frac{\rho}{\sigma^3} f_{G_{(k)}}(x)\mathrm{d}x .
	\label{eq:upper_start}
\end{align}
We next decompose the first term on the right-hand side according to the event $\{|x-\omega_q|\le \delta\}$:
\begin{align}
	1-\int_{0}^{\infty} \Phi\!\Bigl(\normalizedratiothreshold/\sigma\Bigr) f_{G_{(k)}}(x)\mathrm{d}x
	&=
	\int_{0}^{\infty} \Bigl(1-\Phi\!\bigl(\normalizedratiothreshold/\sigma\bigr)\Bigr) f_{G_{(k)}}(x)\mathrm{d}x \\
	&\le
	\int_{|x-\omega_q|\le \delta} \Bigl(1-\Phi\!\bigl(\normalizedratiothreshold/\sigma\bigr)\Bigr) f_{G_{(k)}}(x)\mathrm{d}x
	+\mathrm{P}(|G_{(k)}-\omega_q|\ge \delta).
	\label{eq:upper_decomp}
\end{align}

We now lower bound $\normalizedratiothreshold$ on $\{|x-\omega_q|\le \delta\}$.
As in Step-2, we have $|h(x)-\tau_n|\le \epsilon/4$ for $|x-\omega_q|\le \delta$, hence
\begin{align}
h(x)-y \le (\tau_n+\epsilon/4)-(\tau_n+\epsilon/2)=-\epsilon/4.
\end{align}
Using \eqref{laksjdhlaksjdasd} together with $\mathrm{E}[G_{\le x}]\le x$, we obtain for $|x-\omega_q|\le \delta$ that
\begin{align}
	\normalizedratiothreshold
	&= yx - y\mathrm{E}[G_{\leq x}] - \Bigl[k \mathrm{E}[G_{\leq x}] + (n-k)\mathrm{E}[G_{\geq x}] \Bigr]\Bigl(h(x)-y\Bigr) \\
	&\ge 0 + \Bigl[k \mathrm{E}[G_{\leq x}] + (n-k)\mathrm{E}[G_{\geq x}] \Bigr]\frac{\epsilon}{4}.
	\label{eq:yp_lower1}
\end{align}
Moreover, Step-2 gives $k \mathrm{E}[G_{\le x}] + (n-k)\mathrm{E}[G_{\ge x}] \ge 2n a(\epsilon)$ on $\{|x-\omega_q|\le \delta\}$, and thus
\begin{align}
	\normalizedratiothreshold \ge n a(\epsilon)(\epsilon/2),
	\qquad \forall\,|x-\omega_q|\le \delta.
	\label{eq:yp_lower2}
\end{align}
We also use the upper bound $\sigma^2\le \constgammasig\,n$. Combining this with a Gaussian tail bound, we obtain that for $|x-\omega_q|\le \delta$,
\begin{align}
	1-\Phi\!\Bigl(\normalizedratiothreshold/\sigma\Bigr)
	&\le \exp\!\Bigl(-\tfrac{1}{2}\bigl(\normalizedratiothreshold/\sigma\bigr)^2\Bigr) \\
	&\le \exp\!\Bigl(-\frac{\bigl(n a(\epsilon)(\epsilon/2)\bigr)^2}{2\constgammasig\,n}\Bigr)
	\le \exp\!\Bigl(-\frac{\bigl(n a(\epsilon)(\epsilon/2)-b(\epsilon)\bigr)_+^2}{2\constgammasig\,n}\Bigr),
	\label{eq:upper_phi}
\end{align}
where in the last inequality we use $b(\epsilon)\ge 0$ and $(u)_+ \le u$ for $u\ge 0$, so that the final exponent is a weaker form that matches Step-3.

Substituting \eqref{eq:upper_phi} into \eqref{eq:upper_decomp} and then combining with \eqref{eq:upper_start}, we obtain
\begin{align}
	\mathrm{P}(\Xi_n\ge \tau_n+\epsilon/2)
	\le
	\exp\!\Bigl(-\frac{\bigl(n a(\epsilon)(\epsilon/2)-b(\epsilon)\bigr)_+^2}{2\constgammasig\,n}\Bigr)
	+\mathrm{P}(|G_{(k)}-\omega_q|\ge \delta)
	+8\int_{0}^{\infty}\frac{\rho}{\sigma^3} f_{G_{(k)}}(x)\mathrm{d}x .
\end{align}
Finally, we bound $\mathrm{P}(|G_{(k)}-\omega_q|\ge \delta)$ using Lemma~\ref{lemma:orderstat_conc} and we bound the remainder integral using Lemma~\ref{lemma:normalizedkurtosisbound}.
This yields
\begin{align}
	\mathrm{P}(\Xi_n\ge \tau_n+\epsilon/2)
	\le
	\exp\!\Bigl(-\frac{\bigl(n a(\epsilon)(\epsilon/2)-b(\epsilon)\bigr)_+^2}{2\constgammasig\,n}\Bigr)
	+2\exp\!\bigl(-2n\,\Delta(\epsilon)^2\bigr)
	+8\constv\,n^{-1/2},
	\label{eq:upper_tail_taun}
\end{align}
which is the desired upper-tail estimate around $\tau_n$.

	\newcounter{constgammashiftcnt}\setcounter{constgammashiftcnt}{\value{cc}}
\newcommand{\constgammashift}{C_{\theconstgammashiftcnt}}\addtocounter{cc}{1}

{\bf Step-5:} We combine the two one sided bounds and then switch from the proxy threshold to the final target threshold. We combine \eqref{eq:lower_tail_taun} and \eqref{eq:upper_tail_taun} with a union bound, and we obtain
\begin{align}
	\mathrm{P}\bigl(|\Xi_n-\tau_n|\ge \epsilon/2\bigr)
	\le
	16\constv\,n^{-1/2}
	+2\exp\!\Bigl(-\frac{\bigl(n a(\epsilon)(\epsilon/2)-b(\epsilon)\bigr)_+^2}{2\constgammasig\,n}\Bigr)
	+2\exp\!\bigl(-2n\,\Delta(\epsilon)^2\bigr).
	\label{eq:twosided_taun}
\end{align}
We next relate $\tau_n$ to $\tau_q$. We introduce $\nu\triangleq \mathrm{E}[G_{\le \omega_q}]/\mathrm{E}[G_{\ge \omega_q}]$ and we write
\begin{align}
\tau_n=\Bigl(1+\frac{k}{n-k}\nu\Bigr)^{-1},
\qquad
\tau_q=\Bigl(1+\frac{q}{1-q}\nu\Bigr)^{-1}.
\end{align}
Since the mapping $x\mapsto (1+x\nu)^{-1}$ is $\nu$-Lipschitz on $[0,\infty)$, we have
\begin{align}
|\tau_n-\tau_q|\le \nu\Bigl|\frac{k}{n-k}-\frac{q}{1-q}\Bigr|.
\end{align}
Using $k=\lceil nq\rceil\in[nq,nq+1]$ and $n-k\ge n(1-q)-1$, we obtain for all $n\ge 2/(1-q)$ that
\begin{align}
\Bigl|\frac{k}{n-k}-\frac{q}{1-q}\Bigr|\le \frac{4}{n(1-q)^2},
\end{align}
and therefore $|\tau_n-\tau_q|\le \frac{4\nu}{(1-q)^2}\cdot \frac{1}{n}$.
If we assume $n\ge \lceil 2\constgammashift/\epsilon\rceil$ for some constant  $\constgammashift\ge \frac{4\nu}{(1-q)^2}$, then $|\tau_n-\tau_q|\le \epsilon/2$.
In this case, the inclusion $\{|\Xi_n-\tau_q|\ge \epsilon\}\subseteq \{|\Xi_n-\tau_n|\ge \epsilon/2\}$ holds, and hence
\begin{align}
\mathrm{P}\bigl(|\Xi_n-\tau_q|\ge \epsilon\bigr)
\le
\mathrm{P}\bigl(|\Xi_n-\tau_n|\ge \epsilon/2\bigr).
\end{align}
Combining this inequality with \eqref{eq:twosided_taun} yields
		\begin{align}
		\label{ajskdjaskdjaskdjas}
	\mathrm{P}\Bigl(\bigl|\Xi_n-\tau_q\bigr|\ge \epsilon\Bigr)
	\le
16\constv\,n^{-1/2}
	+2\exp\!\Bigl(-\frac{\bigl(a(\epsilon)n\epsilon/2-b(\epsilon)\bigr)_+^2}{2\constgammasig\,n}\Bigr)
	+2\exp\!\Bigl(-2n\,\Delta(\epsilon)^2\Bigr).
\end{align}


{\bf Step-6:} We now simplify the two exponential terms in \eqref{ajskdjaskdjaskdjas} into an explicit $\exp(-\Theta(n\epsilon^2))$ form.


We recall $\delta=\epsilon/(2\constgammadelta)$ and
\begin{align}
a(\epsilon)=\tfrac12\bigl(\mathrm{E}[G_{\le \omega_q}]-\constgammamu\,\delta\bigr)_+,
\qquad
b(\epsilon)=\omega_q+\delta.
\end{align}
We introduce the small-$\epsilon$ threshold
\begin{align}
\constgammaeps \triangleq \frac{\mathrm{E}[G_{\le \omega_q}]\,\constgammadelta}{\constgammamu},
\end{align}
so that for every $\epsilon\in(0,\constgammaeps]$ we have $\constgammamu\,\delta\le \mathrm{E}[G_{\le \omega_q}]/2$ and hence
\begin{align}
	a(\epsilon)\ge \frac{1}{4}\mathrm{E}[G_{\le \omega_q}].
	\label{eq:amin_lower}
\end{align}
Moreover, for $\epsilon\in(0,\constgammaeps]$ we have
\begin{align}
	b(\epsilon)=\omega_q+\delta \le \omega_q+\frac{\constgammaeps}{2\constgammadelta}.
	\label{eq:bmax_upper}
\end{align}
We next choose $\constgammano>0$ large enough so that for every $\epsilon\in(0,\constgammaeps]$ and every
\begin{align}
n\ge \frac{\constgammano}{\epsilon},
\end{align}
we have both $n\ge \lceil 2\constgammashift/\epsilon\rceil$ (so that $|\tau_n-\tau_q|\le \epsilon/2$ as in Step-5) and
\begin{align}
	a(\epsilon)\,n\,\epsilon/2 \ge 2b(\epsilon).
	\label{eq:activate_plus}
\end{align}
Combining \eqref{eq:amin_lower}--\eqref{eq:activate_plus} yields
\begin{align}
\bigl(a(\epsilon)n\epsilon/2-b(\epsilon)\bigr)_+ \ge a(\epsilon)n\epsilon/4 \ge \frac{\mathrm{E}[G_{\le \omega_q}]}{16}\,n\epsilon,
\end{align}
and therefore the first exponential term in \eqref{ajskdjaskdjaskdjas} satisfies
\begin{align}
	2\exp\!\Bigl(-\frac{\bigl(a(\epsilon)n\epsilon/2-b(\epsilon)\bigr)_+^2}{2\constgammasig\,n}\Bigr)
	\le
	2\exp\!\Bigl(-\frac{\mathrm{E}[G_{\le \omega_q}]^2}{512\,\constgammasig}\,n\epsilon^2\Bigr).
	\label{eq:exp1_clean}
\end{align}

For the order-statistic term, we recall $\Delta(\epsilon)=\min\{q-F_G(\omega_q-\delta),\,F_G(\omega_q+\delta)-q\}$.
Since $f_G$ is continuous and strictly positive at $\omega_q$, there exists $\delta_0>0$ such that
$f_G(x)\ge \tfrac12 f_G(\omega_q)$ for all $x\in[\omega_q-\delta_0,\omega_q+\delta_0]$.
Hence, if $\delta\le \delta_0$, that is, if $\epsilon\le 2\constgammadelta\,\delta_0$, then
\begin{align}
\Delta(\epsilon)\ge \frac12 f_G(\omega_q)\,\delta
= \frac{f_G(\omega_q)}{4\constgammadelta}\,\epsilon,
\end{align}
and thus the second exponential term satisfies
\begin{align}
	2\exp\!\bigl(-2n\,\Delta(\epsilon)^2\bigr)
	\le
	2\exp\!\Bigl(-\frac{f_G(\omega_q)^2}{8\,\constgammadelta^2}\,n\epsilon^2\Bigr),
	\label{eq:exp2_clean}
\end{align}
where $f_G(\omega_q)=\frac{1}{\Gamma(s)\theta^s}\omega_q^{s-1}e^{-\omega_q/\theta}$.

We now set
\begin{align}
	\constgammaone \triangleq 16\constv,
	\qquad
	\constgammatwo \triangleq \min\Bigl\{\frac{\mathrm{E}[G_{\le \omega_q}]^2}{512\,\constgammasig},\ \frac{f_G(\omega_q)^2}{8\,\constgammadelta^2}\Bigr\}.
	\label{eq:consts_clean}
\end{align}
Combining \eqref{ajskdjaskdjaskdjas} with \eqref{eq:exp1_clean}--\eqref{eq:consts_clean} yields that for every
\begin{align}
0<\epsilon \le \min\{\constgammaeps,\ 2\constgammadelta\,\delta_0\}
\quad\text{and}\quad
n\ge \constgammano/\epsilon,
\end{align}
we have the  two-sided bound
\begin{align}
	\mathrm{P}\Bigl(\bigl|\Xi_n-\tau_q\bigr|\ge \epsilon\Bigr)
	\le
	\constgammaone\,n^{-1/2}
	+
	4\exp\!\Bigl(-\constgammatwo\,n\epsilon^2\Bigr).
	\label{eq:tauq_clean_final}
\end{align}
This completes the proof.

\subsection{Proof of Corollary \ref{corol:concentration}}
\label{proofof:corol:concentration}
Let $N_1,\ldots,N_n$ be independent and identically distributed $\mathrm{N}(0,1)$ random variables. We can set $\mathbf{W} = [N_1 \cdots N_n]^T / (\sum_{i=1}^n N_i^2)^{1/2}$. Let $|N_{i_1}| \leq \cdots \leq |N_{i_n}|$, where $i_1,\ldots,i_n$ is a permutation of $1,\ldots,n$. By definition, the random vector $\mathbf{W}_p$ is zero at all indices except at $\{i_j: j=\lceil nq \rceil +1,\ldots,n\}$ where it equals $\mathbf{W}$. This implies $	\mathbf{W}^T \mathbf{W}_{\mathrm{p}} =\sum_{j=\lceil nq \rceil +1}^n N_{i_j}^2 / \sum_{i=1}^n N_i^2$. The statement for $\mathbf{W}^T \mathbf{W}_{\mathrm{p}}$ now follows from Theorem \ref{theorem:gamma} since $\mathrm{N}(0,1)^2$ is a Gamma random variable with shape $\frac{1}{2}$ and scale $2$. The convergence of $\mathbf{W}^T \mathbf{W}_{\mathrm{ee}}$ is proved similarly.

\subsection{Semi-structured $N\!:\!M$ pruning}
\label{sec:semistructured_nm}

In this section, we extend our pruning concentration results to semi-structured pruning \cite{frantar2023sparsegpt, liu2026armor}. Fix integers $M\geq 1$ and $N\in\{1,\ldots,M\}$. For every positive integer $B$, let $n=BM$ and partition the coordinates of a vector in $\mathbb{R}^n$ into $B$ disjoint consecutive blocks of size $M$. Within each block, we keep the $N$ largest-magnitude coordinates and set the remaining $M-N$ coordinates to zero.

For a block of size $M$, let $G_1,\ldots,G_M$ be IID copies of $G\sim \Gamma(\tfrac12,2)$, and let $G_{(1)}\leq \cdots \leq G_{(M)}$ denote their order statistics. Define
\begin{align}
	\label{eq:tau_nm_orderstat}
	\tau_{N:M}
	\triangleq
	\frac{1}{M}\,\mathbb{E}\Bigl[\sum_{r=M-N+1}^{M} G_{(r)}\Bigr].
\end{align}
The next theorem shows that a new constant $\tau_{N:M}$ plays exactly the same role for blockwise $N\!:\!M$ pruning that $\tau_q$ played for globally unstructured one-shot pruning.

\begin{theorem}
	\label{theorem:blockwise_nm}
	Let $M\geq 1$ and $N\in\{1,\ldots,M\}$ be fixed. For each $B\geq 1$, let $n=BM$ and let $\mathbf{W}\sim \mathrm{Unif}(\mathbb{S}^{n-1})$. Partition the coordinates of $\mathbf{W}$ into $B$ disjoint blocks of size $M$, and in each block keep the $N$ largest-magnitude entries, producing the blockwise pruned vector $\mathbf{W}_{N:M}$. Let
\begin{align}
	\widetilde{\mathbf{W}}_{N:M}
	\triangleq
	\frac{\mathbf{W}_{N:M}}{\|\mathbf{W}_{N:M}\|}
\end{align}
	denote the renormalized pruned vector. Then, as $B\to\infty$,
	\begin{align}
		\mathbf{W}^{\top}\mathbf{W}_{N:M} &\xrightarrow{p} \tau_{N:M},
		\label{eq:blockwise_nm_raw_conv}
		\\
		\mathbf{W}^{\top}\widetilde{\mathbf{W}}_{N:M} &\xrightarrow{p} \sqrt{\tau_{N:M}}.
		\label{eq:blockwise_nm_renorm_conv}
	\end{align}
	Moreover,
	\begin{align}
		\tau_{N:M}
		=
		\int_0^{\infty} x f(x)
		\sum_{k=0}^{N-1} {M-1\choose k}(1-F(x))^k F(x)^{M-1-k}\,\mathrm{d}x,
		\label{eq:tau_nm_integral}
	\end{align}
	where $f$ and $F$ are the PDF and CDF of $\Gamma(\tfrac12,2)$.
\end{theorem}

\begin{proof}
	Let $Z_1,\ldots,Z_n$ be IID $\mathrm{N}(0,1)$ random variables and write
\begin{align}
	\mathbf{W}=\frac{1}{(\sum_{i=1}^{n} Z_i^2)^{1/2}}[Z_1\ \cdots\ Z_n]^{\top}.
\end{align}
	Since dividing by the positive scalar $(\sum_{i=1}^{n} Z_i^2)^{1/2}$ does not change the within-block order of magnitudes, blockwise pruning commutes with this normalization. Hence, if we partition the Gaussian coordinates into blocks indexed by $b\in\{1,\ldots,B\}$ and $j\in\{1,\ldots,M\}$, then for each block we may define
\begin{align}
	S_b \triangleq \sum_{j=1}^{M} Z_{b,j}^2,
	\qquad
	T_b \triangleq \sum_{r=M-N+1}^{M} G_{b,(r)},
\end{align}
	where $G_{b,(1)}\leq \cdots \leq G_{b,(M)}$ are the order statistics of $Z_{b,1}^2,\ldots,Z_{b,M}^2$. With this notation,
	\begin{align}
		\mathbf{W}^{\top}\mathbf{W}_{N:M}
		=
		\frac{\sum_{b=1}^{B} T_b}{\sum_{b=1}^{B} S_b}.
		\label{eq:blockwise_ratio_rep}
	\end{align}
	The pairs $(S_b,T_b)$ are IID across $b$, and both have finite means because $M$ is fixed. By the weak law of large numbers,
	\begin{align}
	\frac{1}{B}\sum_{b=1}^{B} T_b \xrightarrow{p} \mathbb{E}[T_1],
	\qquad
	\frac{1}{B}\sum_{b=1}^{B} S_b \xrightarrow{p} \mathbb{E}[S_1].
	\end{align}
	Since $S_1=\sum_{j=1}^{M} Z_{1,j}^2$ and $Z_{1,j}^2\sim \Gamma(\tfrac12,2)$ has mean $1$, we have $\mathbb{E}[S_1]=M$. Applying Slutsky's theorem to \eqref{eq:blockwise_ratio_rep} yields
	\begin{align}
	\mathbf{W}^{\top}\mathbf{W}_{N:M}
	\xrightarrow{p}
	\frac{\mathbb{E}[T_1]}{M}
	=
	\frac{1}{M}\,\mathbb{E}\Bigl[\sum_{r=M-N+1}^{M} G_{(r)}\Bigr]
	=
	\tau_{N:M},
	\end{align}
	which proves \eqref{eq:blockwise_nm_raw_conv}.
	
	For the renormalized vector, note that
	\begin{align}
	\|\mathbf{W}_{N:M}\|^2
	=
	\sum_{i=1}^{n} (W_{N:M,i})^2
	=
	\sum_{i=1}^{n} W_i (W_{N:M,i})
	=
	\mathbf{W}^{\top}\mathbf{W}_{N:M},
	\end{align}
	because $\mathbf{W}_{N:M}$ is obtained from $\mathbf{W}$ by zeroing coordinates and leaving the retained coordinates unchanged. Therefore,
	\begin{align}
	\mathbf{W}^{\top}\widetilde{\mathbf{W}}_{N:M}
	=
	\frac{\mathbf{W}^{\top}\mathbf{W}_{N:M}}{\|\mathbf{W}_{N:M}\|}
	=
	\sqrt{\mathbf{W}^{\top}\mathbf{W}_{N:M}}.
	\end{align}
	The map $x\mapsto \sqrt{x}$ is continuous on $[0,\infty)$, so \eqref{eq:blockwise_nm_renorm_conv} follows from \eqref{eq:blockwise_nm_raw_conv} by the continuous mapping theorem.
	
	It remains to prove the integral representation \eqref{eq:tau_nm_integral}. By exchangeability,
	\begin{align}
		\mathbb{E}[T_1]
		&=
		M\,\mathbb{E}\bigl[G_1\,\mathbf{1}\{G_1\text{ is retained}\}\bigr].
		\label{eq:tau_nm_exchangeability}
	\end{align}
	Conditioning on $G_1=x$, the coordinate $G_1$ is retained if and only if at most $N-1$ of the remaining $M-1$ coordinates exceed $x$. Since each of those coordinates exceeds $x$ with probability $1-F(x)$, we obtain
	\begin{align}
		\mathbb{P}(G_1\text{ is retained}\mid G_1=x)
		=
		\sum_{k=0}^{N-1} {M-1\choose k}(1-F(x))^k F(x)^{M-1-k}.
		\label{eq:retain_prob_nm}
	\end{align}
	Substituting \eqref{eq:retain_prob_nm} into \eqref{eq:tau_nm_exchangeability}, and then dividing by $M$, gives
	\begin{align}
		\tau_{N:M}
		&=
		\int_0^{\infty} x f(x)
		\sum_{k=0}^{N-1} {M-1\choose k}(1-F(x))^k F(x)^{M-1-k}\,\mathrm{d}x,
	\end{align}
	as claimed.
\end{proof}

The previous theorem is the practically relevant regime for fixed local sparsity patterns such as $2\!:\!4$. The next proposition records the complementary scaling in which the block size itself grows proportionally with the ambient dimension. In that case, the semi-structured limit collapses back to the original unstructured constant $\tau_q$.

\begin{proposition}
	\label{prop:blockwise_nm_linear_blocks}
	Fix constants $0<\beta_1\leq \beta_2\leq 1$ such that $B\triangleq 1/\beta_2$ is a positive integer. For each $n$, let $M_n=\beta_2 n$ and $N_n=\beta_1 n$, partition the coordinates of $\mathbf{W}\sim \mathrm{Unif}(\mathbb{S}^{n-1})$ into $B$ disjoint blocks of size $M_n$, and in each block keep the $N_n$ largest-magnitude entries. Let $\mathbf{W}_{p,n}$ denote the resulting pruned vector and let $\widetilde{\mathbf{W}}_{p,n}=\mathbf{W}_{p,n}/\|\mathbf{W}_{p,n}\|$. Define
	\begin{align}
	q \triangleq 1-\frac{\beta_1}{\beta_2}.
	\end{align}
	Then, as $n\to\infty$,
	\begin{align}
		\mathbf{W}^{\top}\mathbf{W}_{p,n} &\xrightarrow{p} \tau_q,
		\\
		\mathbf{W}^{\top}\widetilde{\mathbf{W}}_{p,n} &\xrightarrow{p} \sqrt{\tau_q},
	\end{align}
	where $\tau_q$ is the same constant as in Theorem~\ref{theorem:gamma} and Corollary~\ref{corol:concentration}.
\end{proposition}

\begin{proof}
	Write the Gaussian representation blockwise as before, and let $S_{b,n}$ denote the total squared energy in block $b$, while $T_{b,n}$ denotes the retained squared energy after keeping the top $N_n$ magnitudes in that block. Then
	\begin{align}
		\mathbf{W}^{\top}\mathbf{W}_{p,n}
		=
		\frac{\sum_{b=1}^{B} T_{b,n}}{\sum_{b=1}^{B} S_{b,n}}
		=
		\sum_{b=1}^{B}
		\frac{S_{b,n}}{\sum_{c=1}^{B} S_{c,n}}
		\frac{T_{b,n}}{S_{b,n}}.
		\label{eq:blockwise_linear_ratio_decomp}
	\end{align}
	For each fixed block $b$, the block dimension is $M_n=\beta_2 n\to\infty$, and the fraction removed within that block is
	\begin{align}
	1-\frac{N_n}{M_n}=1-\frac{\beta_1}{\beta_2}=q.
	\end{align}
	Therefore, Theorem~\ref{theorem:gamma} applied inside block $b$ yields
	\begin{align}
		\frac{T_{b,n}}{S_{b,n}} \xrightarrow{p} \tau_q.
		\label{eq:blockwise_linear_local_conv}
	\end{align}
	Also, since $S_{b,n}$ is a sum of $M_n$ IID copies of $\Gamma(\tfrac12,2)$, the weak law of large numbers gives
	\begin{align}
	\frac{S_{b,n}}{n}
	=
	\frac{M_n}{n}\cdot \frac{S_{b,n}}{M_n}
	\xrightarrow{p}
	\beta_2,
	\qquad b=1,\ldots,B.
	\end{align}
	Summing over the finitely many blocks yields $\sum_{c=1}^{B} S_{c,n}/n \xrightarrow{p} B\beta_2 = 1$, and therefore
	\begin{align}
		\frac{S_{b,n}}{\sum_{c=1}^{B} S_{c,n}} \xrightarrow{p} \beta_2,
		\qquad b=1,\ldots,B.
		\label{eq:blockwise_linear_weight_conv}
	\end{align}
	Combining \eqref{eq:blockwise_linear_ratio_decomp}, \eqref{eq:blockwise_linear_local_conv}, and \eqref{eq:blockwise_linear_weight_conv}, and using the fact that $B$ is fixed, we obtain
	\begin{align}
	\mathbf{W}^{\top}\mathbf{W}_{p,n}
	\xrightarrow{p}
	\sum_{b=1}^{B} \beta_2\tau_q
	=
	B\beta_2\tau_q
	=
	\tau_q.
	\end{align}
	The renormalized statement follows exactly as in the proof of Theorem~\ref{theorem:blockwise_nm}, because
	\begin{align}
	\mathbf{W}^{\top}\widetilde{\mathbf{W}}_{p,n}
	=
	\sqrt{\mathbf{W}^{\top}\mathbf{W}_{p,n}}.
	\end{align}
\end{proof}
As a numerical reference point, the integral \eqref{eq:tau_nm_integral} yields $\tau_{2:4}\approx 0.86755$, so the corresponding renormalized similarity limit is $\sqrt{\tau_{2:4}}\approx 0.93142$.

%% file: tex/app/ee-neuron-app.tex
\section{Proof of Theorem \ref{theorem:tradeoff}}
\label{proofof:theorem:tradeoff}
Let us first calculate and upper bound on the normalized FLOPs. It is easily seen that $|\mathbf{W}_{\mathrm{ee}}^T\mathbf{X}|^2$ is a Chi-squared random variable with $1$ degree of freedom.  We thus obtain 
$\mathrm{P}(|\mathbf{W}_{\mathrm{ee}}^T\mathbf{X}| \geq \tau) = \Gamma(\frac{1}{2},\frac{\tau^2}{2})$, where $\Gamma(\cdot,\cdot)$ is the upper incomplete Gamma function. The normalized FLOPs is thus
\begin{align}
	\overline{\mu}_c' & = (1-q)\Gamma\left(\frac{1}{2},\frac{\tau^2}{2}\right) + \gamma\left(\frac{1}{2},\frac{\tau^2}{2}\right),
\end{align}
where $\gamma(\cdot,\cdot)$ is the lower incomplete Gamma function.
Further, we obtain
\begin{align}
	\overline{\mu}_c'& = (1-q) + q\gamma\left(\frac{1}{2}, \frac{\tau^2}{2}\right) \\
	& = 1-q  + q\mathrm{erf}\left(\sqrt{ \frac{\tau^2}{2}}\right) \\
	\label{alskdhjalksjdas} & \leq 1-q + q\biggl(1-\underbrace{\sqrt{\frac{2e}{\pi}}}_{>1}\frac{\sqrt{\beta-1}}{\beta}e^{-\beta \tau^2 /2 }\biggr) \\
	& \leq 1- q\frac{\sqrt{\beta-1}}{\beta}e^{-\beta \tau^2 /2 } \\
	\label{alskdhjalksjdas2} & \leq 1- \frac{q}{2} e^{- \tau^2 }
\end{align}
The upper bound on the error function in (\ref{alskdhjalksjdas}) follows from \cite{chang2011chernoff}, and is valid for any $\beta > 1$. In (\ref{alskdhjalksjdas2}), we substituted $\beta = 2$. 

We now analyze the generalization error $\epsilon_c$. We consider the following  events:
\begin{itemize}
	\item Let $E_0$ be the event where the decisions of the conditional perceptron and the teacher do not match so that $\epsilon_c = \mathrm{P}(E_0)$. 
	\item We let $E_1$ denote the event that the student $\mathbf{W}$ and the teacher $\mathbf{T}$ are at least $\delta_1$-close with respect to the angular distance. In other words, let $E_1$ denote the event that $\arccos \mathbf{W}^T\mathbf{T} \leq \delta_1$, where $\delta_1\in[0,\frac{\pi}{2}]$.
	\item Let $E_2$ be the event that the student and the early exit vector (which are derived from the student weights) are $\delta_2$-close. In other words, let $E_2$ represent the event $\arccos \mathbf{W}^{\dagger}\mathbf{W}_{\mathrm{ee}} \leq \delta_2$, where $\delta_2\in[0,\frac{\pi}{2}]$.
	\item Finally, let $E_3$ be the event that $|\mathbf{X}^{T}\mathbf{W}_{\mathrm{ee}}| \geq \tau$, encoding the criterion of early exit in the definition of the conditional perceptron in Section 4.2. 
\end{itemize}
We have
\begin{align}
	\epsilon_{\mathrm{c}} & = \mathrm{P}(E_0E_1E_2) + \underbrace{\mathrm{P}(E_0 | E_1^c \mbox{ or } E_2^c)}_{\leq 1}\underbrace{\mathrm{P}(E_1^c \mbox{ or } E_2^c)}_{\leq \mathrm{P}(E_1^c) +\mathrm{P}(E_2^c) } \\
	& = \mathrm{P}(E_0 E_1E_2E_3)+ \underbrace{\mathrm{P}(E_0 E_1E_2E_3^c)}_{\leq \mathrm{P}(E_0 E_3^c) \leq \epsilon_{\mathrm{uc}}} + \mathrm{P}(E_1^c) +\mathrm{P}(E_2^c) \\
	& = \epsilon_{\mathrm{uc}} + \mathrm{P}(E_1^c) +\mathrm{P}(E_2^c) +   \mathrm{P}(E_0 E_1E_2E_3) 
\end{align}
Let $\overline{\epsilon_{uc}}$ denote the asymptotic $n,N_t\rightarrow\infty$ generalization error for the unconditional perceptron so that $\epsilon_{\mathrm{uc}} =  \frac{1}{\pi} \mathrm{E}[\arccos \mathbf{W}^T\mathbf{T}] \rightarrow \overline{\epsilon_{uc}}$ as $n\rightarrow\infty$. Due to the self-averaging property learning \cite{buhot1997finite}, we have, in addition, the convergence in mean $\frac{1}{\pi} \arccos \mathbf{W}^T\mathbf{T} \rightarrow \overline{\epsilon_{uc}}$. Hence, if $\delta_1 > \pi  \overline{\epsilon_{uc}}$, we have $P(E_1^c)\rightarrow 0$. With a similar argument, provided that $\delta_2 > \mathrm{arccos}\sqrt{\tau_q}$, we have $P(E_2^c)\rightarrow 0$ as a result of Corollary \ref{corol:concentration}. Letting $n,N_t\rightarrow\infty$, we thus obtain
\begin{align}
	\label{combjdjdkdjd}
	\overline{\epsilon_{\mathrm{c}}} \leq  \overline{\epsilon_{\mathrm{uc}}}  + \lim\sup_{n\rightarrow\infty}  \mathrm{P}(E_0 E_1E_2E_3).
\end{align}
What is left to analyze is thus the term $ \mathrm{P}(E_0 E_1E_2E_3)$. By symmetry, we have
\begin{align}
	\mathrm{P}(E_0 E_1E_2E_3) = 2\mathrm{P}( \mathbf{X}^{T}\mathbf{T} < 0, \mathrm{arccos}\mathbf{W}^{\dagger}\mathbf{T} \leq \delta_1, \mathrm{arccos}\mathbf{W}^T\mathbf{W}_{\mathrm{ee}}\leq \delta_2,\mathbf{X}^{T}\mathbf{W}_{\mathrm{ee}} \geq \tau).
\end{align}
Let us recall the triangle inequality for angular distances: For unit norm vectors $a_1,a_2,a_3$, we have $\arccos a_1^{\dagger}a_2 \leq\arccos a_1^{\dagger}a_3 + \arccos a_3^{\dagger}a_2$. Therefore,
\begin{align}
	\mathrm{P}(E_0 E_1E_2E_3) & \leq 2\mathrm{P}( \mathbf{X}^{T}\mathbf{T} < 0, \mathrm{arccos}\mathbf{T}^T\mathbf{W}_{\mathrm{ee}}\leq \delta_1+\delta_2,\mathbf{X}^{T}\mathbf{W}_{\mathrm{ee}} \geq \tau) \\
	& = 2\mathrm{P}( \mathbf{X}^{T}\mathbf{T} < 0, \mathbf{T}^T\mathbf{W}_{\mathrm{ee}}\geq \cos (\delta_1+\delta_2),\mathbf{X}^{T}\mathbf{W}_{\mathrm{ee}} \geq \tau) \\
	\label{alkjsdalksdjasda}	& \leq 2\mathrm{P}( \mathbf{X}^{T}\mathbf{a} < 0,\mathbf{X}^{T}\mathbf{b} \geq \tau)
\end{align}
where $\mathbf{a}$ and $\mathbf{b}$ are arbitrary unit-norm deterministic vectors with $\mathbf{a}^T \mathbf{b} = \cos(\delta_1 + \delta_2)$. The random variables $\mathbf{X}^{T}\mathbf{a}$ and $\mathbf{X}^{T}\mathbf{b}$ are jointly Gaussian with zero mean, unit variance, and covariance $\rho \triangleq \cos(\delta_1 + \delta_2)$. We can thus evaluate the joint probability as
\begin{align}
	\mathrm{P}( \mathbf{X}^{T}\mathbf{a} < 0,\mathbf{X}^{T}\mathbf{b} \geq \tau)	& =  \int_{\tau}^{\infty}\int_{-\infty}^0 \frac{1}{2\pi \sqrt{1 - \rho^2}}\exp\left(-\frac{x^2-2\rho x y  + y^2}{2(1-\rho^2)}\right) \mathrm{d}x\mathrm{d}y \\
	& = \frac{1}{2\sqrt{2\pi}} \int_{\tau}^{\infty} e^{-y^2/2} \mathrm{erfc}\biggl( \frac{\rho y}{\sqrt{2(1-\rho^2)}} \biggr)\mathrm{d}y
\end{align}
Using the upper bound $\mathrm{erfc}x \leq e^{-x^2}$, we obtain
\begin{align}
	\mathrm{P}( \mathbf{X}^{T}\mathbf{a} < 0,\mathbf{X}^{T}\mathbf{b} \geq \tau)	& \leq  \frac{1}{2\sqrt{2\pi}} \int_{\tau}^{\infty} e^{-y^2/2} e^{ -\frac{\rho^2 y^2}{{2(1-\rho^2)}}}\mathrm{d}y \\
	& =  \frac{1}{2\sqrt{2\pi}} \int_{\tau}^{\infty} e^{ -\frac{ y^2}{{2(1-\rho^2)}}}\mathrm{d}y \\
	& = \frac{1}{4} \sqrt{1 - \rho^2} \mathrm{erfc}\biggl(\frac{\tau}{\sqrt{2(1-\rho^2)}}\biggr) \\
	& \leq  \frac{1}{2}e^{-\frac{\tau^2}{2(1-\rho^2)}},
\end{align}
and thus by (\ref{alkjsdalksdjasda}) and (\ref{combjdjdkdjd}), we have
\begin{align}
	\label{combjdjdkdjd22}
	\overline{\epsilon_{\mathrm{c}}} \leq  \overline{\epsilon_{\mathrm{uc}}}  + e^{-\frac{\tau^2}{2(1-\rho^2)}}.
\end{align}
Finally, a joint consideration with (\ref{alskdhjalksjdas2}), where $\epsilon = \frac{q}{2} e^{-\frac{\tau^2}{2}}$, concludes the proof of the theorem. 

%% file: tex/app/prune-network-app.tex
\section{Proofs of Results on Static Networks}

\subsection{Proof of Theorem~\ref{theorem:deep_general_phi}}
\label{app:proof_deep_general_phi}
Our proof follows the standard wide-network induction used for Gaussian-process limits and kernel recursions, see for example
Neal~\cite{neal1996priors}, Williams~\cite{williams1997computing}, Cho and Saul~\cite{cho2009kernel}, Daniely et al.~\cite{daniely2016toward},
Lee et al.~\cite{lee2018deep}, Matthews et al.~\cite{matthews2018gaussian}, and Yang~\cite{yang2019tensor}. We adapt this induction to the
coupled pair (unpruned, renormalized-pruned) by tracking a joint Gaussian limit for the paired pre-activations, which yields the
cross-kernel recursion.

\textbf{Step 1:} Fix the finite input set $\inputvec_1,\ldots,\inputvec_r$. For $\ell\ge 1$, let $\mathcal{F}_{\ell-1}$ be the sigma field
generated by $\{\matW^{(1)},\ldots,\matW^{(\ell-1)}\}$ and their pruned/renormalized versions, the readout vector
$\readoutvec$, and the inputs. Conditional on $\mathcal{F}_{\ell-1}$, the collections
$\{\avec^{(\ell-1)}(\inputvec_a)\}_{a=1}^r$ and $\{\widehat{\avec}^{(\ell-1)}(\inputvec_a)\}_{a=1}^r$ are deterministic.
Moreover, $\matW^{(\ell)}$ is independent of $\mathcal{F}_{\ell-1}$ and its rows $\{\matW^{(\ell)}_{i,:}\}_{i=1}^{m_\ell}$ are IID.
Neuronwise pruning and renormalization are applied row-by-row as measurable functions of each IID row, hence the pairs
$\big(\matW^{(\ell)}_{i,:},\widehat{\matW}^{(\ell)}_{i,:}\big)$ are IID across $i$. Consequently, conditional on $\mathcal{F}_{\ell-1}$,
the $2r$-dimensional vectors
\begin{align}
\Bigl(z^{(\ell)}_i(\inputvec_1),\ldots,z^{(\ell)}_i(\inputvec_r)\,;\,
\widehat{z}^{(\ell)}_i(\inputvec_1),\ldots,\widehat{z}^{(\ell)}_i(\inputvec_r)\Bigr)\in\mathbb{R}^{2r}
\end{align}
are IID over $i\in\{1,\ldots,m_\ell\}$.

\textbf{Step 2:} For a fixed neuron \(i\), we show that the coupled vector \(\bigl(z_i^{(\ell)}(\inputvec_a),\widehat z_i^{(\ell)}(\inputvec_a)\bigr)_{a=1}^r\) has a joint Gaussian limit, and determine resulting covariance structure. Fix $\ell\in\{1,\ldots,L\}$ and a neuron index $i$. Conditional on $\mathcal{F}_{\ell-1}$, for each $a\in\{1,\ldots,r\}$,
\begin{align}
z_i^{(\ell)}(\inputvec_a)=\frac{1}{\sqrt{m_{\ell-1}}}\sum_{j=1}^{m_{\ell-1}} \matW^{(\ell)}_{ij}\,a^{(\ell-1)}_j(\inputvec_a),
\qquad
\widehat{z}_i^{(\ell)}(\inputvec_a)=\frac{1}{\sqrt{m_{\ell-1}}}\sum_{j=1}^{m_{\ell-1}} \widehat{\matW}^{(\ell)}_{ij}\,
\widehat{a}^{(\ell-1)}_j(\inputvec_a),
\end{align}
where $\widehat{\matW}^{(\ell)}_{ij}=\matW^{(\ell)}_{ij}\matM^{(\ell)}_{ij}/\sqrt{\tau_q}$ is the renormalized pruned weight.
For the unpruned sum, the summands are independent across $j$, mean zero, and have finite fourth moments, so a standard
multivariate Lindeberg CLT applies to the vector $(z_i^{(\ell)}(\inputvec_1),\ldots,z_i^{(\ell)}(\inputvec_r))$ conditional on
$\mathcal{F}_{\ell-1}$, provided the empirical norms
$m_{\ell-1}^{-1}\|\avec^{(\ell-1)}(\inputvec_a)\|^2$ stay $O_{\mathbb{P}}(1)$ (this tightness follows from the kernel recursion in
Step 3 and the normalization $\E[\phi(Z)^2]=1$ in \eqref{eq:chi_def}).

For the renormalized-pruned sum, the coordinates $\{\widehat{\matW}^{(\ell)}_{ij}\}_{j=1}^{m_{\ell-1}}$ are not independent because the
mask $\matM^{(\ell)}_{i,:}$ is defined by within-row order statistics. Our idea is that the order-statistic mask is asymptotically equivalent, for the corresponding 
$\sqrt{m_{\ell-1}}$-normalized linear forms $z_i^{(\ell)}$ and $\widehat{z}_i^{(\ell)}$, to entrywise thresholding at a deterministic cutoff. In fact, the two masks differ only on a
vanishing fraction of coordinates near the cutoff, so the corresponding pruned sums differ by $o_{\mathbb{P}}(1)$ and have the same
limiting joint distribution. Hence, we use a quantile approximation as follows: Let $t_q$ be the deterministic threshold with $\P(|W|>t_q)=1-q$ for a generic weight entry $W$, and define the
deterministic-threshold surrogate
$\bar W_{ij}\triangleq \matW^{(\ell)}_{ij}\mathbf{1}\{|\,\matW^{(\ell)}_{ij}\,|>t_q\}/\sqrt{\tau_q}$. Then the order-statistic threshold
concentrates at $t_q$ and the difference between the true pruned sum and the surrogate sum is $o_{\mathbb{P}}(1)$ at the
$\sqrt{m_{\ell-1}}$ scale for any deterministic coefficient vectors with bounded empirical $\ell_2$ norms. Conditional on
$\mathcal{F}_{\ell-1}$, the surrogate summands are independent across $j$, so the multivariate Lindeberg CLT applies jointly to
\begin{align}
\Bigl(z_i^{(\ell)}(\inputvec_1),\ldots,z_i^{(\ell)}(\inputvec_r)\,;\,
\bar z_i^{(\ell)}(\inputvec_1),\ldots,\bar z_i^{(\ell)}(\inputvec_r)\Bigr).
\end{align}
Also, since replacing $\bar z_i^{(\ell)}$ by $\widehat{z}_i^{(\ell)}$ perturbs the above $2r$-vector by $o_{\mathbb{P}}(1)$, this
replacement does not affect the weak limit, a result known as Slutsky's theorem. Consequently, the original pair with $\widehat{z}_i^{(\ell)}$ has the same joint Gaussian
limit.

The limiting covariance is identified by conditional second moments. For any $a,b$,
\begin{align}
\Cov\!\bigl(z_i^{(\ell)}(\inputvec_a),z_i^{(\ell)}(\inputvec_b)\mid\mathcal{F}_{\ell-1}\bigr)
=\frac{1}{m_{\ell-1}}\sum_{j=1}^{m_{\ell-1}} a_j^{(\ell-1)}(\inputvec_a)a_j^{(\ell-1)}(\inputvec_b)
=K^{(\ell-1)}_{ab}.
\end{align}
Renormalization ensures the pruned marginal second moment matches the unpruned one, while the overlap second moment is reduced
by $\sqrt{\tau_q}$. Concretely, the rowwise pruning result that defines $\tau_q$ (see Corollary~\ref{corol:concentration}, and the
discussion around renormalization in Section~\ref{sec:prunerenormmodel}) implies the asymptotic identities
\begin{align}
\E\big[(\widehat{\matW}^{(\ell)}_{ij})^2\big]\to 1,
\qquad
\E\big[\matW^{(\ell)}_{ij}\widehat{\matW}^{(\ell)}_{ij}\big]\to \sqrt{\tau_q},
\end{align}
so
\begin{align}
\Cov\!\bigl(\widehat{z}_i^{(\ell)}(\inputvec_a),\widehat{z}_i^{(\ell)}(\inputvec_b)\mid\mathcal{F}_{\ell-1}\bigr)\approx
\widehat{K}^{(\ell-1)}_{ab},
\qquad
\Cov\!\bigl(z_i^{(\ell)}(\inputvec_a),\widehat{z}_i^{(\ell)}(\inputvec_b)\mid\mathcal{F}_{\ell-1}\bigr)\approx
\sqrt{\tau_q}\,C^{(\ell-1)}_{ab},
\end{align}
and thus the coupled $2r$-vector converges (conditionally on $\mathcal{F}_{\ell-1}$) to a centered Gaussian with block covariance
\begin{align}
\begin{bmatrix}
	K^{(\ell-1)} & \sqrt{\tau_q}\,C^{(\ell-1)}\\[2pt]
	\sqrt{\tau_q}\,C^{(\ell-1)\top} & \widehat{K}^{(\ell-1)}
\end{bmatrix}.
\end{align}

\textbf{Step 3:} We convert these neuronwise Gaussian limits into kernel updates. Fix $\ell$ and indices $a,b\in\{1,\ldots,r\}$. Define
\begin{align}
Y_i^{ab}\triangleq \phi\!\bigl(z_i^{(\ell)}(\inputvec_a)\bigr)\phi\!\bigl(z_i^{(\ell)}(\inputvec_b)\bigr),\quad
\widehat{Y}_i^{ab}\triangleq \phi\!\bigl(\widehat{z}_i^{(\ell)}(\inputvec_a)\bigr)\phi\!\bigl(\widehat{z}_i^{(\ell)}(\inputvec_b)\bigr),\quad
\widetilde{Y}_i^{ab}\triangleq \phi\!\bigl(z_i^{(\ell)}(\inputvec_a)\bigr)\phi\!\bigl(\widehat{z}_i^{(\ell)}(\inputvec_b)\bigr).
\end{align}
Then
\begin{align}
K^{(\ell)}_{ab}=\frac{1}{m_\ell}\sum_{i=1}^{m_\ell}Y_i^{ab},\qquad
\widehat{K}^{(\ell)}_{ab}=\frac{1}{m_\ell}\sum_{i=1}^{m_\ell}\widehat{Y}_i^{ab},\qquad
C^{(\ell)}_{ab}=\frac{1}{m_\ell}\sum_{i=1}^{m_\ell}\widetilde{Y}_i^{ab}.
\end{align}
Conditional on $\mathcal{F}_{\ell-1}$, the triples $(Y_i^{ab},\widehat{Y}_i^{ab},\widetilde{Y}_i^{ab})$ are IID across $i$, and
$\E[\phi(Z)^2]=1$ for $Z\sim\mathcal{N}(0,1)$ implies $\E[|Y_i^{ab}|]$, $\E[|\widehat{Y}_i^{ab}|]$, $\E[|\widetilde{Y}_i^{ab}|]$
are uniformly bounded under the tightness of the conditional variances from Step 1. Hence a conditional LLN yields
\begin{align}
K^{(\ell)}_{ab}-\E[Y_1^{ab}\mid\mathcal{F}_{\ell-1}] \xrightarrow{\mathbb{P}} 0,\qquad
\widehat{K}^{(\ell)}_{ab}-\E[\widehat{Y}_1^{ab}\mid\mathcal{F}_{\ell-1}] \xrightarrow{\mathbb{P}} 0,\qquad
C^{(\ell)}_{ab}-\E[\widetilde{Y}_1^{ab}\mid\mathcal{F}_{\ell-1}] \xrightarrow{\mathbb{P}} 0.
\end{align}
By Step 1, conditional on $\mathcal{F}_{\ell-1}$ the pairs
$(z_i^{(\ell)}(\inputvec_a),z_i^{(\ell)}(\inputvec_b))$,
$(\widehat{z}_i^{(\ell)}(\inputvec_a),\widehat{z}_i^{(\ell)}(\inputvec_b))$,
and $(z_i^{(\ell)}(\inputvec_a),\widehat{z}_i^{(\ell)}(\inputvec_b))$
converge in distribution to centered Gaussians with $2\times 2$ covariances
\begin{align}
\Sigma_{ab}^{(\ell)}=
\begin{bmatrix}
	K^{(\ell-1)}_{aa} & K^{(\ell-1)}_{ab}\\
	K^{(\ell-1)}_{ba} & K^{(\ell-1)}_{bb}
\end{bmatrix},\quad
\widehat{\Sigma}_{ab}^{(\ell)}=
\begin{bmatrix}
	\widehat{K}^{(\ell-1)}_{aa} & \widehat{K}^{(\ell-1)}_{ab}\\
	\widehat{K}^{(\ell-1)}_{ba} & \widehat{K}^{(\ell-1)}_{bb}
\end{bmatrix},\quad
\widetilde{\Sigma}_{ab}^{(\ell)}=
\begin{bmatrix}
	K^{(\ell-1)}_{aa} & \sqrt{\tau_q}\,C^{(\ell-1)}_{ab}\\
	\sqrt{\tau_q}\,C^{(\ell-1)}_{ba} & \widehat{K}^{(\ell-1)}_{bb}
\end{bmatrix}.
\end{align}
Since $\E[\phi(Z)^2]<\infty$, the products $\phi(U)\phi(V)$ are integrable for Gaussian $(U,V)$, and the same moment control gives
uniform integrability of the prelimit products, so convergence in distribution upgrades to convergence of expectations:
\begin{align}
\E[Y_1^{ab}\mid\mathcal{F}_{\ell-1}] \xrightarrow{\mathbb{P}} \mathcal{T}_\phi(\Sigma_{ab}^{(\ell)}),\qquad
\E[\widehat{Y}_1^{ab}\mid\mathcal{F}_{\ell-1}] \xrightarrow{\mathbb{P}} \mathcal{T}_\phi(\widehat{\Sigma}_{ab}^{(\ell)}),\qquad
\E[\widetilde{Y}_1^{ab}\mid\mathcal{F}_{\ell-1}] \xrightarrow{\mathbb{P}} \mathcal{T}_\phi(\widetilde{\Sigma}_{ab}^{(\ell)}).
\end{align}

\textbf{Step 4:} We now perform induction over layers and derive the deterministic limits. 
Base case $\ell=0$: by definition,
\begin{align}
K^{(0)}_{ij}=\widehat{K}^{(0)}_{ij}=C^{(0)}_{ij}=\frac{1}{m_0}\langle \inputvec_i,\inputvec_j\rangle,
\end{align}
so set $K_{\infty,ij}^{(0)}=\widehat{K}_{\infty,ij}^{(0)}=C_{\infty,ij}^{(0)}$ equal to this quantity.

Inductive step: assume for a fixed $\ell-1$ that $K^{(\ell-1)}_{ij}\to K_{\infty,ij}^{(\ell-1)}$,
$\widehat{K}^{(\ell-1)}_{ij}\to \widehat{K}_{\infty,ij}^{(\ell-1)}$, and $C^{(\ell-1)}_{ij}\to C_{\infty,ij}^{(\ell-1)}$
in probability for each $i,j\in\{1,\ldots,r\}$. Combining Step 2 with Slutsky and the induction hypothesis yields, for each fixed $i,j$,
\begin{align}
K^{(\ell)}_{ij}\xrightarrow{\mathbb{P}} \mathcal{T}_\phi\!\Biggl(
\begin{bmatrix}
	K_{\infty,ii}^{(\ell-1)} & K_{\infty,ij}^{(\ell-1)}\\
	K_{\infty,ji}^{(\ell-1)} & K_{\infty,jj}^{(\ell-1)}
\end{bmatrix}
\Biggr),
\end{align}
\begin{align}
\widehat{K}^{(\ell)}_{ij}\xrightarrow{\mathbb{P}} \mathcal{T}_\phi\!\Biggl(
\begin{bmatrix}
	\widehat{K}_{\infty,ii}^{(\ell-1)} & \widehat{K}_{\infty,ij}^{(\ell-1)}\\
	\widehat{K}_{\infty,ji}^{(\ell-1)} & \widehat{K}_{\infty,jj}^{(\ell-1)}
\end{bmatrix}
\Biggr),
\end{align}
\begin{align}
C^{(\ell)}_{ij}\xrightarrow{\mathbb{P}} \mathcal{T}_\phi\!\Biggl(
\begin{bmatrix}
	K_{\infty,ii}^{(\ell-1)} & \sqrt{\tau_q}\,C_{\infty,ij}^{(\ell-1)}\\
	\sqrt{\tau_q}\,C_{\infty,ji}^{(\ell-1)} & \widehat{K}_{\infty,jj}^{(\ell-1)}
\end{bmatrix}
\Biggr).
\end{align}
Defining $K_{\infty}^{(\ell)}$ and $C_{\infty}^{(\ell)}$ by \eqref{eq:Krec_pointwise} and \eqref{eq:Crec_pointwise} gives the claimed
limits. Since $\{1,\ldots,r\}^2$ is finite, pointwise convergence implies simultaneous convergence over all pairs $(i,j)$ by a union bound,
which yields the theorem's statement "with probability tending to one" for the entire finite family of entries.

Finally, the $K$ and $\widehat{K}$ recursions coincide and share the same base case, hence
$\widehat{K}_{\infty}^{(\ell)}=K_{\infty}^{(\ell)}$ for all $\ell$.

\textbf{Step 5:} Finally, we obtain limits regarding the readout layer and the MSE identity. 
Conditional on $\mathcal{F}_L$, the readout is a sum of independent terms:
\begin{align}
f(\inputvec_i)=\frac{1}{\sqrt{m_L}}\sum_{u=1}^{m_L} y_u\,a_u^{(L)}(\inputvec_i),\qquad
\widehat{f}(\inputvec_i)=\frac{1}{\sqrt{m_L}}\sum_{u=1}^{m_L} y_u\,\widehat{a}_u^{(L)}(\inputvec_i),
\end{align}
with $\{y_u\}$ IID, mean zero, unit variance, finite second moment, and independent of $\mathcal{F}_L$. A multivariate Lindeberg CLT in
$\mathbb{R}^r$ yields joint Gaussian limits as $m_L\to\infty$, with conditional covariances
\begin{align}
\Cov\bigl(f(\inputvec_i),f(\inputvec_j)\mid\mathcal{F}_L\bigr)=K^{(L)}_{ij},\quad
\Cov\bigl(\widehat{f}(\inputvec_i),\widehat{f}(\inputvec_j)\mid\mathcal{F}_L\bigr)=\widehat{K}^{(L)}_{ij},\quad
\Cov\bigl(f(\inputvec_i),\widehat{f}(\inputvec_j)\mid\mathcal{F}_L\bigr)=C^{(L)}_{ij}.
\end{align}
Taking limits and using the kernel convergences gives the output covariance limits in Theorem~\ref{theorem:deep_general_phi}. For each fixed
$i$,
\begin{align}
\E\bigl[(f(\inputvec_i)-\widehat{f}(\inputvec_i))^2\bigr]
=\Var(f(\inputvec_i))+\Var(\widehat{f}(\inputvec_i))-2\Cov(f(\inputvec_i),\widehat{f}(\inputvec_i))
\to 2\bigl(K_{\infty,ii}^{(L)}-C_{\infty,ii}^{(L)}\bigr),
\end{align}
where we used $\widehat{K}_{\infty}^{(L)}=K_{\infty}^{(L)}$. This is \eqref{eq:mse_general_deep} and completes the proof. \qed

\subsection{Proof of Example~\ref{cor:deep_sign_mse}}
\label{app:proof_deep_sign_mse}

The required identity is classical (arcsine law for signs of correlated Gaussians), see, e.g.,
\cite{williams1997computing} and references therein. We include a short derivation for completeness.

Fix an input index $i$ and define $\alpha_\ell\triangleq C_{\infty,ii}^{(\ell)}$, with $\alpha_0=1$. Under the
assumption $\frac{1}{m_0}\|\inputvec_i\|^2=1$ and the normalization \eqref{eq:chi_def}, we have
$K_{\infty,ii}^{(\ell)}=\widehat{K}_{\infty,ii}^{(\ell)}=1$ for all $\ell$. Setting $j=i$ in the cross-kernel recursion
\eqref{eq:Crec_pointwise} therefore gives, for $\ell\ge 1$,
\begin{align}
\alpha_\ell
=
\mathcal{T}_{\act}\!\Biggl(
\begin{bmatrix}
	1 & \sqrt{\tau_q}\,\alpha_{\ell-1}\\
	\sqrt{\tau_q}\,\alpha_{\ell-1} & 1
\end{bmatrix}
\Biggr)
=
\E\big[\act(U)\act(V)\big],
\end{align}
where $(U,V)\sim\mathcal{N}(0,\Sigma)$ has unit variances and correlation
$\rho=\Sigma_{12}=\sqrt{\tau_q}\,\alpha_{\ell-1}$.

To evaluate $\E[\act(U)\act(V)]$, note that
\begin{align}
\E[\act(U)\act(V)]
=
\P(UV\ge 0)-\P(UV<0)
=
2\,\P(UV\ge 0)-1.
\end{align}
By symmetry,
\begin{align}
\P(UV\ge 0)=\P(U>0,V>0)+\P(U<0,V<0)=2\,\P(U>0,V>0).
\end{align}
A standard bivariate-Gaussian quadrant computation yields
\begin{align}
\P(U>0,V>0)=\frac{1}{4}+\frac{1}{2\pi}\arcsin(\rho),
\end{align}
hence $\P(UV\ge 0)=\frac{1}{2}+\frac{1}{\pi}\arcsin(\rho)$ and therefore
\begin{align}
\E[\act(U)\act(V)]
=
2\Bigl(\frac{1}{2}+\frac{1}{\pi}\arcsin(\rho)\Bigr)-1
=
\frac{2}{\pi}\arcsin(\rho).
\end{align}
Substituting $\rho=\sqrt{\tau_q}\,\alpha_{\ell-1}$ gives
\begin{align}
\alpha_\ell=\frac{2}{\pi}\arcsin\!\bigl(\sqrt{\tau_q}\,\alpha_{\ell-1}\bigr), \qquad \ell\ge 1,
\end{align}
which is the claimed recursion. In particular, the corresponding asymptotic MSE is $2(1-\alpha_L)$ by
\eqref{eq:mse_general_deep}. \qed


\subsection{Proof of Example~\ref{cor:deep_nrelu_mse}}
\label{app:proof_deep_nrelu_mse}

The closed-form kernel for ReLU applied to a correlated Gaussian pair is classical, see, e.g., \cite{cho2009kernel}
and GP-limit expositions such as \cite{lee2018deep,matthews2018gaussian}. We include the short calculation for
completeness.

Fix an input index $i$ and define $\alpha_\ell\triangleq C_{\infty,ii}^{(\ell)}$, with $\alpha_0=1$. Under the
assumption $\frac{1}{m_0}\|\inputvec_i\|^2=1$ and the normalization \eqref{eq:chi_def}, we have
$K_{\infty,ii}^{(\ell)}=\widehat{K}_{\infty,ii}^{(\ell)}=1$ for all $\ell$. Setting $j=i$ in \eqref{eq:Crec_pointwise}
therefore yields, for $\ell\ge 1$,
\begin{align}
\alpha_\ell
=
\mathcal{T}_{\phi}\!\Biggl(
\begin{bmatrix}
	1 & \sqrt{\tau_q}\,\alpha_{\ell-1}\\
	\sqrt{\tau_q}\,\alpha_{\ell-1} & 1
\end{bmatrix}
\Biggr)
=
\E\big[\phi(U)\phi(V)\big],
\end{align}
where $(U,V)$ is centered bivariate normal with unit variances and correlation
$\rho=\sqrt{\tau_q}\,\alpha_{\ell-1}$.

Now take $\phi(t)=\sqrt{2}\max\{t,0\}$. For such $(U,V)$, the classical ReLU kernel formula gives
\begin{align}
\E\big[\max\{U,0\}\max\{V,0\}\big]
=
\frac{1}{2\pi}\Bigl(\sqrt{1-\rho^2}+\rho(\pi-\arccos\rho)\Bigr),
\qquad \rho\in[-1,1].
\end{align}
Multiplying by $\sqrt{2}$ on each coordinate, we obtain
\begin{align}
\E\big[\phi(U)\phi(V)\big]
=
2\,\E\big[\max\{U,0\}\max\{V,0\}\big]
=
\frac{1}{\pi}\Bigl(\sqrt{1-\rho^2}+\rho(\pi-\arccos\rho)\Bigr)
=
\kappa_{\rm R}(\rho).
\end{align}
Substituting $\rho=\sqrt{\tau_q}\,\alpha_{\ell-1}$ yields
\begin{align}
\alpha_\ell=\kappa_{\rm R}\!\bigl(\sqrt{\tau_q}\,\alpha_{\ell-1}\bigr), \qquad \ell\ge 1,
\end{align}
which is the claimed recursion. In particular, the corresponding asymptotic MSE is $2(1-\alpha_L)$ by
\eqref{eq:mse_general_deep}. \qed


\subsection{Non-renormalized pruning recursions}
\label{app:nonrenorm_static_networks}

The renormalization in Section~\ref{sec:prunerenormmodel} is not required for the Gaussian-limit analysis. Its role is to remove the scalar attenuation caused by pruning. This appendix records the corresponding recursion when the masked weights $\prunedW^{(\ell)}$ are used directly, without the factor $1/\sqrt{\tau_q}$.

Let $\widetilde{\avec}^{(\ell)}(\inputvec)$ and $\widetilde{f}(\inputvec)$ denote the features and output of the non-renormalized pruned network. Define
\begin{align}
\widetilde{K}^{(\ell)}_{ij}\triangleq m_\ell^{-1}
\langle \widetilde{\avec}^{(\ell)}(\inputvec_i),\widetilde{\avec}^{(\ell)}(\inputvec_j)\rangle,
\qquad
\widetilde{C}^{(\ell)}_{ij}\triangleq m_\ell^{-1}
\langle \avec^{(\ell)}(\inputvec_i),\widetilde{\avec}^{(\ell)}(\inputvec_j)\rangle.
\end{align}
The base case is
\begin{align}
K^{(0)}_{\infty,ij}=\widetilde{K}^{(0)}_{\infty,ij}=\widetilde{C}^{(0)}_{\infty,ij}
=\frac{1}{m_0}\langle \inputvec_i,\inputvec_j\rangle.
\end{align}
For each pruned row, Corollary~\ref{corol:concentration} gives two limits: its squared norm is asymptotically reduced by $\tau_q$, and its unnormalized overlap with the original row is also $\tau_q$. Therefore the same proof as Appendix~\ref{app:proof_deep_general_phi} yields, for $\ell\ge 1$,
\begin{align}
\widetilde{K}^{(\ell)}_{\infty,ij}
&=
\mathcal{T}_\phi\!\Biggl(
\begin{bmatrix}
\tau_q\widetilde{K}^{(\ell-1)}_{\infty,ii} & \tau_q\widetilde{K}^{(\ell-1)}_{\infty,ij}\\
\tau_q\widetilde{K}^{(\ell-1)}_{\infty,ji} & \tau_q\widetilde{K}^{(\ell-1)}_{\infty,jj}
\end{bmatrix}
\Biggr),
\label{eq:nonrn_K_rec}
\\
\widetilde{C}^{(\ell)}_{\infty,ij}
&=
\mathcal{T}_\phi\!\Biggl(
\begin{bmatrix}
K^{(\ell-1)}_{\infty,ii} & \tau_q\widetilde{C}^{(\ell-1)}_{\infty,ij}\\
\tau_q\widetilde{C}^{(\ell-1)}_{\infty,ji} & \tau_q\widetilde{K}^{(\ell-1)}_{\infty,jj}
\end{bmatrix}
\Biggr).
\label{eq:nonrn_C_rec}
\end{align}
The output covariances satisfy
\begin{align}
\mathrm{Cov}(f(\inputvec_i),f(\inputvec_j))\to K^{(L)}_{\infty,ij},\quad
\mathrm{Cov}(\widetilde{f}(\inputvec_i),\widetilde{f}(\inputvec_j))\to \widetilde{K}^{(L)}_{\infty,ij},\quad
\mathrm{Cov}(f(\inputvec_i),\widetilde{f}(\inputvec_j))\to \widetilde{C}^{(L)}_{\infty,ij},
\end{align}
and hence
\begin{equation}
\E\big[(f(\inputvec_i)-\widetilde{f}(\inputvec_i))^2\big]
\to K^{(L)}_{\infty,ii}+\widetilde{K}^{(L)}_{\infty,ii}-2\widetilde{C}^{(L)}_{\infty,ii}.
\label{eq:nonrn_mse}
\end{equation}
Comparing \eqref{eq:nonrn_K_rec}--\eqref{eq:nonrn_C_rec} with \eqref{eq:Krec_pointwise}--\eqref{eq:Crec_pointwise} shows what renormalization removes. With renormalization, the pruned self-kernel is kept on the same scale as the original self-kernel and the row overlap coefficient becomes $\sqrt{\tau_q}$. Without renormalization, the pruned self-kernel is attenuated by $\tau_q$ at each layer and the unnormalized row overlap coefficient is $\tau_q$.

\begin{example}[Sign activation without renormalization]
For $\phi(t)=\act(t)$ and unit-norm inputs, $\widetilde{K}^{(\ell)}_{\infty,ii}=1$ for all $\ell\ge 1$, because the sign nonlinearity is invariant to positive rescaling. If $\widetilde{\alpha}_\ell\triangleq \widetilde{C}^{(\ell)}_{\infty,ii}$, then \eqref{eq:nonrn_C_rec} gives
\begin{align}
\widetilde{\alpha}_\ell
=
\mathcal{T}_{\act}\!\left(
\begin{bmatrix}
1 & \tau_q\widetilde{\alpha}_{\ell-1}\\
\tau_q\widetilde{\alpha}_{\ell-1} & \tau_q
\end{bmatrix}
\right)
=\frac{2}{\pi}\arcsin\!\bigl(\sqrt{\tau_q}\,\widetilde{\alpha}_{\ell-1}\bigr).
\end{align}
Thus the covariance matrix contains $\tau_q$ in the non-renormalized cross entry, but after converting to the Gaussian correlation used by the sign kernel, the one-dimensional sign recursion is the same as in Example~\ref{cor:deep_sign_mse}. This is consistent with the scale invariance of the sign activation.
\end{example}

\begin{example}[Normalized ReLU without renormalization]
	\label{ex:deep_nrelu_no_renorm_mse}
	Let $\phi(t)=\sqrt{2}\max\{t,0\}$ be the normalized ReLU, and assume unit-norm inputs so that
	$K_{\infty,ii}^{(\ell)}=1$ for the unpruned network. If pruning is applied without renormalization, then the pruned self-kernel satisfies
	\begin{align}
	\widetilde K_{\infty,ii}^{(\ell)}=\tau_q^\ell .
	\end{align}
	Define the normalized cross-alignment
	\begin{align}
	\beta_\ell
	\triangleq
	\frac{\widetilde C_{\infty,ii}^{(\ell)}}
	{\sqrt{K_{\infty,ii}^{(\ell)}\widetilde K_{\infty,ii}^{(\ell)}}}
	=
	\frac{\widetilde C_{\infty,ii}^{(\ell)}}{\tau_q^{\ell/2}},
	\qquad \beta_0=1 .
	\end{align}
	Then, for $\ell\geq 1$,
	\begin{align}
	\beta_\ell
	=
	\kappa_{\rm R}\!\left(\sqrt{\tau_q}\,\beta_{\ell-1}\right),
	\end{align}
	where
	\begin{align}
	\kappa_{\rm R}(\rho)
	\triangleq
	\frac{1}{\pi}
	\left(
	\sqrt{1-\rho^2}
	+\rho(\pi-\arccos\rho)
	\right),
	\qquad \rho\in[-1,1].
	\end{align}
	Thus, after normalizing for output scale, the alignment recursion is the same as in the renormalized model. The difference is that, without renormalization, the pruned output variance is attenuated by $\tau_q^L$. Therefore,
	\begin{equation}
		\mathrm{MSE}_{\rm no\;renorm}(q)
		=
		1+\tau_q^L
		-
		2\tau_q^{L/2}\beta_L(q).
		\label{eq:no_renorm_relu_mse}
	\end{equation}
	In contrast, the renormalized model has unit output variance on both sides and gives the alignment-only distortion
	\begin{align}
	\mathrm{MSE}_{\rm renorm}(q)
	=
	2(1-\beta_L(q)).
	\end{align}
	Hence renormalization removes the scale attenuation term and leaves only the loss of alignment caused by pruning.
\end{example}

%% file: tex/app/thres-calib-app.tex

\section{Threshold calibration under a compute budget}
\label{app:threshold_calibration}

Early-exit policies are specified by an exit location $\ell$ and a confidence threshold $\tau$. Let
$s_\ell(\inputvec)\in\mathbb{R}$ denote the confidence statistic used at exit $\ell$, and trigger an early exit when
$s_\ell(\inputvec)\ge \tau$. Let the normalized full-compute cost be $1$, and let $\mu_\ell\in(0,1)$ denote the
normalized cost incurred when the model exits at $\ell$. Under this cost model, the normalized average compute is
\begin{equation}
	\label{eq:mu_compute_generic}
	\overline{\mu}_{\mathrm{c}}^{(\ell)}(\tau)
	= \mu_\ell\,\P\{s_\ell(\inputvec)\ge \tau\} + 1\cdot \P\{s_\ell(\inputvec)< \tau\}
	= 1-(1-\mu_\ell)\,\P\{s_\ell(\inputvec)\ge \tau\}.
\end{equation}
Given a target compute budget $\overline{\mu}_{\mathrm{c}}\in(0,1]$, we choose $\tau$ by validation-set calibration.
Specifically, define the required early-exit frequency
\begin{align}
p_{\rm exit}\;\triangleq\;\frac{1-\overline{\mu}_{\mathrm{c}}}{1-\mu_\ell}\in[0,1].
\end{align}
Since $\tau\mapsto \P\{s_\ell(\inputvec)\ge \tau\}$ is nonincreasing, the compute constraint
$\overline{\mu}_{\mathrm{c}}^{(\ell)}(\tau)\le \overline{\mu}_{\mathrm{c}}$ is equivalent to
$\P\{s_\ell(\inputvec)\ge \tau\}\ge p_{\rm exit}$. Among all such thresholds, a natural compute-matching choice is the
largest $\tau$ that still satisfies the constraint, which sets the constraint at equality in the population.
In practice, we estimate $\P\{s_\ell(\inputvec)\ge \tau\}$ on a held-out validation set by sorting the validation scores
$\{s_\ell(\inputvec_i)\}_{i=1}^n$ and selecting $\tau$ as the empirical quantile that achieves exit frequency
approximately $p_{\rm exit}$. Repeating this calibration for each candidate $\ell$ yields a family of compute-matched
policies, and $\ell$ is then selected by validation error.

In the special case of the conditional perceptron, the confidence statistic is $s(\inputvec)=|v_1|$ and the normalized
average compute takes the explicit form
\begin{equation}
	\label{eq:mu_compute_perceptron}
	\overline{\mu}_{\mathrm{c}}(q,\tau)=1-q\,\mathrm{erfc}\!\Bigl(\frac{\tau}{\sqrt{2}}\Bigr),
\end{equation}
where $q\in(0,1]$ is the sparsity rate and $\mathrm{erfc}(\cdot)$ is the complementary error function. Writing the compute
constraint as $\overline{\mu}_{\mathrm{c}}\le 1-\epsilon$ (so that $\epsilon$ is the compute gap), the feasible set is
\begin{align}
q\,\mathrm{erfc}\!\Bigl(\frac{\tau}{\sqrt{2}}\Bigr)\ge \epsilon,
\end{align}
which is nonempty only when $q\in[\epsilon,1]$. The compute-matching threshold is obtained by enforcing equality, yielding
\begin{equation}
	\label{eq:tau_cal_perceptron}
	\tau_{\mathrm{cal}}(q,\epsilon)
	= \sqrt{2}\,\mathrm{erfc}^{-1}\!\Bigl(\frac{\epsilon}{q}\Bigr),\qquad q\in[\epsilon,1].
\end{equation}
This mirrors the empirical quantile calibration rule above, but here the calibration map admits an analytic inverse.

\begin{figure}[H]
	\centering
	\scalebox{1}{\includegraphics{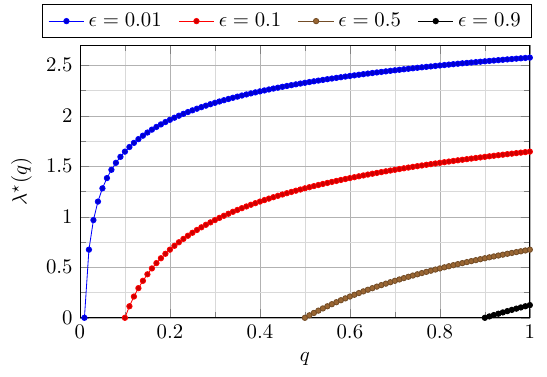}}
	\caption{Compute-matching thresholds $\tau_{\mathrm{cal}}(q,\epsilon)$ under different sparsities and compute gaps.}\vspace{-10pt}
	\label{fig:optt}
\end{figure}

Fig.~\ref{fig:optt} plots $\tau_{\mathrm{cal}}(q,\epsilon)$ for several values of $\epsilon$. The map is increasing in $q$.
Intuitively, when pruning is more aggressive (larger $q$), meeting a fixed compute gap $\epsilon$ requires a smaller early
exit rate $\epsilon/q$, which is achieved by a higher threshold. Conversely, when $q$ is small, meeting the same compute
gap requires more frequent early exits, which forces a lower threshold. The regime $q<\epsilon$ is infeasible, since even
exiting on every input cannot achieve compute gap $\epsilon$.

If one additionally optimizes the sparsity level under the same compute budget, the closed form
\eqref{eq:tau_cal_perceptron} reduces the tuning from a two-dimensional search over $(q,\tau)\in[\epsilon,1]\times[0,\infty)$
to a one-dimensional search over $q\in[\epsilon,1]$ with $\tau$ set to $\tau_{\mathrm{cal}}(q,\epsilon)$. In our numerical
experiments, the resulting univariate objective $q\mapsto \epsilon_{\mathrm{c}}(q,\tau_{\mathrm{cal}}(q,\epsilon))$ is consistently unimodal, so we locate its minimizer efficiently with a simple binary search over $q$.

%% file: tex/app/adaptive-net-app.tex
\section{Proofs of Results on Adaptive Networks}
In this section, we present the proofs of our main results on adaptive networks.

\subsection{Proof of Theorem \ref{thm:deep_tradeoff_explicit}}
\label{proofof:thm:deep_tradeoff_explicit}

We write the compute constraint as $\overline{\mu}_{\mathrm{c}}=1-\varepsilon$ for some
$\varepsilon\in(0,1-\kappa_\ell)$. Under the hypothesis $\overline{\mu}_{\mathrm{c}}\le 1-\varepsilon$, it suffices to
exhibit a threshold $\netexitthres$ whose expected compute satisfies $\mu_{\mathrm{c}}'(\netexitthres)\le \overline{\mu}_{\mathrm{c}}$
and whose resulting error satisfies the stated bound.

Let $\widehat{Y}_{\mathrm{uc}}(\inputvec)\triangleq \act(f_L(\inputvec))$ denote the full-compute predictor, and let
$\widehat{Y}_{\mathrm{c}}(\inputvec)$ be the early-exit predictor with threshold $\netexitthres$.
If $\widehat{Y}_{\mathrm{c}}(\inputvec)\neq Y$, then either $\widehat{Y}_{\mathrm{uc}}(\inputvec)\neq Y$, or
$\widehat{Y}_{\mathrm{uc}}(\inputvec)=Y$ and $\widehat{Y}_{\mathrm{c}}(\inputvec)\neq \widehat{Y}_{\mathrm{uc}}(\inputvec)$.
The latter event can only occur when an early exit is taken, namely when $|f_\ell(\inputvec)|\ge \netexitthres$, and when
the exit and final predictions disagree, namely $\act(f_\ell(\inputvec))\neq \act(f_L(\inputvec))$. Taking probabilities yields
\begin{align}
	\epsilon_{\mathrm{c}}
	\le
	\epsilon_{\mathrm{uc}}
	+
	\P_{\inputvec}\!\left(\act(f_{\ell}(\inputvec))\neq \act(f_L(\inputvec)),\ |f_{\ell}(\inputvec)|\ge \netexitthres\right).
	\label{eq:deep_risk_decomp_inproof}
\end{align}

We bound the disagreement term using Assumption~\ref{assump:subg_coupling}. Fix a class $y\in\{-1,+1\}$ and let
$U\triangleq \tilde f_{\ell,y}(X)$ and $V\triangleq \tilde f_{L,y}(X)$ denote the corresponding class-standardized scores.
Since $\act(\cdot)$ depends only on the sign, centering and rescaling within a class does not affect the disagreement event,
and it suffices to control
\begin{align}
\P\big(\act(U)\neq \act(V),\,|U|\ge t \mid Y=y\big)
\end{align}
for an appropriate threshold $t$ corresponding to $\netexitthres$. By symmetry,
\begin{align}
\P\big(\act(U)\neq \act(V),\,|U|\ge t \mid Y=y\big)
=
2\,\P\big(U\ge t,\,V\le 0\mid Y=y\big).
\end{align}
For any $s\ge t$, Markov's inequality and \eqref{eq:subg_mgf_R} give, with $R\triangleq V-\rho_\ell U$,
\begin{align}
	\P\big(V\le 0\mid U=s,Y=y\big)
	&=
	\P\big(R\le -\rho_\ell s\mid U=s,Y=y\big)\\
	&\le
	\exp(-\lambda\rho_\ell s)\,
	\E\!\left[\exp(\lambda R)\mid U=s,Y=y\right]\\
	&\le
	\exp\!\left(-\lambda\rho_\ell s+\frac{\lambda^2(1-\rho_\ell^2)}{2}\right)
	\qquad \forall \lambda\ge 0.
\end{align}
Optimizing over $\lambda$ yields $\lambda^\star=\rho_\ell s/(1-\rho_\ell^2)$ and therefore
\begin{align}
	\P\big(V\le 0\mid U=s,Y=y\big)
	\le
	\exp\!\left(-\frac{\rho_\ell^2}{2(1-\rho_\ell^2)}\,s^2\right),
	\qquad s\ge t.
	\label{eq:cond_tail_V_given_U_subg}
\end{align}
Consequently,
\begin{align}
	\P\big(U\ge t,\,V\le 0\mid Y=y\big)
	&\le
	\int_{t}^{\infty}\P\big(V\le 0\mid U=s,Y=y\big)\,d\P(U\le s\mid Y=y)\nonumber\\
	&\le
	\exp\!\left(-\frac{\rho_\ell^2}{2(1-\rho_\ell^2)}\,t^2\right)\,
	\P\big(U\ge t\mid Y=y\big).
	\label{eq:joint_event_bound_subg}
\end{align}
Let $r(t)\triangleq \P(|U|\ge t\mid Y=y)$ denote the within-class exit rate. Using
$\P(U\ge t\mid Y=y)\le r(t)$ and combining \eqref{eq:joint_event_bound_subg} with symmetry gives
\begin{align}
	\P\big(\act(U)\neq \act(V),\,|U|\ge t\mid Y=y\big)
	\le
	r(t)\,\exp\!\left(-\frac{\rho_\ell^2}{2(1-\rho_\ell^2)}\,t^2\right).
	\label{eq:disagreement_bound_subg}
\end{align}
Averaging over classes and translating back from standardized scores to the original threshold $\netexitthres$
yields
\begin{align}
	\epsilon_{\mathrm{c}}-\epsilon_{\mathrm{uc}}
	\le
	r(\netexitthres)\,
	\exp\!\left(-\frac{\rho_\ell^2}{2(1-\rho_\ell^2)}\,\netexitthres^2\right),
	\label{eq:excess_tradeoff_inproof_subg}
\end{align}
where $r(\netexitthres)\triangleq \P_{\inputvec}(|f_\ell(\inputvec)|\ge \netexitthres)$ denotes the early-exit rate.

We now choose $\netexitthres$ to meet the compute constraint. Since
$\mu_{\mathrm{c}}'(\netexitthres)=1-(1-\kappa_\ell)r(\netexitthres)$, the requirement
$\mu_{\mathrm{c}}'(\netexitthres)\le 1-\varepsilon$ is equivalent to $r(\netexitthres)\ge r_\varepsilon$, where
$r_\varepsilon\triangleq \varepsilon/(1-\kappa_\ell)$. Under Assumption~\ref{assump:subg_coupling}(i), the standardized
exit statistic is sub-Gaussian, hence for all $t\ge 0$,
\begin{align}
	r(t)=\P(|U|\ge t\mid Y=y)\le 2e^{-t^2/2}.
	\label{eq:r_tail_subg}
\end{align}
Choose $t_\varepsilon\triangleq \sqrt{2\log(2/r_\varepsilon)}$, so that $2e^{-t_\varepsilon^2/2}=r_\varepsilon$.
Then \eqref{eq:r_tail_subg} ensures $r(t_\varepsilon)\le r_\varepsilon$, and therefore selecting any threshold
$\netexitthres_\varepsilon$ such that $r(\netexitthres_\varepsilon)=r_\varepsilon$ yields
$\mu_{\mathrm{c}}'(\netexitthres_\varepsilon)=1-\varepsilon$. Substituting $t=t_\varepsilon$ into
\eqref{eq:excess_tradeoff_inproof_subg} and using $r(\netexitthres_\varepsilon)=r_\varepsilon$ gives
\begin{align}
	\epsilon_{\mathrm{c}}-\epsilon_{\mathrm{uc}}
	\le
	r_\varepsilon\,
	\exp\!\left(-\frac{\rho_\ell^2}{2(1-\rho_\ell^2)}\,t_\varepsilon^2\right)
	=
	r_\varepsilon\left(e^{-t_\varepsilon^2/2}\right)^{\frac{\rho_\ell^2}{1-\rho_\ell^2}}
	=
	r_\varepsilon\left(\frac{r_\varepsilon}{2}\right)^{\frac{\rho_\ell^2}{1-\rho_\ell^2}}.
	\label{eq:excess_subg_mid}
\end{align}
Let $A'\triangleq \frac{1}{2(1-\rho_\ell^2)}$. Since $\frac{\rho_\ell^2}{1-\rho_\ell^2}=2A'-1$, we obtain
\begin{align}
	\epsilon_{\mathrm{c}}-\epsilon_{\mathrm{uc}}
	\le
	2^{1-2A'}\,r_\varepsilon^{2A'}.
	\label{eq:excess_subg_power}
\end{align}
Using $r_\varepsilon\le 1$ and $2^{1-2A'}\le (8\pi)^{A'}$ (since $8\pi>1$ and $A'\ge 1/2$) yields the simpler bound
$\epsilon_{\mathrm{c}}-\epsilon_{\mathrm{uc}}
\le
(8\pi)^{A'}\,r_\varepsilon^{A'}
$. Finally, substituting $r_\varepsilon=\varepsilon/(1-\kappa_\ell)$ gives
\begin{align}
\epsilon_{\mathrm{c}}
\le
\epsilon_{\mathrm{uc}}
+
\left(\frac{8\pi\,\varepsilon}{1-\kappa_\ell}\right)^{A'},
\end{align}
which proves the claim.

\subsection{Proof of Proposition~\ref{prop:rho_gap_law_quenched}}
\label{proofof:prop:rho_gap_law_quenched}

Fix a class $y\in\{-1,+1\}$ and a layer $\ell\in\{1,\ldots,L-1\}$. Write $W=W^{(m)}$ for brevity. Throughout, all
expectations, variances, and correlations are with respect to $\P_{\inputvec}$ conditional on the realized backbone $W$
and on the event $\{Y=y\}$.

Proposition~\ref{prop:rho_gap_law_quenched} concerns the within-class correlation of the class-standardized scores, so
additive constants play no role once we subtract within-class means. Since the distillation definition
\eqref{eq:distilldef} is an affine (intercept) least-squares regression, we may equivalently work with centered variables.
Concretely, consider the within-class least-squares problem with an intercept
\begin{align}
(\boldsymbol{w}_{\ell,y}^{\star},b_{\ell,y}^{\star})
\in
\arg\min_{\boldsymbol{w},\,b}\;
\E_{\inputvec}\!\left[\bigl(f_L^{(m)}(\inputvec)-\boldsymbol{w}^{\top}\avec^{(\ell)}(\inputvec)-b\bigr)^2\mid Y=y\right].
\end{align}
Define the within-class centered final score and centered features
\begin{align}
g_{L,y}^{(m)}(\inputvec)\triangleq f_L^{(m)}(\inputvec)-\mu_{L,y}^{(m)},
\qquad
\mu_{L,y}^{(m)}\triangleq \E_{\inputvec}[f_L^{(m)}(\inputvec)\mid Y=y],
\end{align}
\begin{align}
\bar{\avec}^{(\ell)}(\inputvec)\triangleq \avec^{(\ell)}(\inputvec)-\E_{\inputvec}[\avec^{(\ell)}(\inputvec)\mid Y=y].
\end{align}
Optimality in $b$ gives
\begin{align}
b_{\ell,y}^{\star}=\mu_{L,y}^{(m)}-(\boldsymbol{w}_{\ell,y}^{\star})^{\top}\E_{\inputvec}[\avec^{(\ell)}(\inputvec)\mid Y=y],
\end{align}
and therefore the centered fitted score satisfies
\begin{align}
\Big((\boldsymbol{w}_{\ell,y}^{\star})^{\top}\avec^{(\ell)}(\inputvec)+b_{\ell,y}^{\star}\Big)
-
\E_{\inputvec}\!\Big[(\boldsymbol{w}_{\ell,y}^{\star})^{\top}\avec^{(\ell)}(\inputvec)+b_{\ell,y}^{\star}\mid Y=y\Big]
=
(\boldsymbol{w}_{\ell,y}^{\star})^{\top}\bar{\avec}^{(\ell)}(\inputvec).
\end{align}
Hence the within-class correlation between the standardized exit and final scores is the same as the correlation between
the standardized versions of $(\boldsymbol{w}_{\ell,y}^{\star})^{\top}\bar{\avec}^{(\ell)}(\inputvec)$ and
$g_{L,y}^{(m)}(\inputvec)$. We will work with this centered representation.

Let $\mathcal{S}_\ell$ denote the closed linear span in $L^2(\P_{\inputvec}\mid Y=y)$ of
$\{\boldsymbol{w}^{\top}\bar{\avec}^{(\ell)}(\cdot):\boldsymbol{w}\in\mathbb{R}^{m}\}$. Since
$(\boldsymbol{w}_{\ell,y}^{\star})^{\top}\bar{\avec}^{(\ell)}$ is the within-class least-squares predictor of
$g_{L,y}^{(m)}$ from $\bar{\avec}^{(\ell)}$, it equals the $L^2(\P_{\inputvec}\mid Y=y)$ orthogonal projection of
$g_{L,y}^{(m)}$ onto $\mathcal{S}_\ell$, that is,
\begin{align}
(\boldsymbol{w}_{\ell,y}^{\star})^{\top}\bar{\avec}^{(\ell)}(\cdot)=\mathrm{Proj}_{\mathcal{S}_\ell} g_{L,y}^{(m)}(\cdot).
\end{align}
For an orthogonal projection, $\Cov(\mathrm{Proj}_{\mathcal{S}_\ell} g,\,g)=\Var(\mathrm{Proj}_{\mathcal{S}_\ell} g)$, and hence
\begin{align}
	\big(\widetilde{\rho}_{\ell,y}^{(m)}\big)^2
	=
	\frac{\Var_{\inputvec}\!\bigl(\mathrm{Proj}_{\mathcal{S}_\ell} g_{L,y}^{(m)}(\inputvec)\mid Y=y\bigr)}
	{\Var_{\inputvec}\!\bigl(g_{L,y}^{(m)}(\inputvec)\mid Y=y\bigr)}.
	\label{eq:rho_as_R2_class_app}
\end{align}
Thus it suffices to identify the asymptotic fraction of the within-class variance of $g_{L,y}^{(m)}$ captured by the
projection onto $\mathcal{S}_\ell$.

Let $Z\sim\mathcal{N}(0,1)$ and recall $\chi_\phi=\E[Z\phi(Z)]$. Under the activation normalization in \eqref{eq:chi_def},
the $L^2(\mathcal{N}(0,1))$ orthogonal decomposition
\begin{align}
	\phi(Z)=\chi_\phi Z+\xi,
	\qquad
	\E[Z\xi]=0,
	\qquad
	\E[\xi^2]=1-\chi_\phi^2
	\label{eq:phi_decomp_formal_class}
\end{align}
shows that the degree-one component retains a fraction $\chi_\phi^2$ of the variance.

We now relate the final score to layer-$\ell$ features in the wide random-backbone regime. Standard infinite-width results
for fully-connected Gaussian networks imply that for fixed $L$ and $\ell$, the relevant empirical second moments
concentrate around deterministic limits governed by the NNGP recursion, see for example
\cite{lee2018deep,matthews2018gaussian,yang2019tensor}. In particular, under the $1/\sqrt{m}$ scaling and the
normalization in \eqref{eq:chi_def}, the layerwise variances remain at their fixed point, so the within-class variance of
the centered final score satisfies
\begin{align}
	\Var_{\inputvec}\!\bigl(g_{L,y}^{(m)}(\inputvec)\mid Y=y\bigr)\ \xrightarrow[]{\ \mathrm{in\ prob.}\ }\ 1.
	\label{eq:var_final_to_1_step3}
\end{align}

Next, apply the decomposition \eqref{eq:phi_decomp_formal_class} to each of the $L-\ell$ random layers between $\ell$ and
$L$. The degree-one component propagates linearly through each random layer and, at each step, retains a multiplicative
factor $\chi_\phi$ in $L^2$ when projecting onto the span of the previous-layer preactivations. Iterating over the depth
gap yields an orthogonal decomposition in $L^2(\P_{\inputvec}\mid Y=y)$ of the centered final score:
\begin{align}
	g_{L,y}^{(m)}(\inputvec)
	=
	\chi_\phi^{\,L-\ell}\,U_{\ell,y}^{(m)}(\inputvec)+V_{\ell,y}^{(m)}(\inputvec),
	\label{eq:score_decomp_gap_formal_class}
\end{align}
where $U_{\ell,y}^{(m)}(\cdot)\in\mathcal{S}_\ell$ and $V_{\ell,y}^{(m)}(\cdot)\perp\mathcal{S}_\ell$ in
$L^2(\P_{\inputvec}\mid Y=y)$. Moreover, moment concentration and the variance fixed-point property imply that
\begin{align}
	\Var_{\inputvec}\!\bigl(U_{\ell,y}^{(m)}(\inputvec)\mid Y=y\bigr)\ \xrightarrow[]{\ \mathrm{in\ prob.}\ }\ 1,
	\qquad
	\Var_{\inputvec}\!\bigl(V_{\ell,y}^{(m)}(\inputvec)\mid Y=y\bigr)\ \xrightarrow[]{\ \mathrm{in\ prob.}\ }\ 1-\chi_\phi^{2(L-\ell)}.
	\label{eq:UV_var_limits_step3}
\end{align}
Since $U_{\ell,y}^{(m)}\in\mathcal{S}_\ell$ and $V_{\ell,y}^{(m)}\perp\mathcal{S}_\ell$, we have
$\mathrm{Proj}_{\mathcal{S}_\ell} g_{L,y}^{(m)}(\inputvec)
=
\chi_\phi^{\,L-\ell}\,U_{\ell,y}^{(m)}(\inputvec)$, 
and therefore, by \eqref{eq:UV_var_limits_step3},
\begin{align}
	\Var_{\inputvec}\!\bigl(\mathrm{Proj}_{\mathcal{S}_\ell} g_{L,y}^{(m)}(\inputvec)\mid Y=y\bigr)
	=
	\chi_\phi^{2(L-\ell)}\Var_{\inputvec}\!\bigl(U_{\ell,y}^{(m)}(\inputvec)\mid Y=y\bigr)
	\ \xrightarrow[]{\ \mathrm{in\ prob.}\ }\ 
	\chi_\phi^{2(L-\ell)}.
	\label{eq:var_proj_limit_step3}
\end{align}

Combining \eqref{eq:rho_as_R2_class_app}, \eqref{eq:var_final_to_1_step3}, and \eqref{eq:var_proj_limit_step3} yields $\big(\widetilde{\rho}_{\ell,y}^{(m)}\big)^2
\ \xrightarrow[]{\ \mathrm{in\ prob.}\ }\ 
\chi_\phi^{\,2(L-\ell)}$. If $\chi_\phi\ge 0$ (which holds for standard monotone activations under \eqref{eq:chi_def}), then taking square roots gives
$\widetilde{\rho}_{\ell,y}^{(m)}\to \chi_\phi^{\,L-\ell}$ in probability. The argument does not depend on $y$, completing the proof.

\subsection{Proof of Corollary \ref{cor:deep_tradeoff_gap}}
\label{proofof:cor:deep_tradeoff_gap}

We apply Theorem~\ref{thm:deep_tradeoff_explicit} conditional on the realized backbone $W^{(m)}$. The only quantity in the
bound that depends on $W^{(m)}$ is the exponent
\(
\deepnetexpo^{(m)} \triangleq 0.5/\bigl(1-(\widetilde{\rho}_{\ell}^{(m)})^2\bigr),
\)
where $\widetilde{\rho}_{\ell}^{(m)}$ denotes the (true) within-class correlation from Proposition~\ref{prop:rho_gap_law_quenched}.
By Proposition~\ref{prop:rho_gap_law_quenched}, for each $y$ we have
$\widetilde{\rho}_{\ell,y}^{(m)}\to \chi_\phi^{\,L-\ell}\triangleq \widetilde{\rho}_{\ell,\infty}$ in probability, and hence the
same limit holds for $\widetilde{\rho}_{\ell}^{(m)}$. By continuity of the map $u\mapsto 0.5/(1-u^2)$ on $(-1,1)$, it follows that
$\deepnetexpo^{(m)}\to \deepnetexpo_{\infty}$ in probability, where
$\deepnetexpo_{\infty}=0.5/(1-\widetilde{\rho}_{\ell,\infty}^2)$.
Substituting $\deepnetexpo^{(m)}=\deepnetexpo_{\infty}+o_{\P}(1)$ into the bound of Theorem~\ref{thm:deep_tradeoff_explicit} yields that,
with probability approaching $1$ over the draw of $W^{(m)}$,
\begin{align}
\epsilon_{\mathrm{c}}
\le
\epsilon_{\mathrm{uc}}
+
\left[\frac{8\pi\,\varepsilon}{1-\kappa_\ell}\right]^{\deepnetexpo_{\infty}+o(1)}.
\end{align}
Absorbing the $o(1)$ term into the exponent gives \eqref{eq:deep_tradeoff_gap}.

%% file: tex/app/simul-app.tex

\section{Additional Numerical Experiments}
\label{app:more_experiments}

This appendix collects numerical simulations that further validate the analytical results and provide additional
diagnostics.

\subsection{Pruning and concentration of measure}
\label{app:concentration_pruning}

We first verify the concentration behavior underlying Theorem~\ref{theorem:gamma} and
Corollary~\ref{corol:concentration}. We generate i.i.d.\ pruning masks and evaluate the corresponding scalar
random quantity $\mathbf{W}^T \mathbf{W}_{\mathrm{ee}}$ over many trials, for increasing feature dimension $d$ and fixed
sparsity rate $q$.

\begin{figure}[H]
	\centering
	\begin{subfigure}[H]{0.47\linewidth}
		\centering
		\includegraphics[width=\linewidth]{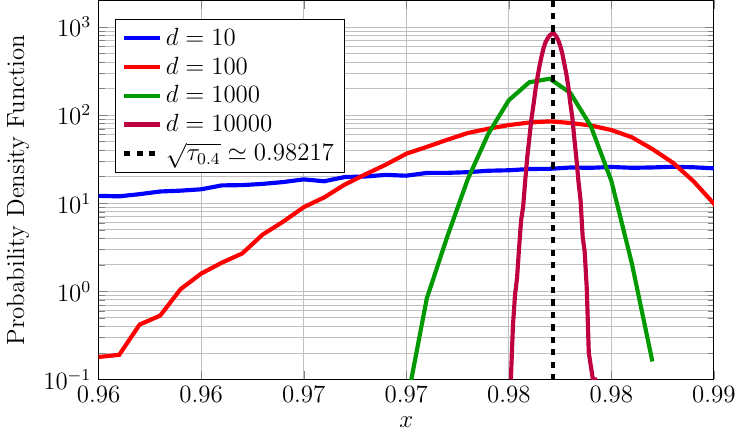}
		\caption{Simulated PDFs for $q=0.4$.}
		\label{fig:concent_q04}
	\end{subfigure}
	\hfill
	\begin{subfigure}[H]{0.47\linewidth}
		\centering
		\includegraphics[width=\linewidth]{software/concent-simul2.pdf}
		\caption{Simulated PDFs for $q=0.9$.}
		\label{fig:concent_q09}
	\end{subfigure}\\
	\vspace{10pt}
	\begin{subfigure}[H]{0.47\linewidth}
		\centering
		\includegraphics[width=\linewidth]{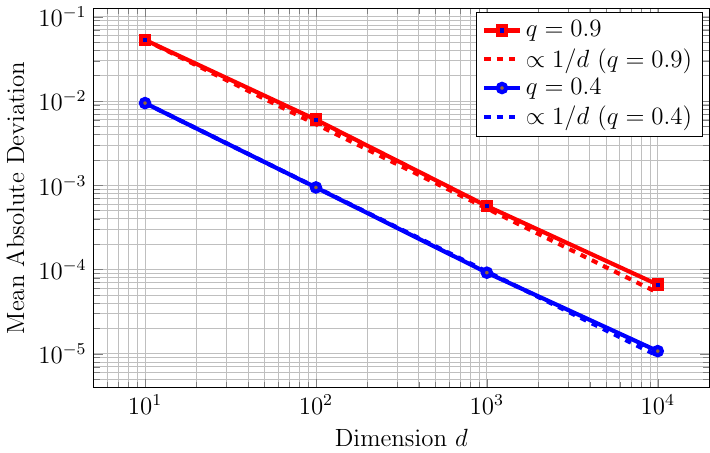}
		\caption{Mean absolute deviation from the limit.}
		\label{fig:concent_mad}
	\end{subfigure}
	\hfill
	\begin{subfigure}[H]{0.47\linewidth}
		\centering
		\includegraphics[width=\linewidth]{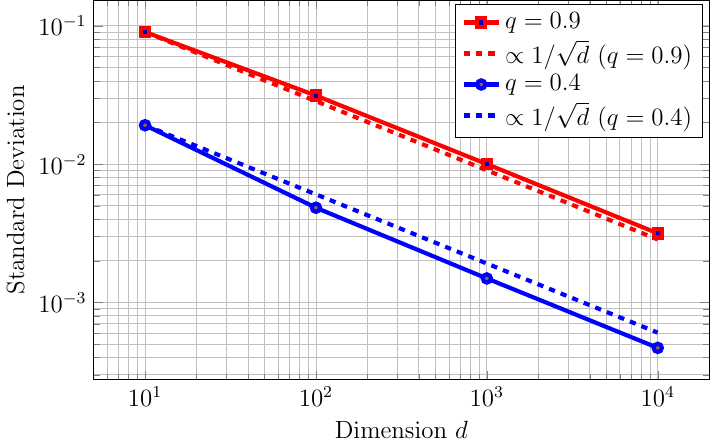}
		\caption{Standard deviation.}
		\label{fig:concent_sd}
	\end{subfigure}
	\caption{Pruning and concentration of measure at large feature dimension.}
	\label{fig:concentration_all}
\end{figure}

Figure~\ref{fig:concentration_all} extends the main-text concentration
experiment. Panels~\ref{fig:concent_q04} and~\ref{fig:concent_q09}
show the simulated PDFs of $\mathbf W^\top \mathbf W_{\mathrm{ee}}$
for $q=0.4$ and $q=0.9$, respectively, using $100{,}000$ trials and
$d\in\{10,100,1000,10000\}$. In both cases, the distributions
concentrate around the limits predicted by
Corollary~\ref{corol:concentration}, namely
$\sqrt{\tau_{0.4}}\simeq 0.98217$ and
$\sqrt{\tau_{0.9}}\simeq 0.66279$. The $q=0.9$ case is the one shown
in the main text, while the $q=0.4$ case illustrates the same phenomenon
under lighter pruning.

Panels~\ref{fig:concent_mad} and~\ref{fig:concent_sd} quantify the
finite-dimensional convergence. The mean absolute deviation from the
limit decays approximately as $1/d$, while the standard deviation decays
approximately as $1/\sqrt d$. These empirical rates are consistent with
the concentration behavior predicted by Theorem~\ref{theorem:gamma},
although a sharp proof of these observed rates is left for future work.

\subsection{Semi-structured pruning at the neuron level}
\label{app:semistructured_neuron}

\input{tex/app/semi-struct-app}

\subsection{Compute constrained tradeoffs for the conditional perceptron}
\label{app:perceptron_tradeoff}

We next report additional simulations for the conditional perceptron that complement the main paper results.


For a target normalized compute constraint $\overline{\mu}_{\mathrm{c}} \in (0,1)$, we select an operating point
$(q,\tau)$ by searching over $(q,\tau)$ pairs as described in Appendix \ref{app:threshold_calibration}. Concretely, we run a
fixed-depth binary search (10 iterations) that queries candidate points using a Monte Carlo estimate of the
generalization error. Each query uses 100 test samples per realization of the teacher and training vectors. We average
over enough independent realizations so that the standard error of the estimated generalization error is within 10\% of
its mean value. This procedure is designed to reliably locate a near-optimal operating point in the
compute--performance plane while keeping runtime modest.


\begin{figure}[H]
	\vspace{-4pt}
	\centering
	\begin{subfigure}[H]{0.48\linewidth}
		\centering
		\includegraphics[width=\linewidth]{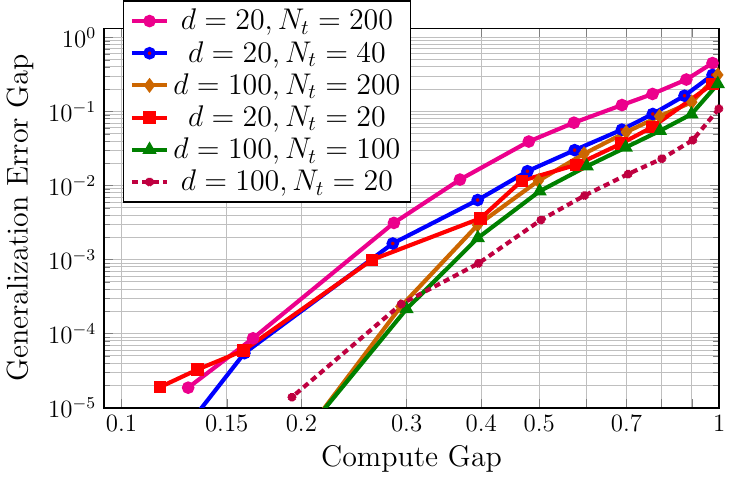}
		\caption{Excess error versus compute gap (log-log).}
		\label{fig:app_generrgaps}
	\end{subfigure}
	\hfill
	\begin{subfigure}[H]{0.48\linewidth}
		\centering
		\includegraphics[width=\linewidth]{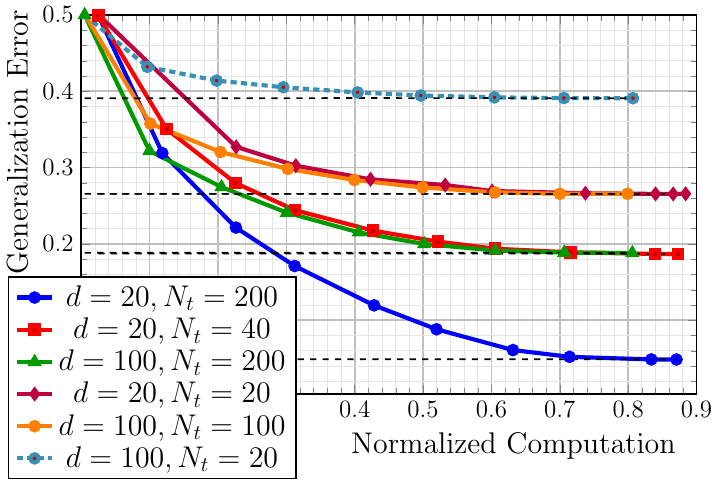}
		\caption{Generalization error versus expected compute.}
		\label{fig:app_errconv}
	\end{subfigure}
	\vspace{-3pt}
	\caption{Conditional perceptron compute--accuracy tradeoffs under compute-constrained operating points obtained by
		binary search.}
	\label{fig:app_perceptron_tradeoff}\vspace{-8pt}
\end{figure}

Figure~\ref{fig:app_generrgaps} plots the excess generalization error of the conditional predictor relative to the
unconditional predictor, versus the compute gap $1 - \text{(Normalized Compute)}$, on log-log axes. Theorem
\ref{theorem:tradeoff} predicts a power-law decay in the small-gap regime, which appears as an approximately linear trend
on these axes. This behavior is consistently observed across the tested configurations. Figure~\ref{fig:app_errconv}
replots the same sweeps as generalization error versus expected compute, with horizontal dashed lines indicating the
unconditional error rates. Configurations with matching ratios $N_t/d$ converge to essentially the same unconditional
baseline, and the conditional predictors approach these baselines quickly as compute increases.

\begin{figure}[H]
	\centering
	\begin{tikzpicture}
		\begin{axis}[
			width=12.5cm,
			height=8.5cm,
			xlabel={Compute Gap},
			ylabel={Generalization Error Gap},
			xlabel style={font=\large},
			ylabel style={font=\large},
			tick label style={font=\normalsize},
			grid=both,
			thick,
			xmode=log,
			ymode=log,
			ymin=1e-5,
			xmax=1,
			xmin=0.12,
			legend style={font=\normalsize, at={(0.97,0.03)}, anchor=south east},
			xtick={0.1,0.15,0.2,0.3,0.4,0.5,0.6,0.7,0.8,0.9,1},
			xticklabels={0.1,0.15,0.2,0.3,0.4,0.5,,0.7,,,1},
			]
			
			\addplot+[color=magenta, line width=2pt, mark=otimes*] coordinates {
				(0.97500000, 0.450358)
				(0.88143408, 0.270029)
				(0.77381775, 0.172442)
				(0.68780233, 0.122118)
				(0.57163747, 0.070589)
				(0.48028938, 0.039271)
				(0.36848788, 0.012027)
				(0.28565350, 0.003128)
				(0.16601349, 0.00008636)
				(0.12929029, 0.00001863)
			};
			\addlegendentry{$d=20,\;N_t=200$ frontier}
			
\addplot[
only marks,
mark=*,
mark size=1.8pt,
black,
forget plot
] coordinates {
	(0.16601349, 0.00008636)
	(0.28565350, 0.003128)
	(0.36848788, 0.012027)
	(0.48028938, 0.039271)
	(0.77381775, 0.172442)
};

\addplot[black, line width=1.6pt] coordinates {
	(0.13834458, 0.00003328)
	(0.19921619, 0.00022409)
};
\addlegendentry{Theoretical tangent $A$}

\addplot[blue, line width=1.6pt] coordinates {
	(0.13834458, 0.00002824)
	(0.19921619, 0.00026405)
};
\addlegendentry{Empirical tangent $\hat A$}
			
			
			\addplot[black, line width=1.6pt] coordinates {
				(0.13834458, 0.00003328)
				(0.19921619, 0.00022409)
			};
			\addplot[blue, line width=1.6pt] coordinates {
				(0.13834458, 0.00002824)
				(0.19921619, 0.00026405)
			};
			
			\addplot[black, line width=1.6pt] coordinates {
				(0.23804458, 0.00138713)
				(0.34278420, 0.00705367)
			};
			\addplot[blue, line width=1.6pt] coordinates {
				(0.23804458, 0.00093730)
				(0.34278420, 0.01043890)
			};
			
			\addplot[black, line width=1.6pt] coordinates {
				(0.30707323, 0.00769420)
				(0.44218546, 0.01879972)
			};
			\addplot[blue, line width=1.6pt] coordinates {
				(0.30707323, 0.00532374)
				(0.44218546, 0.02717049)
			};
			
			\addplot[black, line width=1.6pt] coordinates {
				(0.40024115, 0.03031340)
				(0.57634726, 0.05087557)
			};
			\addplot[blue, line width=1.6pt] coordinates {
				(0.40024115, 0.02124375)
				(0.57634726, 0.07259601)
			};
			
			\addplot[black, line width=1.6pt] coordinates {
				(0.64484813, 0.14370167)
				(0.92858130, 0.20693040)
			};
			\addplot[blue, line width=1.6pt] coordinates {
				(0.64484813, 0.10107459)
				(0.92858130, 0.29420097)
			};
			
			
			\node[font=\scriptsize, anchor=north west, align=left]
			at (axis cs:0.171,0.000074)
			{\shortstack{$A=5.23$\\$\hat A=6.13$}};
			
			\node[font=\scriptsize, anchor=north west, align=left]
			at (axis cs:0.294,0.00255)
			{\shortstack{$A=4.46$\\$\hat A=6.61$}};
			
			\node[font=\scriptsize, anchor=north west, align=left]
			at (axis cs:0.379,0.0099)
			{\shortstack{$A=2.45$\\$\hat A=4.47$}};
			
			\node[font=\scriptsize, anchor=north west, align=left]
			at (axis cs:0.494,0.0325)
			{\shortstack{$A=1.42$\\$\hat A=3.37$}};
			
			\node[font=\scriptsize, anchor=north west, align=left]
			at (axis cs:0.792,0.145)
			{\shortstack{$A=1.00$\\$\hat A=2.93$}};
			
		\end{axis}
	\end{tikzpicture}
	\caption{Optimized excess-error frontier for the conditional perceptron in the $d=20$, $N_t=200$ setting, with local theoretical tangents (black) and empirical tangents (blue) superimposed at representative operating points. Since the frontier is obtained by re-optimizing the operating point for each target compute gap, it is not a single fixed-$q$ power law, so the relevant comparison is local rather than global.}
	\label{fig:app_generrgaps_slopeanno}
\end{figure}

	Figure~\ref{fig:app_generrgaps_slopeanno} isolates the $d=20,N_t=200$ optimized frontier from
	Fig.~\ref{fig:app_generrgaps} and superimposes short local tangent segments at representative operating points.
	Note that this frontier is not a single fixed-$q$ curve. Rather, for each target
	compute gap, we choose a near-optimal operating point as in Appendix~\ref{app:threshold_calibration}, so there is no single
	theorem-predicted slope for the entire line. For Fig.~\ref{fig:app_generrgaps} in the
	$d=20,N_t=200$ setting, the approximate operating points are
	$q=0.36,0.40,0.55,0.70,0.80$ for compute gaps
	$0.15,0.25,0.40,0.50,0.75$, respectively. Using
	$\epsilon_{\mathrm{uc}}=0.049$, Theorem~3.1 gives the corresponding exponents
\begin{align}
	A=5.23,\ 4.46,\ 2.45,\ 1.42,\ 1.00,
\end{align}
	while the local empirical log-log slopes are
\begin{align}
	\hat A=6.13,\ 6.61,\ 4.47,\ 3.37,\ 2.93.
\end{align}
	
	In the figure, the dashed segments are the local theoretical tangents
	$y=y_0(x/x_0)^A$, while the dotted segments are the local empirical tangents
	$y=y_0(x/x_0)^{\hat A}$, each drawn through the corresponding point $(x_0,y_0)$ on the
	optimized frontier. The theoretical and empirical slopes show the same qualitative trend:
	they are steeper at smaller compute gaps and flatter at larger compute gaps. Theorem~3.1
	also predicts that as the compute gap approaches zero, the exponent should increase to
\begin{align}
	A \to \frac{0.5}{\sin^2(\pi\epsilon_{\mathrm{uc}})} \approx 21.7.
\end{align}

\subsection{Multi-layer pruning on CIFAR-10, additional theory comparison}
\label{app:cifar10_theory}

This appendix section complements Section~\ref{sec:static-networks} by visualizing the Gaussian-limit prediction of
Theorem~\ref{theorem:deep_general_phi} for renormalized neuronwise magnitude pruning, without retraining. The key quantity is the
\emph{alignment} between the unpruned and renormalized-pruned forward passes, captured by the diagonal cross-kernel
\begin{align}
\alpha_\ell(q)\;\triangleq\;C_{\infty,ii}^{(\ell)}(q),\qquad \alpha_0(q)=1,
\end{align}
where $C_\infty^{(\ell)}$ is defined by the cross-kernel recursion \eqref{eq:Crec_pointwise} and the factor $\sqrt{\tau_q}$
injects the per-layer overlap loss induced by pruning. Under the unit self-kernel convention $K_{\infty,ii}^{(\ell)}=1$,
Theorem~\ref{theorem:deep_general_phi} yields the closed NMSE form
\begin{align}
\mathrm{NMSE}_{\mathrm{th}}(q)\;=\;2\bigl(1-\alpha_L(q)\bigr),
\end{align}
so the predicted distortion is entirely an \emph{alignment deficit} accumulated across depth, rather than a scale attenuation.

\begin{figure}[H]
	\centering
	\includegraphics[width=0.5\linewidth]{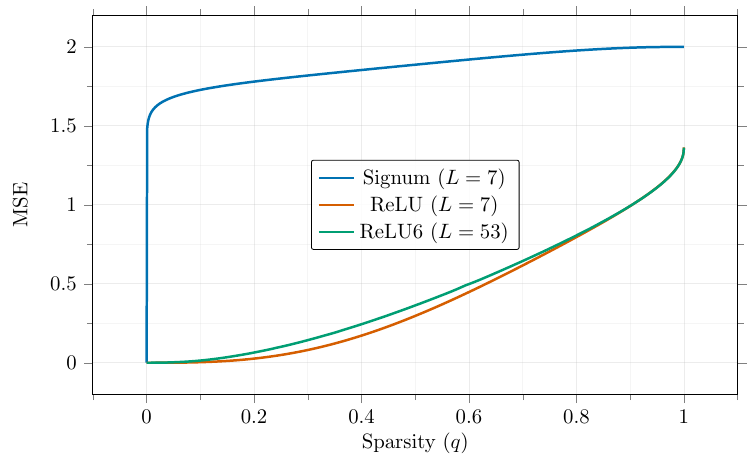}
	\vspace{-4pt}
	\caption{
		Theoretical NMSE curves $\mathrm{NMSE}_{\mathrm{th}}(q)=2(1-\alpha_L(q))$ from Theorem~\ref{theorem:deep_general_phi},
		obtained by iterating the cross-kernel recursion \eqref{eq:Crec_pointwise} at depth $L$ and extracting the diagonal
		alignment $\alpha_L(q)=C_{\infty,ii}^{(L)}(q)$. Shown are sign and (normalized) ReLU at depth $L=7$, and ReLU6 at
		depth $L=53$.}
	\label{fig:app_cifar10_theory}\vspace{-10pt}
\end{figure}

Figure~\ref{fig:app_cifar10_theory} plots the predicted distortion as a function of the pruning ratio $q$ under the
renormalized pruning model. The separation between curves reflects how the activation transforms the overlap loss
$\sqrt{\tau_q}$ through the Gaussian kernel operator $\mathcal{T}_\phi$ in \eqref{eq:Crec_pointwise}. In particular,
the sign nonlinearity produces a sharper alignment loss, so $\alpha_\ell(q)$ drops faster with depth and the predicted
NMSE is larger. Normalized ReLU yields a slower accumulation of alignment loss. ReLU6 behaves similarly to ReLU
locally, but its saturation together with the much larger effective depth leads to faster depth accumulation, which is
qualitatively consistent with the earlier degradation observed for \texttt{MobileNetV2}. The ordering in
Figure~\ref{fig:app_cifar10_theory} matches the empirical ordering in
Figures~\ref{fig:main_cifar_acc}--\ref{fig:main_cifar_nmse}.

Appendix~\ref{app:nonrenorm_static_networks} gives the corresponding recursion without renormalization. In that case,
pruning both changes direction and shrinks signal energy: a pruned row has limiting squared norm $\tau_q$ and overlap
$\tau_q$ with the original row. Renormalization restores the row norm and changes the overlap to $\sqrt{\tau_q}$. Thus,
the main text uses the renormalized model to remove this scale attenuation and focus on alignment loss. For sign
activations, the diagonal alignment recursion is unchanged for $q<1$ because signs are scale-invariant. For normalized
ReLU, pruning without renormalization adds a layerwise attenuation: the pruned self-kernel is $\tau_q^\ell$, and the raw
output MSE, normalized by the unpruned output variance, is
\begin{align}
\mathrm{MSE}^{\mathrm{un}}_{\mathrm{th}}(q)
=
1+\tau_q^L-2\tau_q^{L/2}\beta_L(q),
\end{align}
where $\beta_L(q)$ follows the same normalized alignment recursion as the renormalized curve in
Figure~\ref{fig:app_cifar10_theory}. For ReLU6, which is not homogeneous because of saturation, we evaluate the same
unrenormalized self-kernel and cross-kernel recursion directly.


Figure~\ref{fig:main_cifar_anal} separates the effect of removing the renormalization. The sign curve remains
close to the renormalized prediction because the sign activation discards scale information, so the main effect is still the
loss of directional alignment. In contrast, the ReLU and ReLU6 curves are substantially smaller at high sparsity because
the pruned branch is also attenuated in magnitude. In the extreme sparsity limit, the ReLU-type pruned output collapses
toward zero, so the raw MSE approaches the variance of the full output, namely $1$, rather than the decorrelation value
$2$ that appears when both branches are variance-normalized. Thus, Figures~\ref{fig:app_cifar10_theory} and
\ref{fig:main_cifar_anal} show complementary quantities: the former isolates alignment loss after scale
correction, while the latter describes the raw distortion induced by leaving the surviving weights unscaled.


\subsection{Empirical diagnostic for Assumption~5.1 on DistilBERT (SST-2)}
\label{app:assump51_diagnostic}

Assumption~5.1 models the dependence between an exit score $U$ and the final score $V$ via an approximately linear
conditional mean and a light tailed residual. Since classification scores can be strongly non Gaussian when pooled
across classes, we perform the diagnostic after within class standardization and then pool the standardized values.

\paragraph{Scores and standardization.}
For each example $(x,y)$, let $U(x)$ denote the scalar exit score at layer $\ell$ and let $V(x)$ denote the scalar
final score. For SST-2, we take $V$ to be the logit margin, for example $V=\text{logit}_{\text{pos}}-\text{logit}_{\text{neg}}$,
and $U$ to be the output of the learned exit head applied to the layer $\ell$ representation.
To suppress class mixture effects, we standardize within each class $c\in\{0,1\}$:
\begin{equation}
	\widetilde{U} \triangleq \frac{U-\mathbb{E}[U\mid Y=c]}{\sqrt{\mathrm{Var}(U\mid Y=c)}},
	\qquad
	\widetilde{V} \triangleq \frac{V-\mathbb{E}[V\mid Y=c]}{\sqrt{\mathrm{Var}(V\mid Y=c)}},
	\qquad \text{for } Y=c,
	\label{eq:within_class_standardize}
\end{equation}
and then pool $\{(\widetilde{U}_i,\widetilde{V}_i)\}$ across classes.

\paragraph{Conditional mean check.}
We check whether the pooled conditional mean is approximately linear:
\begin{equation}
	\mathbb{E}[\widetilde{V}\mid \widetilde{U}] \approx \rho_\ell\,\widetilde{U}.
	\label{eq:condmean_check}
\end{equation}
In Figure~\ref{fig:assump51_distilbert_grid}, the dots are binned estimates of
$\mathbb{E}[\widetilde{V}\mid \widetilde{U}]$ (quantile bins in $\widetilde{U}$), and the line is $\rho_\ell \widetilde{U}$,
where $\rho_\ell$ is the empirical correlation between $\widetilde{U}$ and $\widetilde{V}$.
The fit becomes tighter at later depth, consistent with increasing alignment as $\ell$ grows.

\paragraph{Residual tail check.}
Let $a_\ell,b_\ell$ be the least squares regression coefficients of $\widetilde{V}$ on $\widetilde{U}$, and define the
residual
\begin{equation}
	R_\ell \triangleq \widetilde{V}-(a_\ell+b_\ell \widetilde{U}).
	\label{eq:residual_def}
\end{equation}
Assumption \ref{assump:subg_coupling} is intended as an upper bound surrogate, so we do not require exact Gaussianity. The diagnostic below checks whether the linear conditional-mean and light-tail behavior hold approximately over the range that controls high-confidence exits. We empirically check for light tails by plotting
\begin{equation}
	t^2 \longmapsto \log \Pr\big(|R_\ell|\ge t\big).
	\label{eq:tail_plot_def}
\end{equation}
A sub Gaussian-like decay corresponds to an approximately linear trend in this plot over a moderate deviation range.
Figure~\ref{fig:assump51_distilbert_grid} shows such approximately linear decay for $\ell\in\{1,3,5\}$.
Flattening in the extreme tail is expected from finite sample resolution.

\begin{figure}[H]
	\centering
	\begin{subfigure}[H]{0.32\textwidth}
		\centering
		\includegraphics[width=\linewidth]{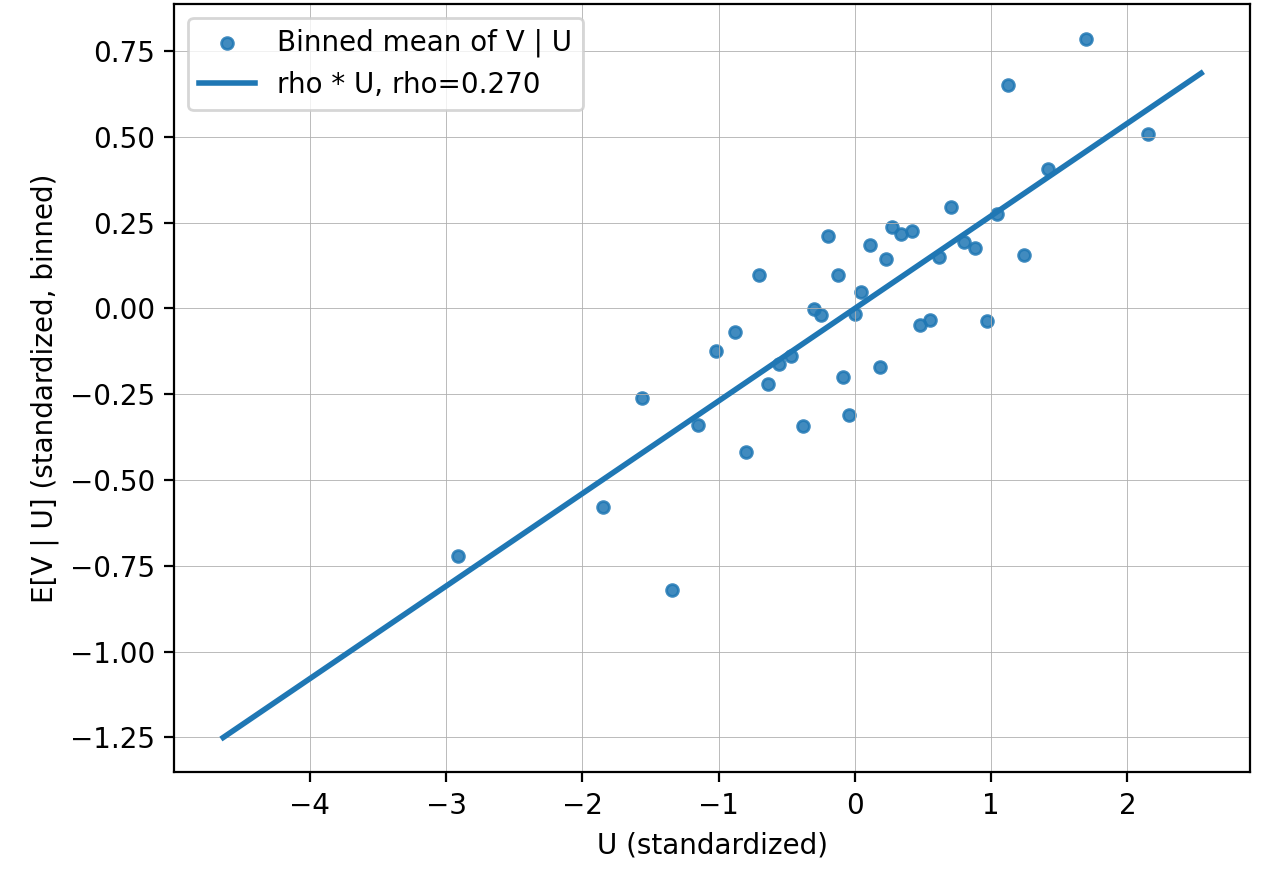}
		\caption{Conditional mean, $\ell=1$.}
		\label{fig:assump51_db_ell1_cond}
	\end{subfigure}
	\hfill
	\begin{subfigure}[H]{0.32\textwidth}
		\centering
		\includegraphics[width=\linewidth]{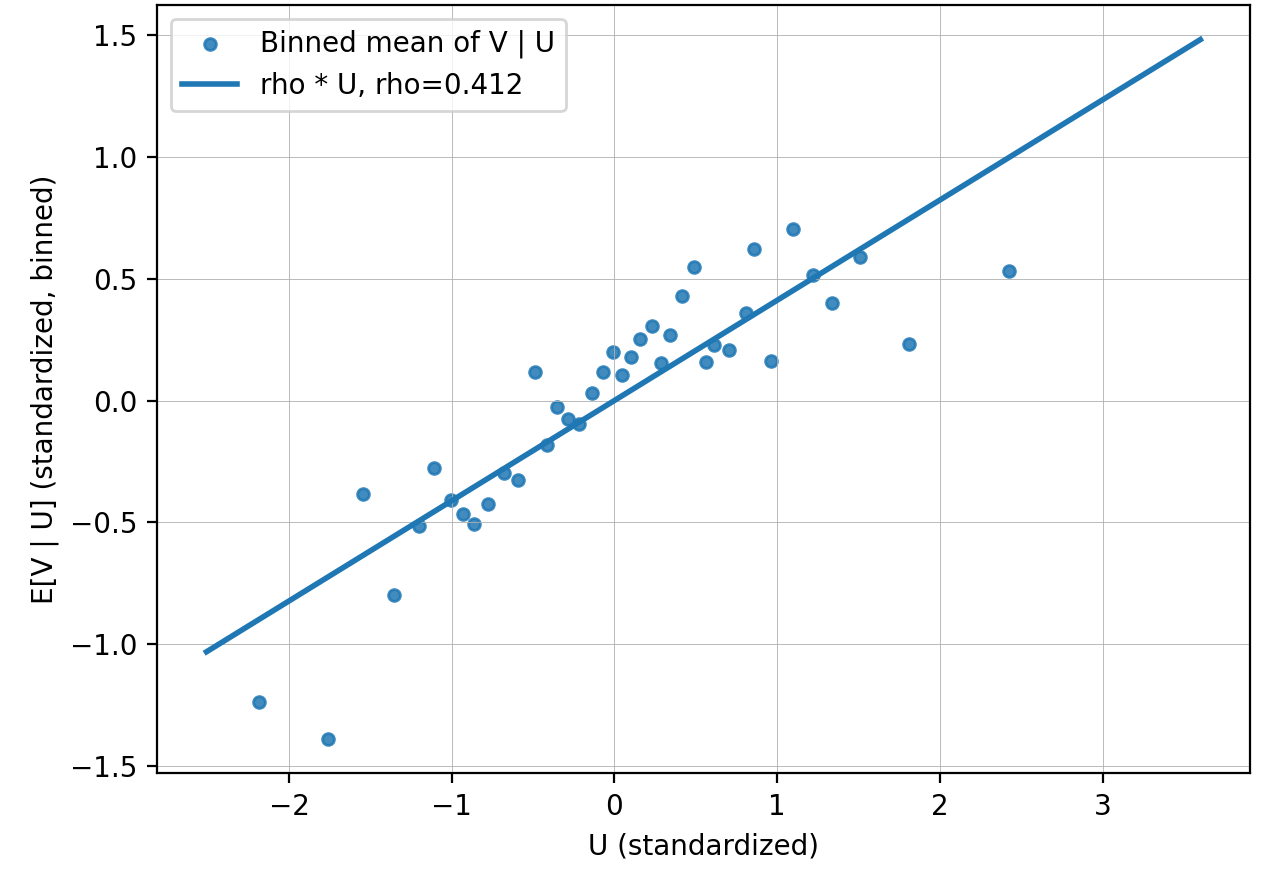}
		\caption{Conditional mean, $\ell=3$.}
		\label{fig:assump51_db_ell3_cond}
	\end{subfigure}
	\hfill
	\begin{subfigure}[H]{0.32\textwidth}
		\centering
		\includegraphics[width=\linewidth]{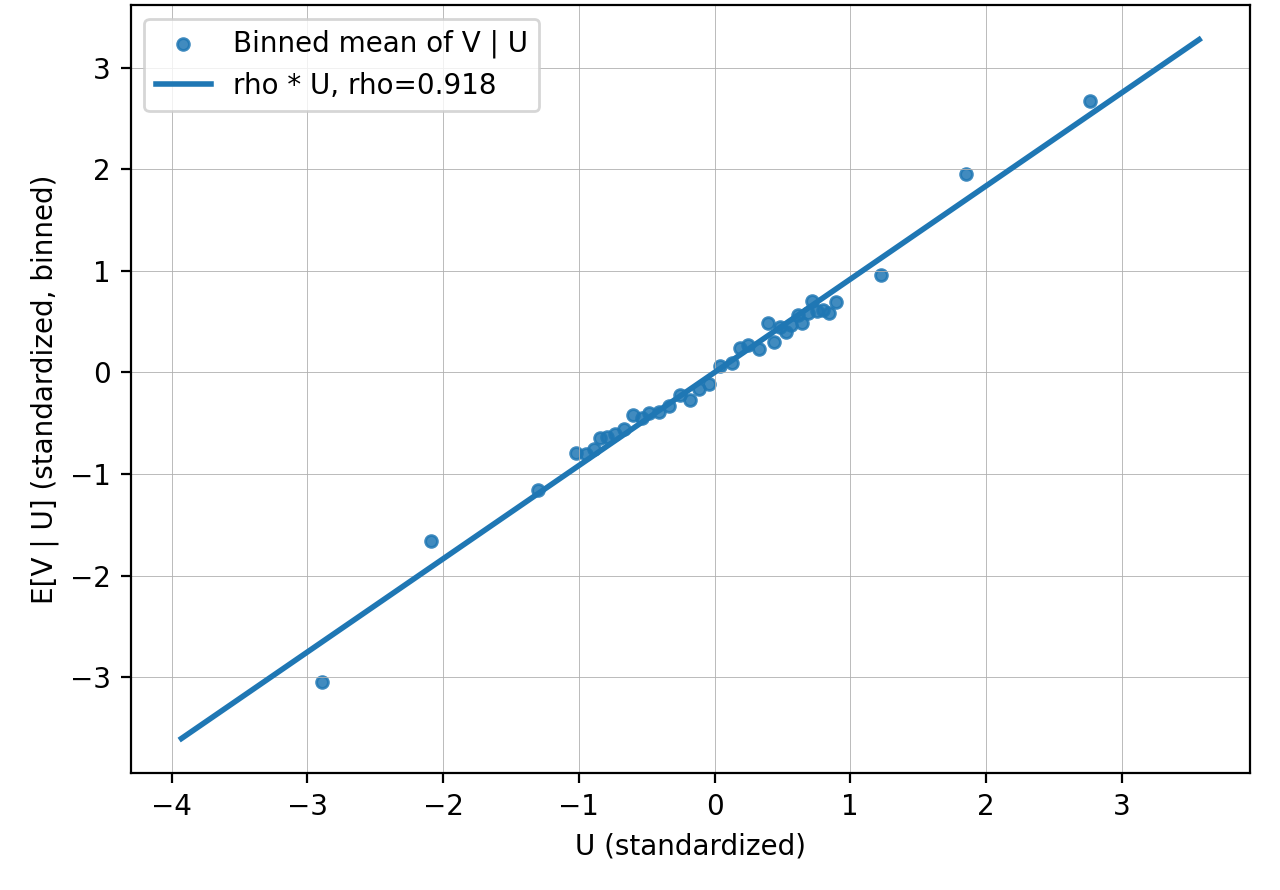}
		\caption{Conditional mean, $\ell=5$.}
		\label{fig:assump51_db_ell5_cond}
	\end{subfigure}
	
	\vspace{6pt}
	
	\begin{subfigure}[H]{0.32\textwidth}
		\centering
		\includegraphics[width=\linewidth]{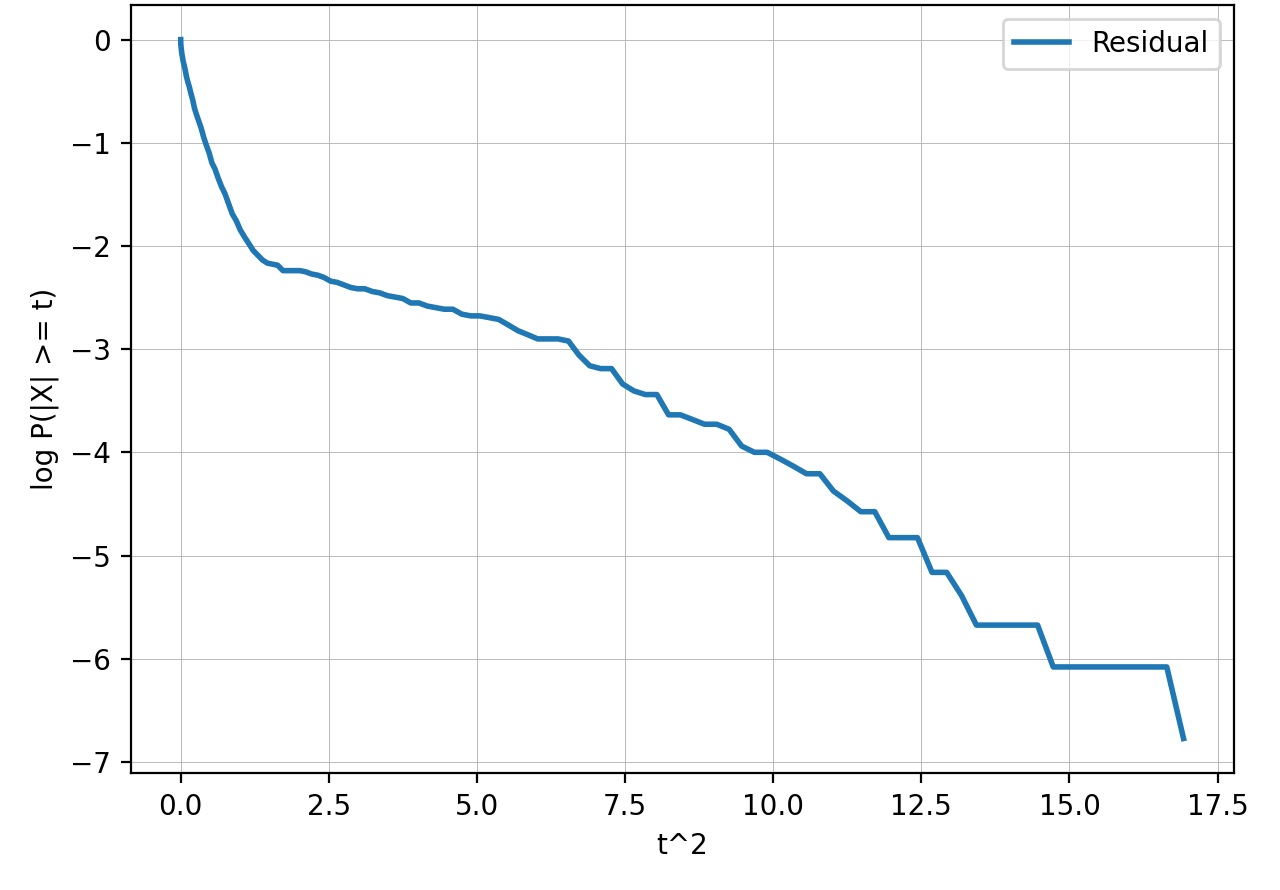}
		\caption{Residual tails, $\ell=1$.}
		\label{fig:assump51_db_ell1_tail}
	\end{subfigure}
	\hfill
	\begin{subfigure}[H]{0.32\textwidth}
		\centering
		\includegraphics[width=\linewidth]{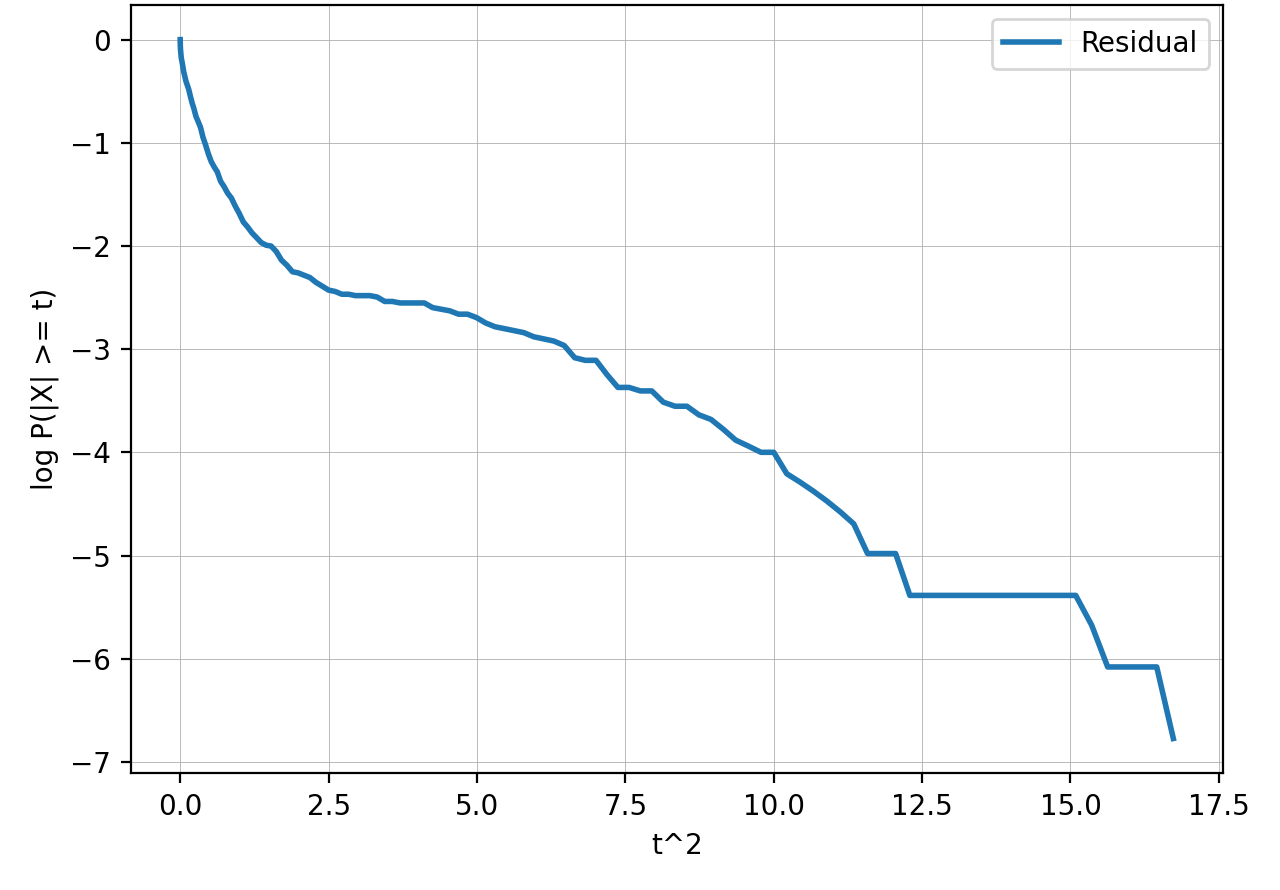}
		\caption{Residual tails, $\ell=3$.}
		\label{fig:assump51_db_ell3_tail}
	\end{subfigure}
	\hfill
	\begin{subfigure}[H]{0.32\textwidth}
		\centering
		\includegraphics[width=\linewidth]{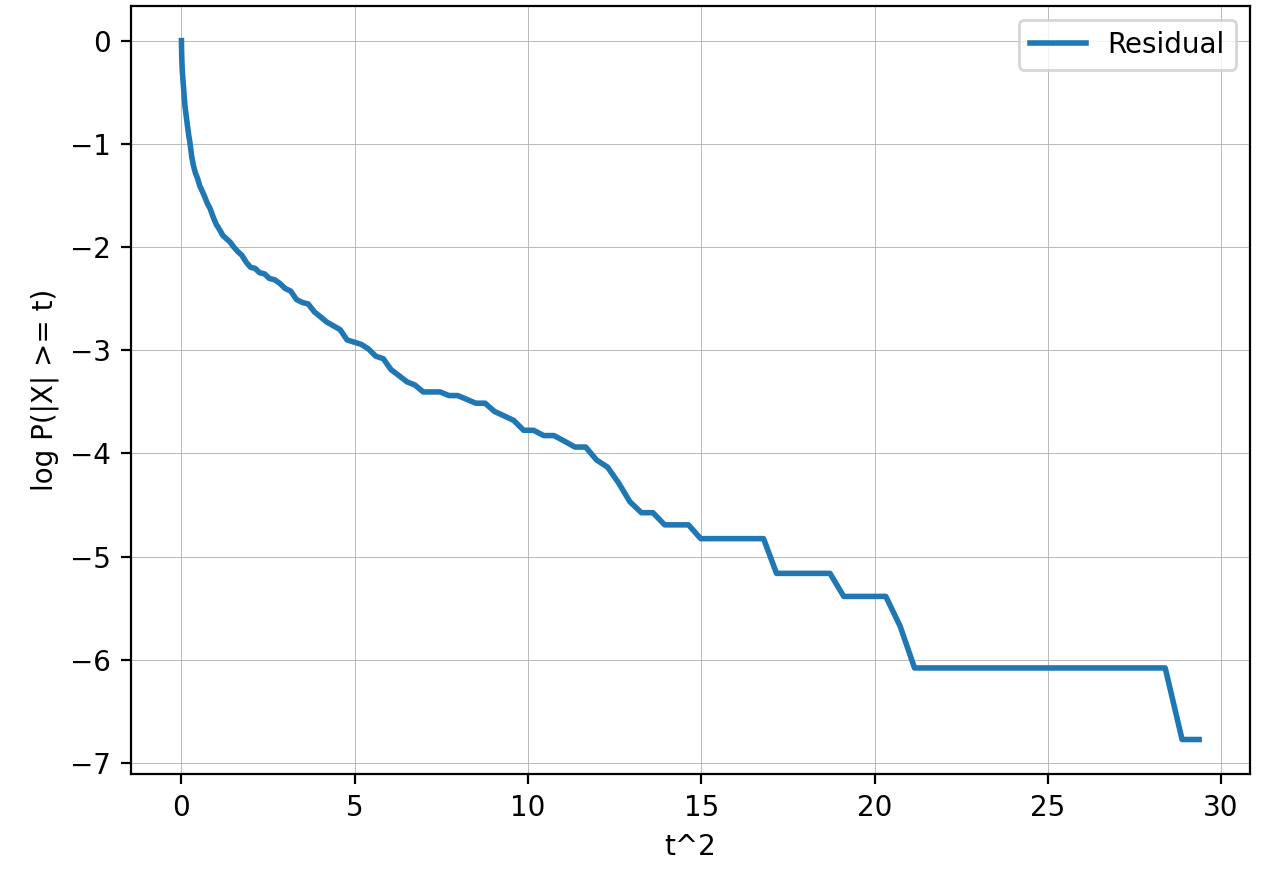}
		\caption{Residual tails, $\ell=5$.}
		\label{fig:assump51_db_ell5_tail}
	\end{subfigure}
	
	\caption{Empirical diagnostic of Assumption~5.1 on DistilBERT (SST-2) after within class standardization.
		Top row: pooled conditional mean estimates $\mathbb{E}[\widetilde{V}\mid \widetilde{U}]$ versus $\widetilde{U}$,
		overlaid with the linear surrogate $\rho_\ell \widetilde{U}$ for $\ell\in\{1,3,5\}$.
		Bottom row: pooled residual tail plots for $R_\ell=\widetilde{V}-(a_\ell+b_\ell\widetilde{U})$,
		showing approximately sub Gaussian-like decay over a moderate deviation range.
		Alignment increases with depth, and the conditional mean becomes increasingly linear.}
	\label{fig:assump51_distilbert_grid}
\end{figure}

Across early, mid, and late exit layers, the within class standardized diagnostic supports Assumption~5.1 as a
reasonable upper bound surrogate for the score coupling used in the analysis: the conditional mean is close to
linear, with alignment increasing with depth, and the residual tails decay approximately exponentially in $t^2$ over
the moderate deviation range that dominates the bound.



\subsection{ViT diagnostics (CIFAR-100): conditional mean and residual tails}
\label{app:vit_diagnostics}

We provide empirical diagnostics for the classwise Gaussian-coupling surrogate used in the main text.
For each chosen exit layer $\ell$, we form paired scalar scores $(U,V)$ from the $\ell$-exit head and the final head,
then standardize within each class to remove class-dependent shifts and scales before pooling across classes.
We then check two signatures: (i) approximate linearity of the conditional mean $\mathbb{E}[V\mid U]$,
and (ii) light (near-Gaussian) tails for the residual $R \triangleq V - \rho_\ell U$, where $\rho_\ell$ is the pooled
within-class correlation of $(U,V)$.

Figure~\ref{fig:vit_cifar100_diagnostics} summarizes results for $\ell\in\{3,6,9,11\}$.
As depth increases, the pooled within-class alignment strengthens and the residual exhibits lighter tails and fewer extreme deviations from Gaussian behavior.
Concretely, the pooled correlation grows from $\rho \approx 0.199$ at $\ell=3$ to $\rho \approx 0.407$ at $\ell=6$,
$\rho \approx 0.616$ at $\ell=9$, and $\rho \approx 0.825$ at $\ell=11$.
In parallel, the pooled residual kurtosis decreases from about $4.725$ to $4.230$ to $3.509$ to $1.736$
(Table~\ref{tab:vit_diagnostics}). These trends support the surrogate as a useful upper-bound model rather than an exact distributional claim: the conditional mean is close to linear over the bulk of the mass, and the residual tail plots are consistent with light-tail behavior. This is the regime used by our analytical results to control high-confidence exits, so moderate approximation error changes constants but does not change the alignment-controlled power-law conclusion.

\begin{figure}[H]
	\vspace{-2pt}
	\centering
	
	\begin{subfigure}[H]{0.24\textwidth}
		\centering
		\includegraphics[width=\linewidth]{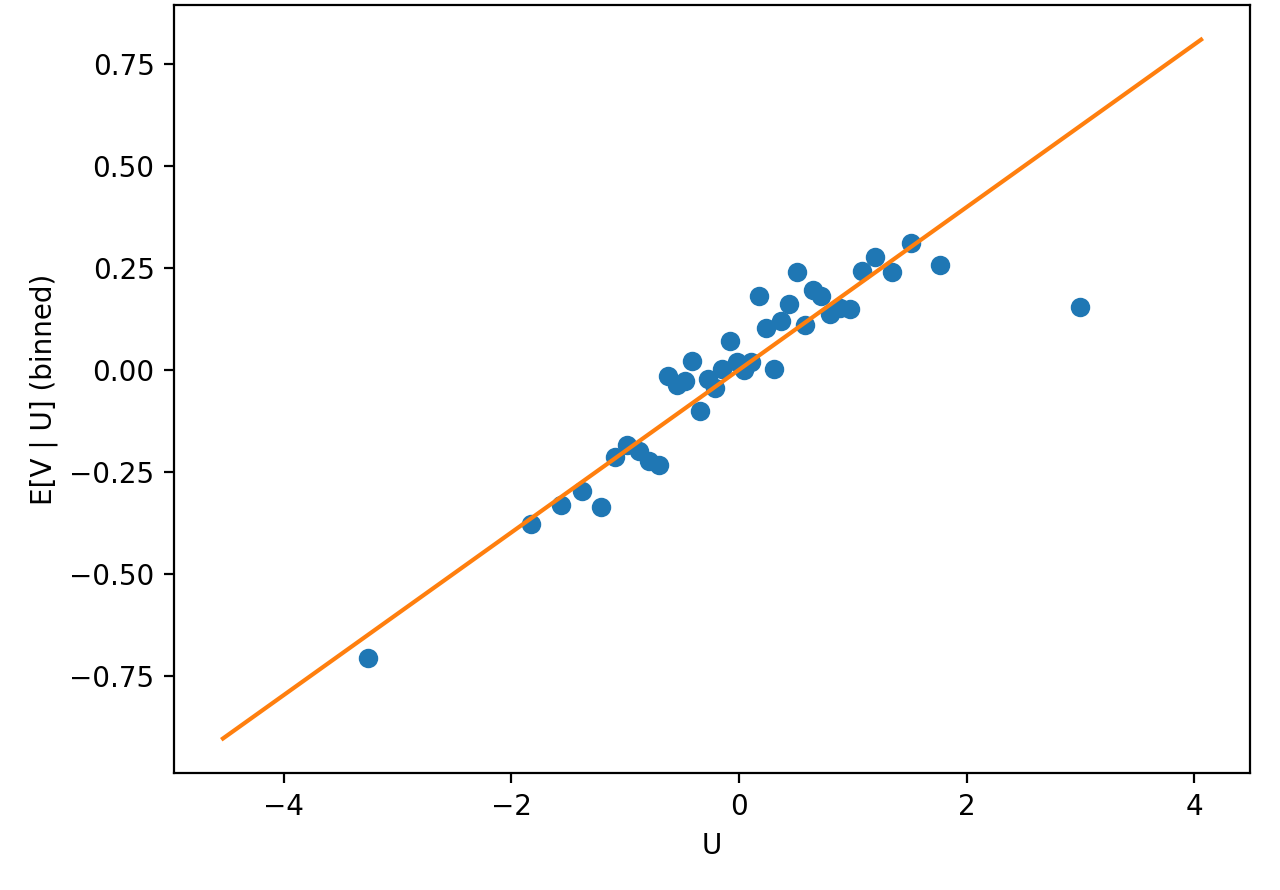}
		\caption{$\ell=3$: pooled $\mathbb{E}[V\mid U]$ (within-class standardized).}
		\label{fig:vit_cifar100_ell3_pooled_condmean}
	\end{subfigure}\hfill
	\begin{subfigure}[H]{0.24\textwidth}
		\centering
		\includegraphics[width=\linewidth]{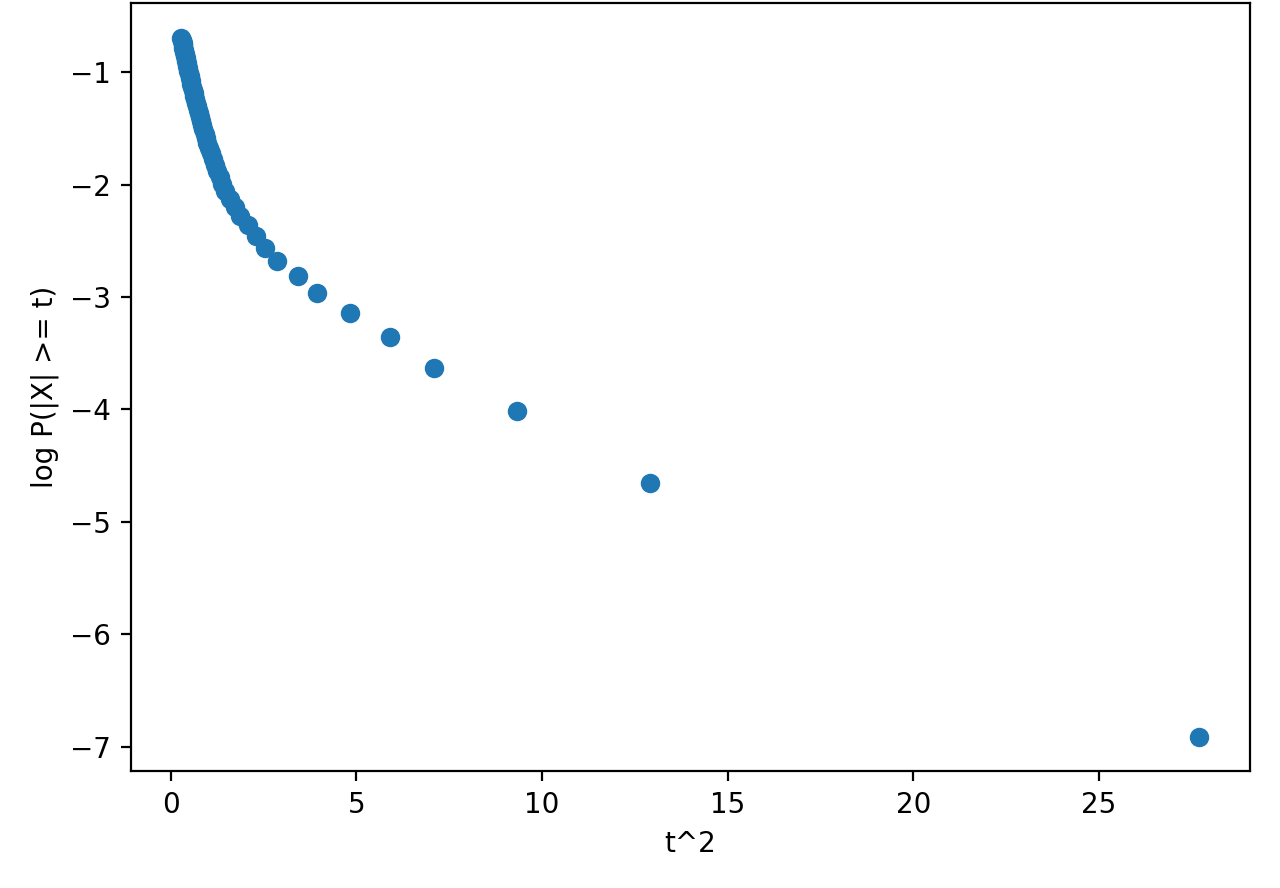}
		\caption{$\ell=3$: residual tail, $\log\mathbb{P}(|R|>t)$ vs $t^2$.}
		\label{fig:vit_cifar100_ell3_pooled_tail}
	\end{subfigure}\hfill
	\begin{subfigure}[H]{0.24\textwidth}
		\centering
		\includegraphics[width=\linewidth]{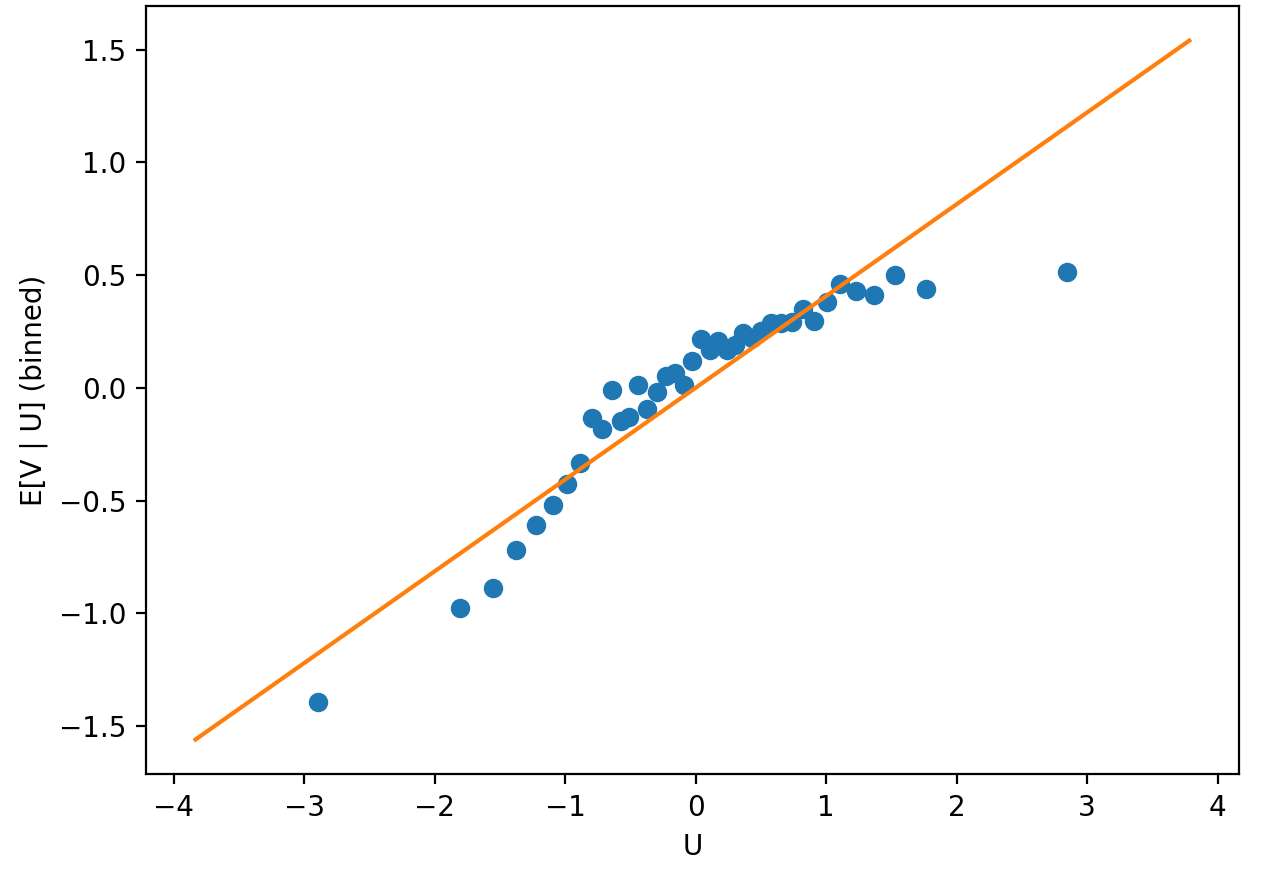}
		\caption{$\ell=6$: pooled $\mathbb{E}[V\mid U]$ (within-class standardized).}
		\label{fig:vit_cifar100_ell6_pooled_condmean}
	\end{subfigure}\hfill
	\begin{subfigure}[H]{0.24\textwidth}
		\centering
		\includegraphics[width=\linewidth]{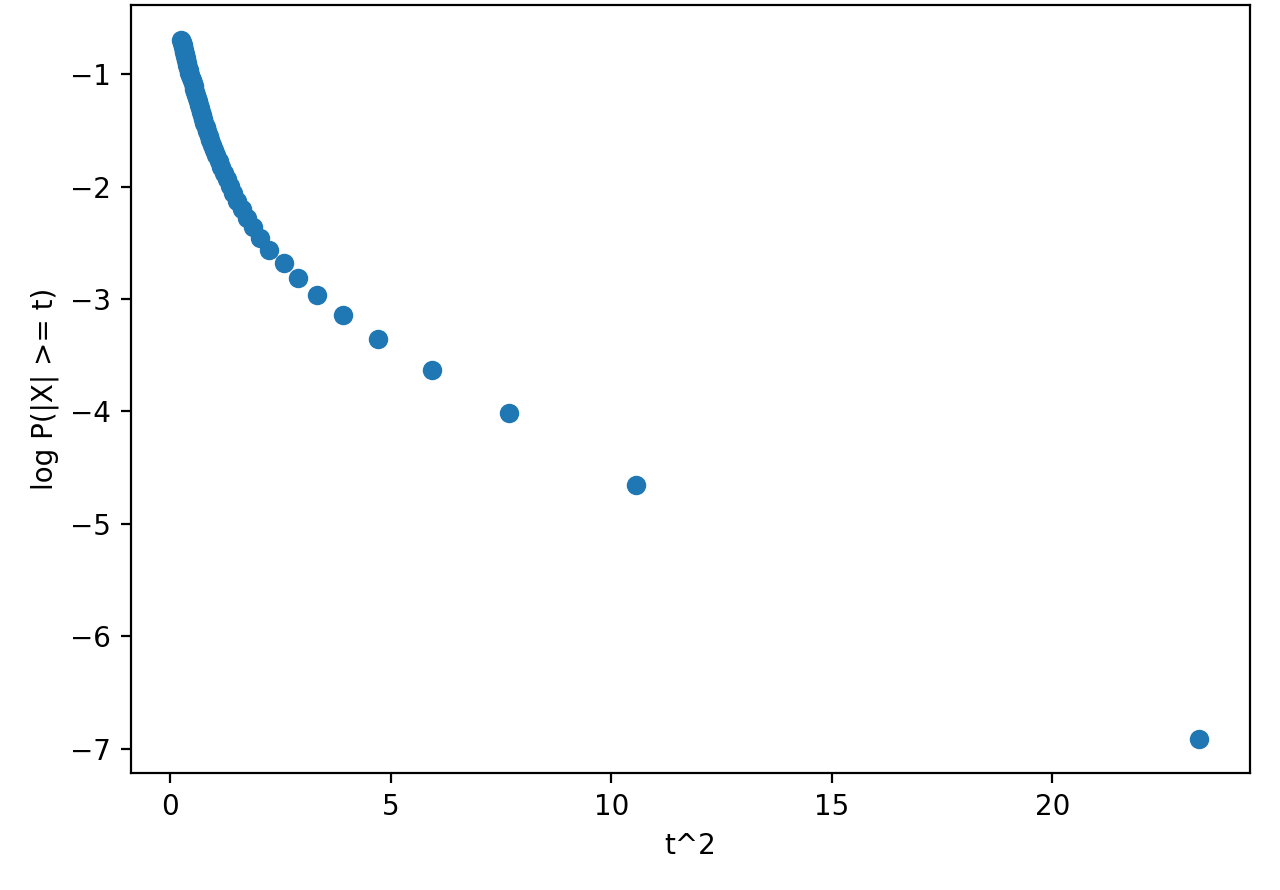}
		\caption{$\ell=6$: residual tail, $\log\mathbb{P}(|R|>t)$ vs $t^2$.}
		\label{fig:vit_cifar100_ell6_pooled_tail}
	\end{subfigure}
	
	\vspace{4pt}
	
	\begin{subfigure}[H]{0.24\textwidth}
		\centering
		\includegraphics[width=\linewidth]{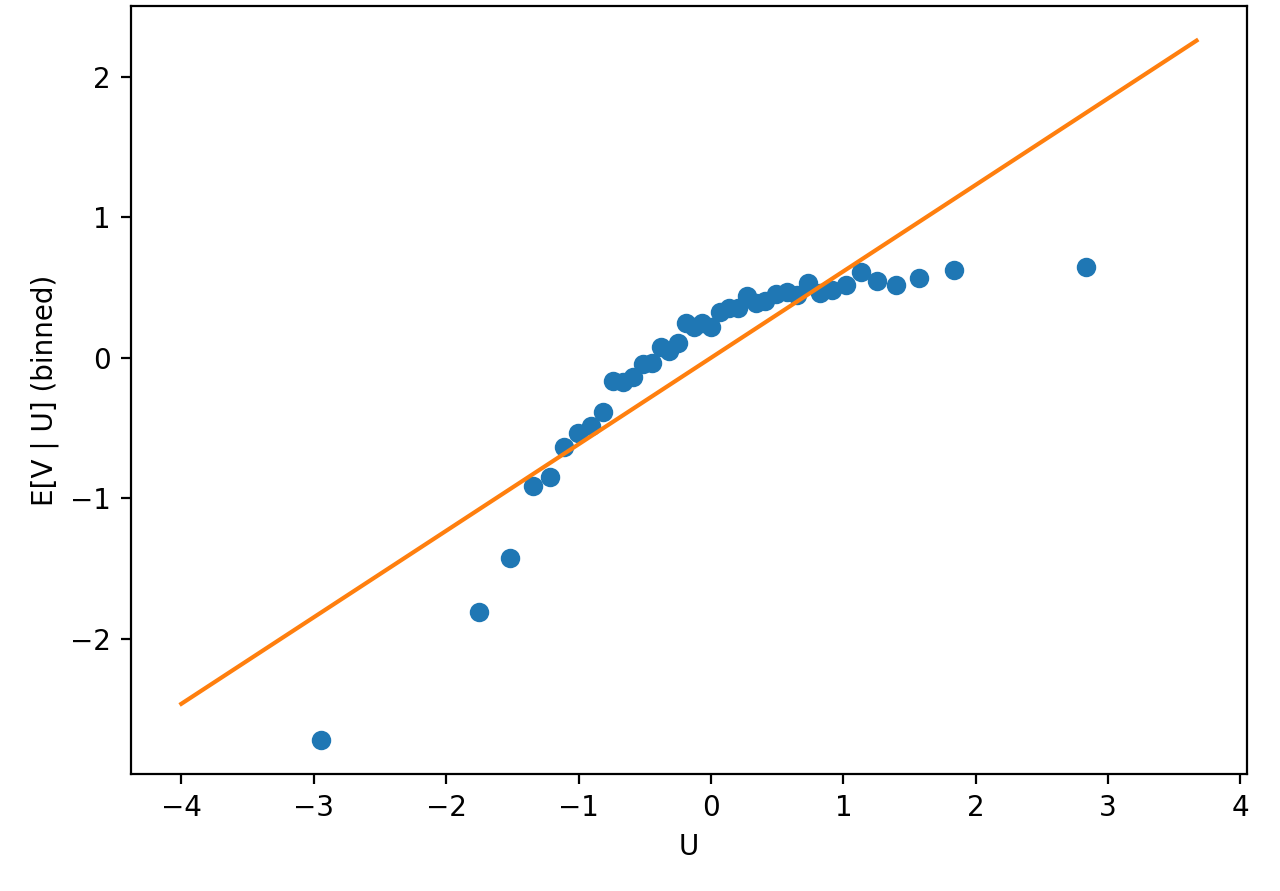}
		\caption{$\ell=9$: pooled $\mathbb{E}[V\mid U]$ (within-class standardized).}
		\label{fig:vit_cifar100_ell9_pooled_condmean}
	\end{subfigure}\hfill
	\begin{subfigure}[H]{0.24\textwidth}
		\centering
		\includegraphics[width=\linewidth]{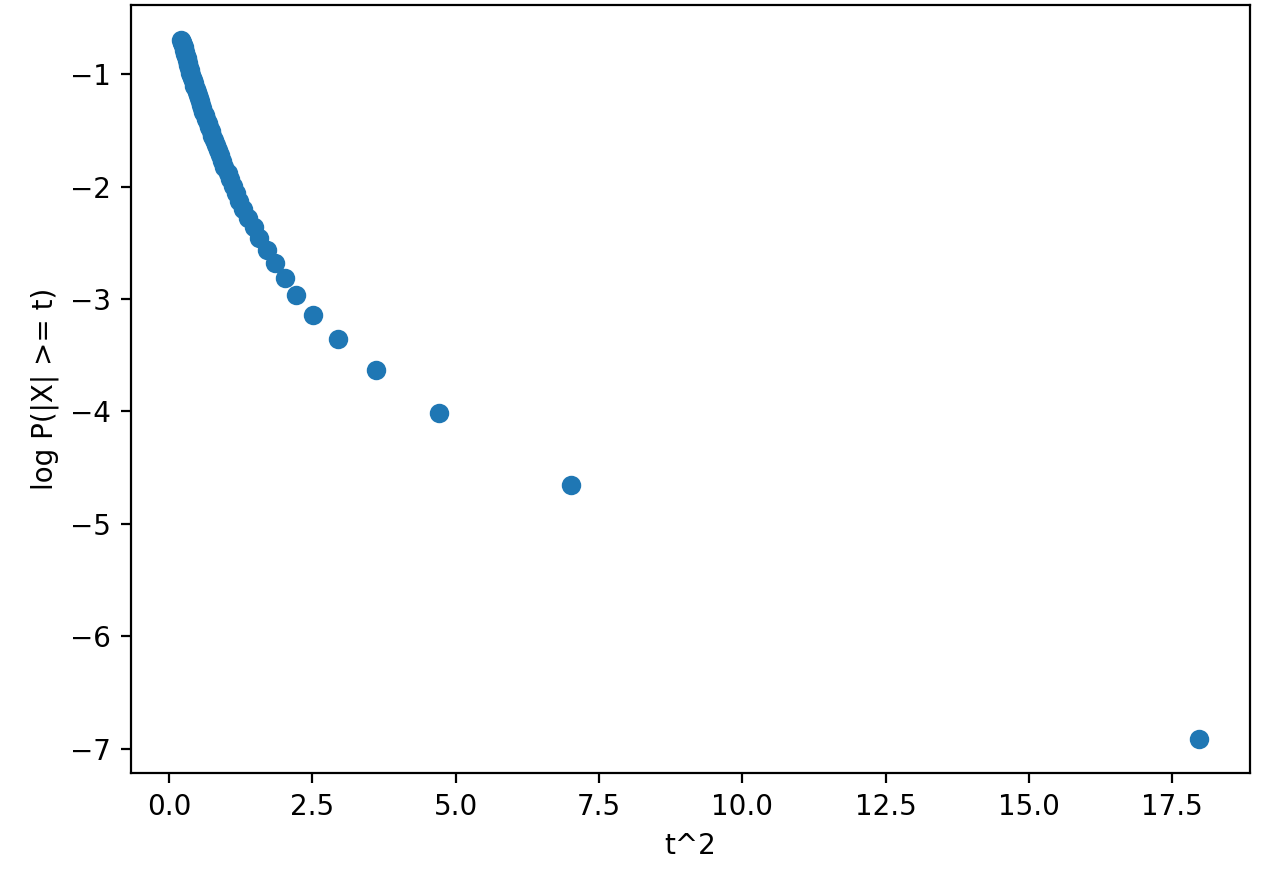}
		\caption{$\ell=9$: residual tail, $\log\mathbb{P}(|R|>t)$ vs $t^2$.}
		\label{fig:vit_cifar100_ell9_pooled_tail}
	\end{subfigure}\hfill
	\begin{subfigure}[H]{0.24\textwidth}
		\centering
		\includegraphics[width=\linewidth]{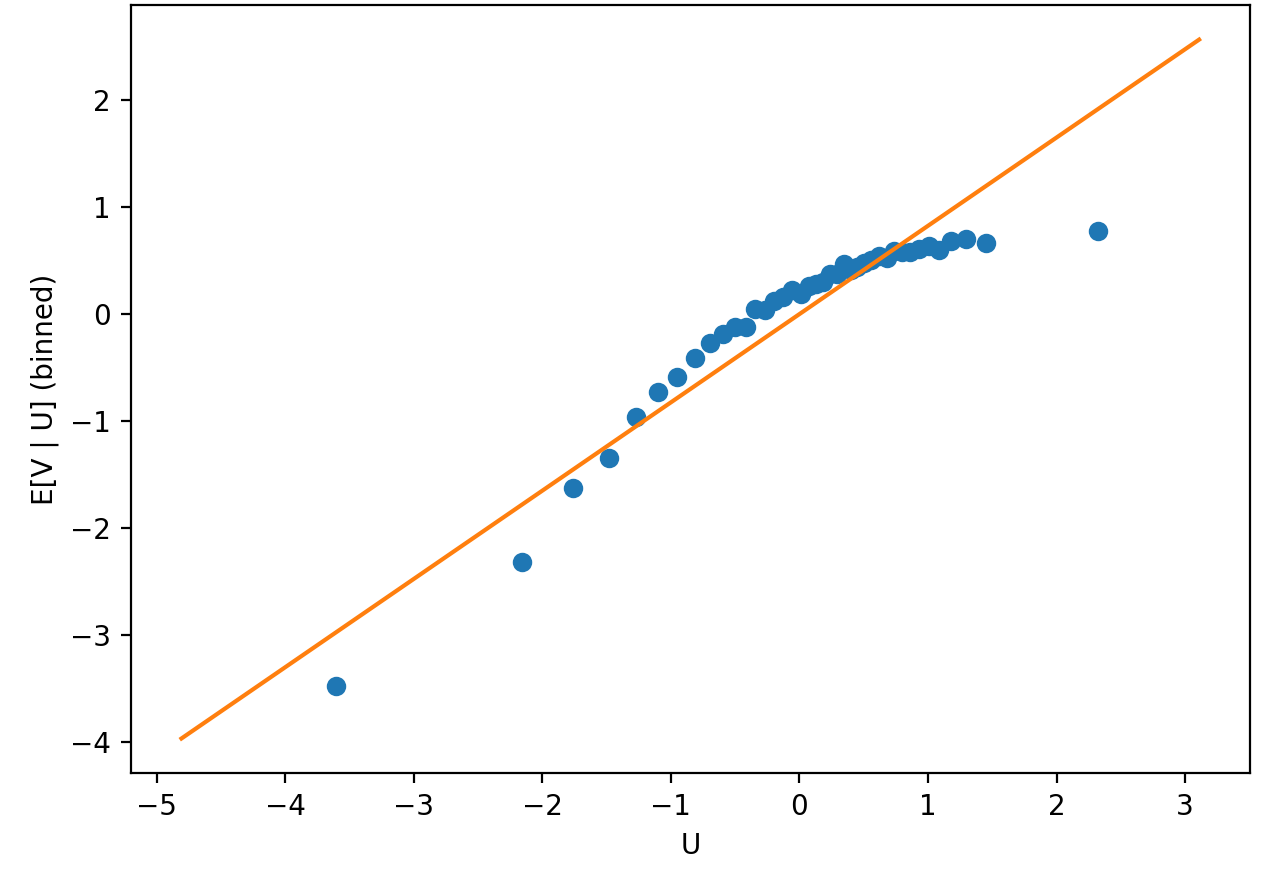}
		\caption{$\ell=11$: pooled $\mathbb{E}[V\mid U]$ (within-class standardized).}
		\label{fig:vit_cifar100_ell11_pooled_condmean}
	\end{subfigure}\hfill
	\begin{subfigure}[H]{0.24\textwidth}
		\centering
		\includegraphics[width=\linewidth]{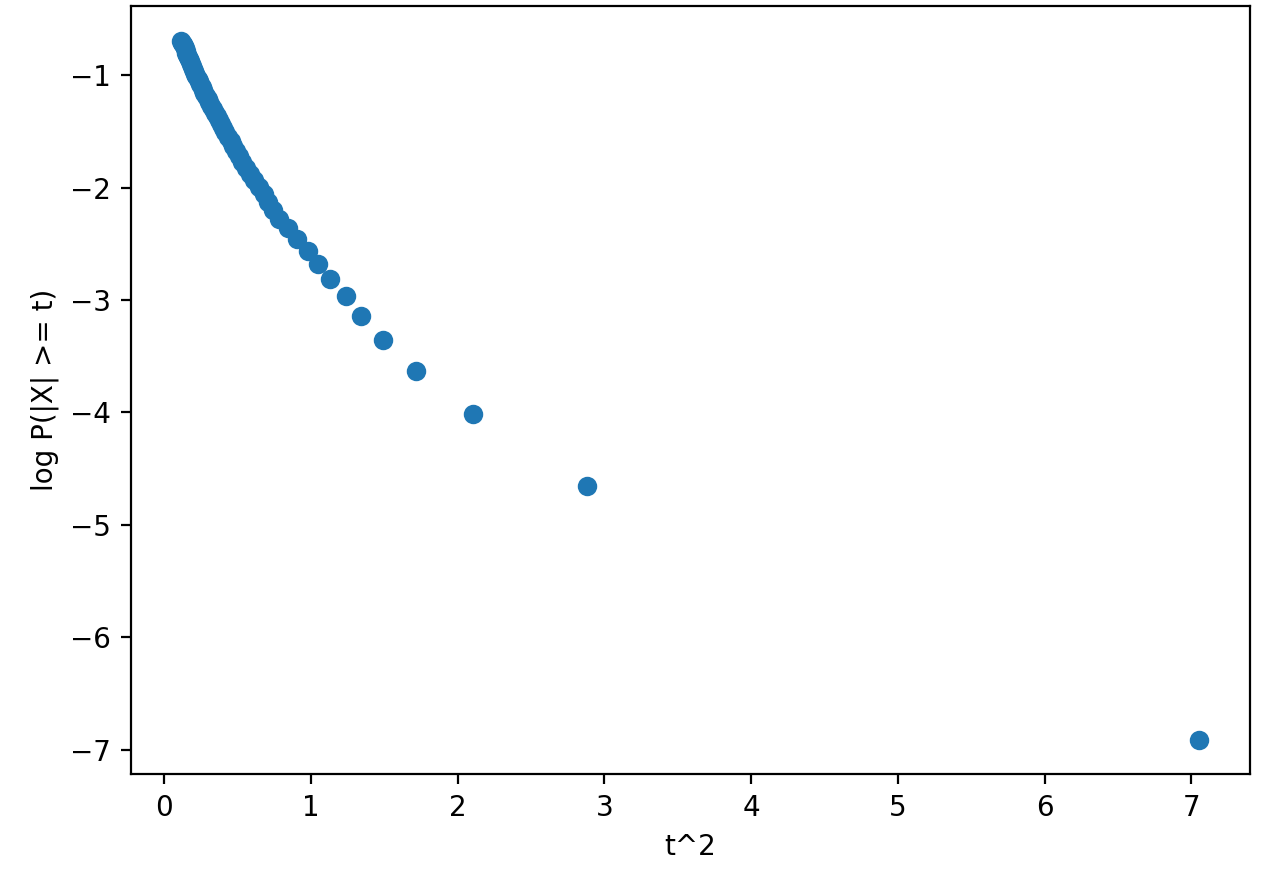}
		\caption{$\ell=11$: residual tail, $\log\mathbb{P}(|R|>t)$ vs $t^2$.}
		\label{fig:vit_cifar100_ell11_pooled_tail}
	\end{subfigure}
	
	\vspace{-4pt}
	\caption{
		ViT (CIFAR-100) diagnostics for the classwise Gaussian-coupling surrogate at exit layers $\ell\in\{3,6,9,11\}$.
		Within-class standardization makes the pooled conditional mean close to linear over the bulk of $U$,
		and the residual tail plots are consistent with light, near-Gaussian decay, increasingly so at deeper layers.
	}
	\label{fig:vit_cifar100_diagnostics}
	\vspace{-6pt}
\end{figure}

\begin{table}[H]
	\centering
	\caption{ViT CIFAR-100 diagnostics for the classwise coupling check (classwise pooled, within-class standardized). We report the pooled correlation $\rho_\ell$, the linear-fit $R^2$ for $\mathbb{E}[V\mid U]$, and tail/heavy-tail proxies via $\psi^2(U)$, $\psi^2(R)$, and residual kurtosis $\kappa(R)$. To summarize tail behavior in a single number, we report an empirical proxy for the squared sub-Gaussian
		Orlicz norm. Recall $\psi^2(X)\triangleq \inf\{s>0:\E[\exp(X^2/s^2)]\le 2\}$. Larger values indicate heavier tails.}
	\label{tab:vit_diagnostics}
	\begin{tabular}{cccccc}
		\toprule
		$\ell$ & $\rho_\ell$ & $R^2$ & $\psi^2(U)$ & $\psi^2(R)$ & $\kappa(R)$ \\
		\midrule
		3  & 0.199 & 0.040 & 3.374 & 6.316 & 4.725 \\
		6  & 0.407 & 0.166 & 2.584 & 6.372 & 4.230 \\
		9  & 0.616 & 0.379 & 2.657 & 5.930 & 3.509 \\
		11 & 0.825 & 0.680 & 3.783 & 3.221 & 1.736 \\
		\bottomrule
	\end{tabular}
\end{table}

%% file: tex/app/semi-struct-app.tex
	We next validate the semi-structured extension at the single-neuron level. We focus first on the practically relevant
$2\!:\!4$ case covered by Theorem~\ref{theorem:blockwise_nm}. Specifically, we draw
$\mathbf{W}\sim \mathrm{Unif}(\mathbb{S}^{n-1})$, partition its coordinates into disjoint blocks of size $M$, and within
each block keep the $N$ largest-magnitude entries to obtain the pruned vector $\mathbf{W}_p$. We then compare the
empirical raw and renormalized similarities,
$\mathbf{W}^T\mathbf{W}_p$ and $\mathbf{W}^T\widetilde{\mathbf{W}}_p$, against the predicted limits $\tau_{N:M}$ and
$\sqrt{\tau_{N:M}}$. For $n=4096$ and the practical $2\!:\!4$ pattern, the predicted constants are
$\tau_{2:4}=0.86755$ and $\sqrt{\tau_{2:4}}=0.93143$, while the corresponding empirical means are $0.86754$ and
$0.93142$, respectively. We also examine how concentration sharpens with dimension. In both the raw and renormalized
cases, the empirical standard deviation decays with an approximately $n^{-1/2}$ law, with fitted log-log slopes
$-0.4989$ and $-0.4993$, respectively.

\begin{figure}[H]
	\centering
	
	\begin{minipage}{0.48\textwidth}
		\centering
		\begin{tikzpicture}
			\begin{axis}[
				width=\linewidth,
				height=0.72\linewidth,
				tick label style={font=\small},
				label style={font=\small},
				title style={font=\small},
				legend style={font=\small},
				grid=both,
				title={(a) Empirical means versus $n$ for $2\!:\!4$},
								title style={at={(0.5,-0.25)}, anchor=north},
				xlabel={$n$},
				ylabel={Mean},
				xmode=log,
				xmin=100, xmax=10000,
				ymin=0.86, ymax=0.94,
				]
				\addplot[blue, thick, mark=*] table[col sep=comma, x=n, y=raw_mean] {fixed_sweep.csv};
				\addplot[blue, thick, dashed] table[col sep=comma, x=n, y=tau] {fixed_sweep.csv};
				\addplot[orange!85!black, thick, mark=square*] table[col sep=comma, x=n, y=ren_mean] {fixed_sweep.csv};
				\addplot[orange!85!black, thick, dashed] table[col sep=comma, x=n, y=sqrt_tau] {fixed_sweep.csv};
				\legend{{$\mathbf{W}^T\mathbf{W}_p$},{$\tau_{2:4}$},{$\mathbf{W}^T\widetilde{\mathbf{W}}_p$},{$\sqrt{\tau_{2:4}}$}}
			\end{axis}
		\end{tikzpicture}
	\end{minipage}
	\hfill
	\begin{minipage}{0.48\textwidth}
		\centering
		\begin{tikzpicture}
			\begin{axis}[
				width=\linewidth,
				height=0.72\linewidth,
				tick label style={font=\small},
				label style={font=\small},
				title style={font=\small},
				legend style={font=\small},
				grid=both,
				title={(b) Empirical standard deviations versus $n$ for $2\!:\!4$},
								title style={at={(0.5,-0.25)}, anchor=north},
				xlabel={$n$},
				ylabel={Standard deviation},
				xmode=log,
				ymode=log,
				xmin=100, xmax=10000,
				]
				\addplot[blue, thick, mark=*] table[col sep=comma, x=n, y=raw_std] {fixed_sweep.csv};
				\addplot[orange!85!black, thick, mark=square*] table[col sep=comma, x=n, y=ren_std] {fixed_sweep.csv};
				\legend{{$\mathbf{W}^T\mathbf{W}_p$},{$\mathbf{W}^T\widetilde{\mathbf{W}}_p$}}
			\end{axis}
		\end{tikzpicture}
	\end{minipage}
	
	\caption{Numerical validation of the fixed-block semi-structured limit for the practical $2\!:\!4$ pattern.
		The empirical means converge to the predicted constants $\tau_{2:4}$ and $\sqrt{\tau_{2:4}}$, and the empirical standard
		deviations decay approximately as $n^{-1/2}$.}
	\label{fig:semistructured_neuron_fixed}
\end{figure}

We also consider the regime in which $N$ and $M$ both scale linearly with $n$, namely $N=\beta_1 n$ and
$M=\beta_2 n$. As discussed in Appendix~\ref{sec:semistructured_nm}, this regime reduces to the original unstructured
constant $\tau_q$ with $q=1-\beta_1/\beta_2$. Figure~\ref{fig:semistructured_neuron_scaling} confirms this prediction for
$\beta_1=0.125$ and $\beta_2=0.25$, corresponding to $q=0.5$. For $n=4096$, the predicted constants are
$\tau_q=0.92867$ and $\sqrt{\tau_q}=0.96368$, while the empirical means are $0.92840$ and $0.96354$, respectively.
The fitted log-log slopes of the standard deviations are $-0.5103$ and $-0.5111$, again consistent with an $n^{-1/2}$
concentration law.

\begin{figure}[H]
	\centering
	\begin{minipage}{0.48\textwidth}
		\centering
		\begin{tikzpicture}
			\begin{axis}[
				width=\linewidth,
				height=0.72\linewidth,
				tick label style={font=\small},
				label style={font=\small},
				title style={font=\small},
				legend style={font=\small},
				grid=both,
				title={(a) Means in the regime $N,M\in O(n)$},
								title style={at={(0.5,-0.25)}, anchor=north},
				xlabel={$n$},
				ylabel={Mean},
				xmode=log,
				xmin=200, xmax=10000,
				ymin=0.92, ymax=0.97,
				]
				\addplot[blue, thick, mark=*] table[col sep=comma, x=n, y=raw_mean] {scaling_sweep.csv};
				\addplot[blue, thick, dashed] table[col sep=comma, x=n, y=tau_q] {scaling_sweep.csv};
				\addplot[orange!85!black, thick, mark=square*] table[col sep=comma, x=n, y=ren_mean] {scaling_sweep.csv};
				\addplot[orange!85!black, thick, dashed] table[col sep=comma, x=n, y=sqrt_tau_q] {scaling_sweep.csv};
				\legend{{$\mathbf{W}^T\mathbf{W}_p$},{$\tau_q$},{$\mathbf{W}^T\widetilde{\mathbf{W}}_p$},{$\sqrt{\tau_q}$}}
			\end{axis}
		\end{tikzpicture}
	\end{minipage}
	\hfill
	\begin{minipage}{0.48\textwidth}
		\centering
		\begin{tikzpicture}
			\begin{axis}[
				width=\linewidth,
				height=0.72\linewidth,
				tick label style={font=\small},
				label style={font=\small},
				title style={font=\small},
				legend style={font=\small},
				grid=both,
				title={(b) Standard deviations in the regime $N,M\in O(n)$},
								title style={at={(0.5,-0.25)}, anchor=north},
				xlabel={$n$},
				ylabel={Standard deviation},
				xmode=log,
				ymode=log,
				xmin=200, xmax=10000,
				]
				\addplot[blue, thick, mark=*] table[col sep=comma, x=n, y=raw_std] {scaling_sweep.csv};
				\addplot[orange!85!black, thick, mark=square*] table[col sep=comma, x=n, y=ren_std] {scaling_sweep.csv};
				\legend{{$\mathbf{W}^T\mathbf{W}_p$},{$\mathbf{W}^T\widetilde{\mathbf{W}}_p$}}
			\end{axis}
		\end{tikzpicture}
	\end{minipage}
	
	\caption{Numerical validation of the linear-scaling regime $N=\beta_1 n$, $M=\beta_2 n$ with
		$(\beta_1,\beta_2)=(0.125,0.25)$. In this case, the semi-structured limit reduces to the original unstructured constant
		$\tau_q$ with $q=0.5$, and the empirical means and standard deviations agree with this prediction.}
	\label{fig:semistructured_neuron_scaling}
\end{figure}

	Finally, to illustrate that fixed-block semi-structured pruning leads to a genuinely new constant $\tau_{N:M}$,
	Table~\ref{tab:semistructured_neuron_patterns} compares several patterns with the same keep ratio $N/M=1/2$,
	namely $1\!:\!2$, $2\!:\!4$, $4\!:\!8$, and $8\!:\!16$. Even at the same keep ratio, the limiting constants differ.
	Thus, for fixed block size, the asymptotic distortion is not determined solely by the retained fraction, but also by the
	local block structure.

\begin{table}[H]
	\centering
	\caption{{Comparison of several fixed-block patterns at the common keep ratio $N/M=1/2$. The empirical
			raw and renormalized similarities agree closely with the predicted constants $\tau_{N:M}$ and $\sqrt{\tau_{N:M}}$,
			respectively. The values increase with block size, showing that fixed-block semi-structured pruning leads to a new
			constant that depends on the local pattern, not only on the retained fraction.}}
	\label{tab:semistructured_neuron_patterns}
	\vspace{1mm}
	
	{\setlength{\tabcolsep}{4pt}
		\renewcommand{\arraystretch}{1.05}
		\footnotesize
		\begin{tabular}{c|c|cc|cc}
			\hline
			Pattern & $\tau_{N:M}$ & \multicolumn{2}{c|}{$\mathbf{W}^{\top}\mathbf{W}_p$} & \multicolumn{2}{c}{$\mathbf{W}^{\top}\widetilde{\mathbf{W}}_p$} \\
			& & Theory & Empirical & Theory & Empirical \\
			\hline
			$1\!:\!2$  & 0.8183 & 0.8183 & $0.8183 \pm 0.0048$ & 0.9046 & $0.9046 \pm 0.0027$ \\
			$2\!:\!4$  & 0.8676 & 0.8676 & $0.8676 \pm 0.0037$ & 0.9314 & $0.9314 \pm 0.0020$ \\
			$4\!:\!8$  & 0.8967 & 0.8967 & $0.8967 \pm 0.0031$ & 0.9469 & $0.9469 \pm 0.0016$ \\
			$8\!:\!16$ & 0.9124 & 0.9124 & $0.9124 \pm 0.0027$ & 0.9552 & $0.9552 \pm 0.0014$ \\
			\hline
		\end{tabular}
	}
\end{table}

\subsection{Semi-structured pruning in random networks}
\label{app:semistructured_networks}

	We also performed finite-width random-network experiments to verify that the same semi-structured overlap constant
	$\tau_{N:M}$ governs depth accumulation at the network level. In these experiments, each row of each weight matrix
	was pruned blockwise according to an $N\!:\!M$ rule and then renormalized by $1/\sqrt{\tau_{N:M}}$, exactly as in the
	theoretical construction. We then compared the original and pruned network outputs on random inputs.
	
	Table~\ref{tab:semistructured_random_networks} summarizes two representative validations. The left panel considers the
	practical $2\!:\!4$ pattern in random sign networks of width $512$. Here the empirical output alignment and symmetric NMSE
	closely follow the kernel-recursion prediction obtained by replacing $\tau_q$ with $\tau_{2:4}$. The right panel considers
	normalized ReLU random networks at fixed depth $L=5$ and common keep ratio $N/M=1/2$. Although the finite-width ReLU
	runs exhibit larger quantitative deviations from the infinite-width prediction, the same qualitative trend is visible:
	larger block sizes lead to larger output alignment and smaller distortion. This is consistent with the fact that the
	corresponding constants $\tau_{N:M}$ increase from $1\!:\!2$ to $8\!:\!16$.

\begin{table}[H]
	\centering
	\caption{Finite-width random-network validation of the semi-structured network recursion. Left: sign networks with the practical $2\!:\!4$ pattern, width $512$, averaged over $10$ random networks and $1024$ random inputs. Right: normalized ReLU networks at fixed depth $L=5$ and width $512$, comparing several $N\!:\!M$ patterns with common keep ratio $N/M=1/2$.}
	\label{tab:semistructured_random_networks}
	\vspace{1mm}
	
	{\setlength{\tabcolsep}{4pt}
		\renewcommand{\arraystretch}{1.05}
		\footnotesize
		
		\begin{minipage}{0.39\textwidth}
			\centering
			{(a) Depth sweep, sign activation, \(2\!:\!4\)}\\[1mm]
			\begin{tabular}{@{}c|cc|cc@{}}
				\hline
				\multirow{2}{*}{$L$} & \multicolumn{2}{c|}{Alignment} & \multicolumn{2}{c@{}}{Symmetric NMSE} \\
				& Theory & Empirical & Theory & Empirical \\
				\hline
				1 & 0.76 & 0.77 $\pm$ 0.01 & 0.47 & 0.46 $\pm$ 0.02 \\
				2 & 0.50 & 0.50 $\pm$ 0.02 & 0.99 & 1.00 $\pm$ 0.04 \\
				3 & 0.31 & 0.32 $\pm$ 0.03 & 1.38 & 1.36 $\pm$ 0.06 \\
				5 & 0.11 & 0.12 $\pm$ 0.03 & 1.78 & 1.75 $\pm$ 0.05 \\
				7 & 0.04 & 0.04 $\pm$ 0.02 & 1.92 & 1.91 $\pm$ 0.05 \\
				\hline
			\end{tabular}
		\end{minipage}
		\hfill
		\begin{minipage}{0.58\textwidth}
			\centering
			{{(b) Pattern sweep, normalized ReLU, \(L=5\)}}\\[1mm]
			\begin{tabular}{@{}c|c|cc|cc@{}}
				\hline
				Pattern & $\tau_{N:M}$ & \multicolumn{2}{c|}{Alignment} & \multicolumn{2}{c@{}}{Symmetric NMSE} \\
				& & Theory & Empirical & Theory & Empirical \\
				\hline
				$1\!:\!2$  & 0.82 & 0.75 & 0.39 $\pm$ 0.06 & 0.51 & 2.36 $\pm$ 1.42 \\
				$2\!:\!4$  & 0.87 & 0.80 & 0.49 $\pm$ 0.03 & 0.40 & 1.36 $\pm$ 0.31 \\
				$4\!:\!8$  & 0.90 & 0.83 & 0.58 $\pm$ 0.03 & 0.34 & 1.32 $\pm$ 0.90 \\
				$8\!:\!16$ & 0.91 & 0.85 & 0.62 $\pm$ 0.03 & 0.30 & 1.18 $\pm$ 0.46 \\
				\hline
			\end{tabular}
		\end{minipage}
	}
\end{table}

%% file: tex/app/related-work.tex
\section{Related Work}
\label{sec:related}

{\bf Pruning and compression with guarantees.}
Pruning and compression of neural networks have been widely studied to reduce model size and inference cost, with several works providing theoretical results for training-free filter pruning and quantization~\cite{lee2025trainingfree,bai2023unified,CHEN2023109780}. Pruning with retraining is an important practical regime, and in many settings it can yield stronger final performance. Classical pruning methods often combine mask selection with weight readjustment or retraining, from second-order pruning methods to iterative prune-and-compress pipelines~\cite{lecun1989optimal,han2016deepcompression}.
Coreset-based pruning methods offer theoretical guarantees on the trade-off between compression rate and approximation or generalization error by identifying a ``core'' subset of the network such that the remainder can be safely removed. These approaches have been applied to prune weights~\cite{baykal2018data,baykal2019sipping}, neurons~\cite{mussay2020data,tukan2022pruning}, and filters~\cite{dubey2018coreset,yin2022coreset}. In particular, coresets have been used to derive generalization bounds for pruning with the help of heuristics~\cite{baykal2019sipping}. Aghasi et al.~\cite{aghasi2017nettrim} propose Net-Trim, a convex optimization-based layer-wise pruning method for ReLU networks with theoretical performance guarantees on maintaining output consistency after pruning.
Our goal here, however, is to understand the pruning operation itself, before any recovery by further optimization. This is the regime in which its effect can be characterized cleanly. More broadly, even when retraining is used, understanding the intrinsic distortion caused by the initial pruning step is still foundational. Extending the theory to prune-and-retrain settings is an important next step, but we view it as building on top of, rather than replacing, an understanding of the one-shot distortion analyzed here.
In contrast to these lines of work, the concentration phenomenon of pruning that we present appears to be novel and has not been previously reported in the existing literature.

{\bf High-dimensional concentration and limits for structured sparsification.}
While classical results in high-dimensional probability, such as Lévy’s Lemma for Lipschitz functions and the Johnson–Lindenstrauss Lemma for random projections over spheres and Gaussians (see, e.g., \cite{ledoux2001concentration,vershynin2018high,boucheron2013concentration}), offer insight into average-case behavior, they do not directly apply to the structured sparsification process considered here. In particular, the transformations we analyze are highly non-Lipschitz and discontinuous. To the best of our knowledge, this form of pruning-induced concentration has not been previously analyzed in the literature. 

{\bf One-shot, training-free, and semi-structured pruning for large models.}
The one-shot regime remains practically relevant in its own right. For very large pretrained models, especially LLMs, post-pruning retraining or full fine-tuning can be expensive or unavailable, and a substantial recent literature therefore studies one-shot or training-free pruning.
Recent work has emphasized training-free and one-shot sparsification of large language models, including reconstruction based one-shot pruning~\cite{frantar2023sparsegpt}, simple weight and activation scoring rules~\cite{sun2023wanda}, and forward-pass only structured pruning pipelines~\cite{kolawole2024everybodyprunenow}. Extensions that incorporate additional signals, such as regional gradients, have also been proposed~\cite{yang2025wandapp}.
Semi-structured $N\!:\!M$ sparsity is especially relevant because patterns such as $2\!:\!4$ are designed for hardware-friendly inference. SparseGPT already considers semi-structured $2\!:\!4$ and $4\!:\!8$ variants, and recent work further studies learnable or post-training semi-structured sparsity for LLMs~\cite{fang2024maskllm,liu2026armor}. This motivates our extension of the concentration analysis to blockwise semi-structured one-shot pruning in Appendix~\ref{sec:semistructured_nm}, which moves the theory closer to pruning regimes of current interest in LLM inference.
These efforts primarily demonstrate empirical feasibility at scale, while our results provide a complementary first-principles explanation via pruning-induced cosine similarity concentration and its propagation through depth.

{\bf Early exit and dynamic inference.}
While the above efforts focus on fixed pruning or data-free pruning strategies, early exit networks offer a dynamic alternative by adaptively terminating inference when intermediate predictions reach sufficient confidence \cite{han2021dynamic,li2024seenn,teerapittayanon2016branchynet, 10494595, laskaridis2021adaptive}. Several works have explored early exit decision policies under uncertainty and resource constraints. In~\cite{jazbec2024fast}, early-exit neural networks are formulated within a risk control framework, and the authors optimize early-exit threshold hyperparameters under reliability constraints. In~\cite{ju2021learning,liu2021dynamic}, the regret of a dynamic early exit scheduling scheme is analyzed, with regret defined in terms of either utility or processing time relative to the optimal static policy for video applications. Another early exit scheme tailored for video applications is analyzed in~\cite{wang2019see}. A regret analysis for unsupervised early exit networks is provided in~\cite{10.1145/3564121.3564137}, and confidence sequences enabling reliable early exits are analyzed in~\cite{jazbec2024nested}. As a different application, early exit networks have been utilized to simplify/prune datasets for faster training \cite{gormez2026dataset}. While these works provide practical algorithms, their theoretical analyses have focused on regret or threshold optimization rather than the relationship between compute and generalization, which is the main focus of this work. 

{\bf Compute-adaptive inference for transformers and LLMs.}
In transformer and LLM settings, closely related compute-adaptive mechanisms include adaptive depth through
layer or sublayer skipping \cite{he2025adaskip,luo2025flexidepth} and system-level policies that combine
model selection with early exit to meet latency or cost constraints \cite{kumar2025helios}. These lines are
compatible with the frozen-backbone regime emphasized in this paper, and further motivate an explicit
characterization of compute--accuracy tradeoffs through alignment between partial and full computations.

{\bf Conditional computation beyond early exit.}
Beyond early exit, conditional computation more broadly includes structures such as neural decision trees and forests, which can be viewed as a special case of conditional computation~\cite{kontschieder2015deep,yang2018deep}. Theoretical work has shown that these models can achieve statistical consistency, i.e., convergence to the Bayes optimal classifier of the underlying data distribution as data and model capacity grow~\cite{biau2019neural,lu2020validation}. MoE models offer another form of conditional activation at scale. Recent works have studied scaling laws for MoE LLMs, providing insights into how performance scales with the number of active experts~\cite{ludziejewski24a,chen2024sparsemoe}. However, these studies rely exclusively on numerical experiments without theoretical analysis. Finally, a theoretical study of memorization capacity for neural networks with conditional computation was presented in~\cite{koyuncu2023memorization}. In contrast, we focus on the generalization performance of conditional networks under structured sparsity and dynamic activation.